\documentclass[10pt,letterpaper]{article}
\usepackage[T1]{fontenc}
\usepackage{times}
\usepackage[letterpaper,top=1in,bottom=1in,left=1.25in,right=1.25in]{geometry}
\usepackage[authoryear,round]{natbib}
\usepackage{amsmath,amssymb,bm}
\usepackage{amsthm}
\newtheorem{proposition}{Proposition}
\newtheorem{definition}{Definition}
\usepackage{graphicx}
\usepackage{tikz}
\usetikzlibrary{arrows.meta,positioning,fit,backgrounds}
\usepackage{booktabs}
\usepackage{multirow}
\usepackage{tabularx}
\usepackage{xcolor}
\usepackage{colortbl}
\usepackage{makecell}
\usepackage{capt-of}
\usepackage[hidelinks]{hyperref}
\usepackage{url}

\definecolor{ourcol}{HTML}{B03A16}
\definecolor{wincol}{HTML}{1E7A4C}
\newcommand{\best}[1]{\textbf{#1}}
\newcommand{\snd}[1]{\underline{#1}}
\newcommand{\ours}[1]{\textbf{\textcolor{ourcol}{#1}}}
\newcommand{\zero}{\textbf{\textcolor{wincol}{0.0}}}

\title{Expandable Modality Experts for\\
Medical Vision--Language Learning}
\author{DM BAPPY\textsuperscript{1}, Eunsu Kim\textsuperscript{1}, Hanjoo Lee\textsuperscript{1},\\
Sung Hak Lee\textsuperscript{2}, Sangjeong Ahn\textsuperscript{1}\\[5pt]
\small\textsuperscript{1}Korea University\\
\small\textsuperscript{2}The Catholic University of Korea}
\date{}
\hypersetup{pdftitle={Expandable Modality Experts for Medical Vision--Language Learning},pdfauthor={DM BAPPY, Eunsu Kim, Hanjoo Lee, Sung Hak Lee, Sangjeong Ahn}}
\begin{document}
\maketitle
\begin{abstract}
Learning across medical imaging modalities requires making use of images
without paired text and accommodating new domains while retaining existing
representations. We propose a medical vision--language technique built
around an expandable bank of modality experts. Its Multiway Transformer
combines shared visual computation with deterministically routed feed-forward
experts and a separate text pathway. A modality-aware visual tokenizer supplies
targets for masked pretraining on 19.4 million unpaired image rows across seven
modalities; contrastive alignment then uses 3.23 million image--text pairs.
Additional fundus and optical coherence tomography experts learn from paired
data while the original image and text embedding paths remain frozen.
With a ViT-Large backbone, the proposed technique achieves the best or tied-best result in
24 of 25 benchmark settings across nine modalities: 10 of 11 zero-shot
classification benchmarks, all 10 linear probes, and all four retrieval
benchmarks. Mean balanced accuracy is 0.674 for zero-shot classification and
0.856 for linear probing, and mean bidirectional retrieval recall is 20.70\%.
A separate three-seed sequential experiment with matched adaptation data yields
fundus and OCT retrieval gains of 14.95 and 9.31 percentage points, with no
measured loss on preceding fundus or pathology tasks. The preservation argument
requires unchanged existing computation paths and inference conditions, and
feature-identity checks verify retention on evaluated inputs. The proposed technique provides
a modular approach to learning and expanding medical vision--language
representations.

\end{abstract}
\section{Introduction}\label{sec:introduction}
Medical representation learning must account for the visual differences among
radiography, CT, MRI, ultrasound, histopathology, dermatology, endoscopy, fundus
photography, and optical coherence tomography (OCT). A common language interface
can support text-prompted classification and image--text retrieval across these
domains, but its development faces two practical constraints. Many available
images have no paired text, and training on an additional modality can change
the representations used by existing tasks. We study an architecture that
addresses these constraints through unpaired pretraining and expandable
modality-specific capacity.

Image--text supervision alone provides an incomplete basis for this goal.
Paired medical data vary substantially across modalities, and their collection
can require expert annotation or the extraction of figure--caption pairs from
publications. The resulting training corpus need not reflect the distribution
of available clinical images. BiomedCLIP (BMC)~\citep{zhang2023biomedclip},
the BIOMEDICA baseline (BMCA)~\citep{biomedica2024},
ConceptCLIP~\citep{nie2026conceptclip}, and
UniMed-CLIP~\citep{khattak2024unimed} establish useful biomedical
image--text interfaces. Vision-only models such as
MerMED-FM~\citep{zhou2026mermed} support representation learning across
modalities but do not supply a native text interface. MUSK~\citep{xiang2025musk}
combines unpaired pretraining and paired alignment within pathology.
These approaches motivate separating visual learning from language alignment
while extending the training scheme across medical imaging domains.

Expansion creates a second difficulty. Updating shared parameters to learn a
new modality can alter the embeddings used by existing tasks. Rehearsal and
distillation can constrain such changes~\citep{yao2025retcop}, while progressive
architectures preserve earlier computations by freezing existing
paths~\citep{rusu2016progressive}. For a medical vision--language encoder, this
principle must account for the image encoder, text encoder, projection heads,
and inference conditions. Preserving only a subset of these components does
not establish that previous image--text similarities remain unchanged.

We introduce the proposed technique, which combines unpaired visual
pretraining with an expandable expert bank. The model builds on a Multiway
Transformer~\citep{wang2023beit3} with a Domain-Aware Mixture-of-Experts
(DA-MoE) in the visual feed-forward layers. Each image uses a known modality
identifier to select its expert, alongside a shared visual pathway. A separate
text pathway interacts with the image encoder through contrastive alignment.
A modality-aware VQ-KD tokenizer~\citep{peng2022beitv2} supplies targets for
masked-image pretraining. The seven-modality core learns from $19.4$M unpaired
image rows and $3.23$M image--text pairs. Fundus and OCT experts subsequently
learn from paired data with the existing embedding paths frozen; the final
nine-modality assembly uses $3.34$M pairs in total across alignment and expansion.

Our setting concerns learning one encoder across heterogeneous image sources;
it does not require patient-matched images across modalities. At inference,
each image is encoded through its modality route. We evaluate whether the
resulting embeddings support text-prompted classification, linear probing,
and cross-modal retrieval, and whether new experts learn useful representations
while earlier pathways are retained.

The study makes three contributions.
\begin{itemize}
\item \textbf{Learning across modalities from unpaired images and paired text.}
We combine a modality-aware visual tokenizer, masked pretraining, and
contrastive alignment within one routed architecture. This separates the use
of image-only data from the image--text supervision required for alignment.
\item \textbf{Expansion through frozen existing paths.}
We formalize the conditions under which adding modality experts preserves
the original image and text embedding functions. We distinguish a sequential
experiment with matched adaptation data from the final assembly of
independently trained full-data experts.
\item \textbf{Evaluation of transfer, adaptation, and retention.}
We report $25$ benchmark settings across nine modalities, controlled ViT-B
comparisons, and three-seed adaptation results. The proposed technique\ is best or tied-best in
$24$ of the $25$ main settings. Feature-identity checks and retrieval tests
assess retention on evaluated inputs.
\end{itemize}

Section~\ref{sec:method} describes the architecture and objectives.
Sections~\ref{sec:evaluation} and~\ref{sec:benchmarks} present the evaluation
design and main benchmarks. Section~\ref{sec:expansion-results} separates
adaptation from retention, and Sections~\ref{sec:abl} and~\ref{sec:discussion}
examine the supporting evidence and its limits. Qualitative image comparisons
are collected in the appendices.

\section{Modality experts for medical vision--language learning}\label{sec:method}
The proposed technique\ has three ingredients: a Multiway backbone whose feed-forward stack is a
\emph{Domain-Aware Mixture-of-Experts} (\S\ref{sec:damoe}); a two-stage
training pipeline (\S\ref{sec:train}); and a continual-expansion rule that preserves
existing embedding functions under explicit conditions (\S\ref{sec:cl}).
Throughout, $m\!\in\!\mathcal{M}$ denotes an
input's modality. Figure~\ref{fig:pipeline} gives the end-to-end overview.

\begin{figure}[tbp]
\centering
\includegraphics[width=0.98\linewidth]{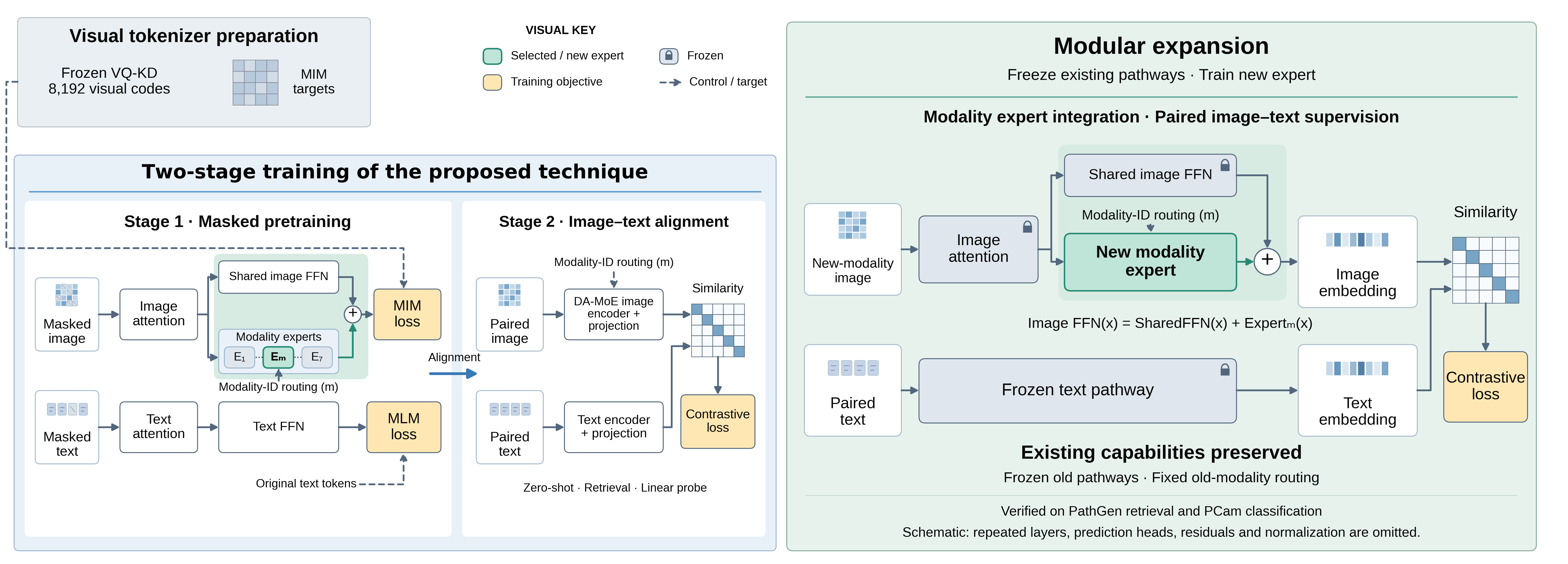}\\[2pt]
{\footnotesize (a) the proposed technique\ pipeline}\\[7pt]
\setlength{\tabcolsep}{0pt}
\begin{tabular}{@{}c@{\hspace{14pt}}c@{}}
\includegraphics[height=2.12in]{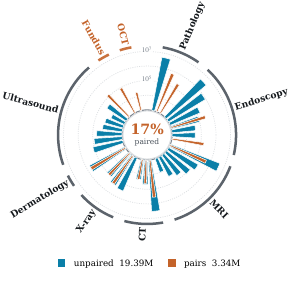} &
\includegraphics[height=2.12in]{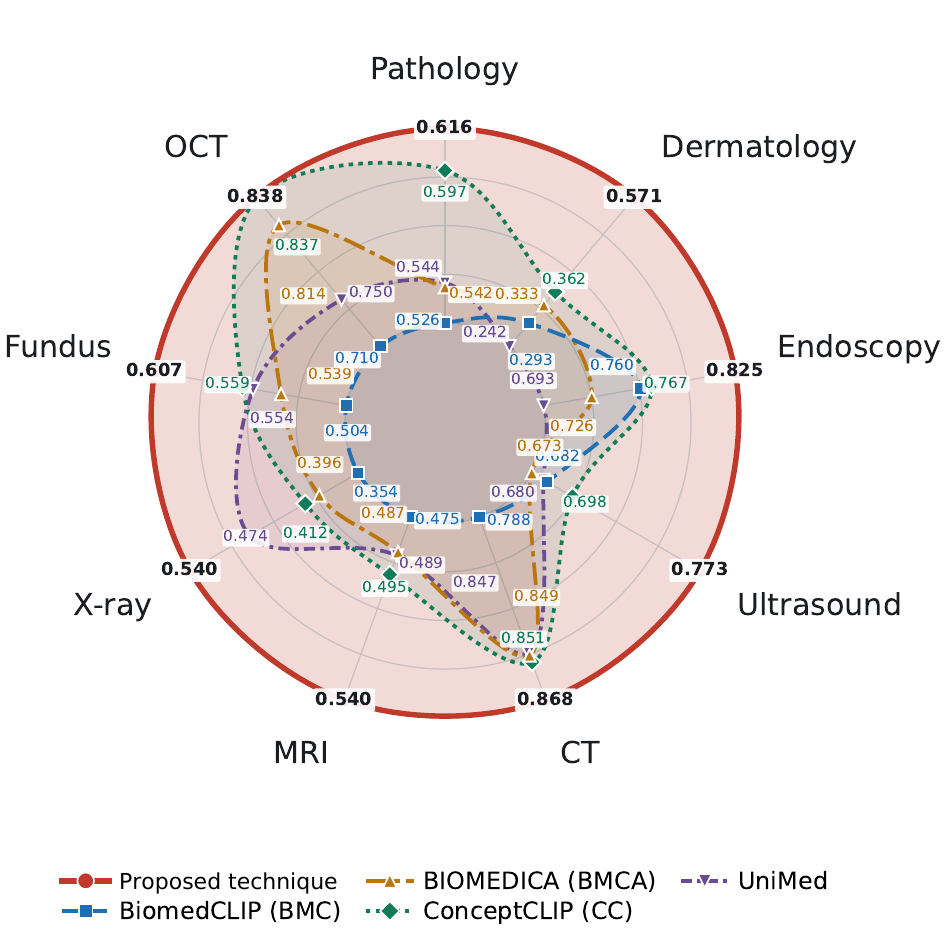} \\[3pt]
{\footnotesize (b) training corpus} & {\footnotesize (c) modality-level results} \\
\end{tabular}
\caption{\textbf{Proposed technique: pipeline, training data, and evaluation breadth.}
\textbf{(a)} A modality-aware tokenizer supplies discrete targets for masked pretraining;
contrastive alignment then connects the separate image and text pathways. New modalities use
paired data to train new routed experts while existing paths remain frozen. \textbf{(b)} Each
spoke is a source dataset grouped by modality. On a shared $\log_{10}$ scale, wide bars show
Stage-1 image rows and narrow bars show Stage-2 pairs; fundus and OCT use paired data only.
The core pair count is approximately $17\%$ of the masked-pretraining image-row count.
\textbf{(c)} Raw mean performance by modality across applicable task families, with retrieval
percentages divided by $100$. Each axis uses its own range; markers are data and curves are
visual guides.}
\label{fig:pipeline}
\label{fig:radar}
\end{figure}

\subsection{Backbone and Domain-Aware Mixture-of-Experts}\label{sec:damoe}
\textbf{Multiway backbone.} The proposed technique\ builds on a Multiway Transformer \citep{wang2023beit3}: an image
and a caption are each patch- or token-embedded and encoded through their own forward pass,
and every layer applies token-type-specific self-attention followed by a token-type-specific
feed-forward branch. The pre-normalized block is
\begin{equation}
\begin{aligned}
\bm{h}^{(\ell)}_{t}
&= \bm{u}^{(\ell-1)}_{t}
+ \operatorname{Attn}^{(\ell)}_{t}\!\left(
\operatorname{LN}^{(\ell)}_{t,1}
(\bm{u}^{(\ell-1)}_{t})\right), \\
\bm{u}^{(\ell)}_{t}
&= \bm{h}^{(\ell)}_{t}
+ \operatorname{FFN}^{(\ell)}_{t}\!\left(
\operatorname{LN}^{(\ell)}_{t,2}
(\bm{h}^{(\ell)}_{t})\right),
\qquad t\in\{\mathrm{vision},\mathrm{language}\}.
\end{aligned}
\end{equation}
Attention and FFN denote the corresponding implementation blocks, including their internal
projections and any implementation-specific operations in addition to the displayed
pre-normalization; dropout and stochastic depth are omitted for clarity. The
subscript $t$ selects an independently parameterized attention and FFN branch, and the two
normalization subscripts distinguish the attention and FFN operations. The two towers share no attention weights and
interact only through the Stage-2 contrastive loss on their pooled embeddings, never through
a joint attended sequence. It is the vision FFN branch that we make \emph{domain-aware}.

\textbf{Domain-Aware MoE.} We replace the single vision FFN at a set of layers
$\mathcal{L}\subseteq\{1,\dots,24\}$ (backbone depth for the reported ViT-Large model)
with a bank of $E$ domain experts
$\{f^{(\ell)}_1,\dots,f^{(\ell)}_E\}$ plus a shared expert $s^{(\ell)}$. A \emph{hard}
router maps each image to the single expert of its modality,
\begin{equation}
\operatorname{FFN}^{(\ell)}_{\mathrm{vision}}(\bm{x};m)
=s^{(\ell)}(\bm{x})+f^{(\ell)}_{r(m)}(\bm{x}),
\qquad \ell\in\mathcal{L},\quad r:\mathcal{M}\!\to\!\{1,\dots,E\},
\label{eq:route}
\end{equation}
so pathology, CT, MRI, ultrasound, X-ray, dermatology, and endoscopy each have an expert while
$s^{(\ell)}$ carries shared visual information. Routing is deterministic and specified by the
image modality ID. Only the selected modality-specific expert contributes alongside the
shared FFN. Different modalities retain separate expert parameters while jointly training
the shared visual pathway. No load-balancing objective is used in the reported configuration.

\textbf{Expert placement.} We index the $24$ Transformer layers as $1,\ldots,24$.
$\mathcal{L}$ controls how much dedicated capacity a modality receives. \emph{Integrate}-mode
routes every fourth layer, $\mathcal{L}=\{4,8,12,16,20,24\}$, giving depth-$6$ experts.
\emph{Expand}-mode routes every layer, $\mathcal{L}=\{1,\ldots,24\}$, giving depth-$24$
experts. Layers outside $\mathcal{L}$ use their ordinary visual FFN. The seven-modality
ViT-L core uses the six integrate-mode routed layers. For the full-data fundus and OCT
extensions, the remaining layers are converted with function-preserving FFN remapping and
zero-output branches, and each new expert spans all $24$ layers. This adds capacity for the
new modalities while preserving the existing computation paths (\S\ref{sec:fundus}).

\subsection{Two-stage training pipeline}\label{sec:train}
\textbf{Modality-aware tokenizer.} A single VQ-KD tokenizer
\citep{peng2022beitv2} $q:\text{image}\!\to\!\{1,\dots,K\}^{N}$ ($K\!=\!8192$) maps each
image to a grid of discrete codes, trained to reconstruct per-modality teacher features
through a quantized bottleneck. External teacher features supervise tokenizer training. The
resulting tokenizer is then frozen and produces discrete targets for the proposed technique's masked-image
objective. Gradients from this objective update the proposed technique\ but are not propagated through the
tokenizer or its teachers.

\textbf{Stage 1: masked unimodal pretraining.} With the tokenizer frozen we pretrain
the backbone with two independent objectives, each on its own forward pass: masked image
modeling against the VQ-KD codes, and masked language modeling on unpaired text. The image
and text streams are unpaired. Let $\mathcal{I}_{M}$ and $\mathcal{T}_{M}$ contain masked
sample--token pairs across the image and text microbatches, respectively, and let
$c_{b,i}$ denote the frozen tokenizer target at image position $i$ in sample $b$. Routing each image through its
modality expert (Eq.~\ref{eq:route}),
\begin{equation}
\mathcal{L}_{\mathrm{MIM}}
=-\frac{1}{|\mathcal{I}_{M}|}
\sum_{(b,i)\in\mathcal{I}_{M}}
\log p_\theta\big(c_{b,i}\mid\widetilde{\bm{x}}^{\mathrm v}_{b},m_b\big),
\end{equation}
\begin{equation}
\mathcal{L}_{\mathrm{MLM}}
=-\frac{1}{|\mathcal{T}_{M}|}
\sum_{(q,j)\in\mathcal{T}_{M}}
\log p_\theta\big(t_{q,j}\mid\widetilde{\bm{x}}^{\mathrm t}_{q}\big),
\end{equation}
\begin{equation}
\boxed{\mathcal{L}_{\mathrm{stage1}}
=\mathcal{L}_{\mathrm{MIM}}+\mathcal{L}_{\mathrm{MLM}}.}
\end{equation}
where $\widetilde{\bm{x}}^{\mathrm{v}}_{b}$ and $\widetilde{\bm{x}}^{\mathrm{t}}_{q}$ are
masked image and text inputs. The indices $b$ and $q$ refer to independent microbatches and
do not indicate paired samples; MLM has no access to image tokens. Stage~1 is the equally
weighted sum of the mean MIM and mean MLM losses.
If no text tokens are masked, the implementation sets $\mathcal{L}_{\mathrm{MLM}}=0$.
Stage~1 trains the shared visual pathway and the original modality-specific experts before
image--text alignment. During subsequent expansion, a new expert learns from paired data
against the frozen text embedding space. 


\textbf{Stage 2: contrastive alignment.} For core Stage~2 training, we attach linear vision/text heads,
pool and $\ell_2$-normalize a global image embedding $\bm{z}^{\mathrm{v}}$ and text embedding
$\bm{z}^{\mathrm{t}}$, and align them over $3.23$M image--text pairs. Let $a$ be the learned
logit-scale parameter and define
\begin{equation}
s_{ik}=\gamma\,(\bm{z}^{\mathrm v}_i)^\top\bm{z}^{\mathrm t}_k,
\qquad \gamma=\min\{\exp(a),100\}.
\end{equation}
The effective temperature is $\tau=1/\gamma$. The symmetric core-training objective is
\begin{equation}
\boxed{\mathcal{L}_{\mathrm{CLIP}} = -\frac{1}{2B}\sum_{i=1}^{B}
\Bigg[
\log\frac{\exp(s_{ii})}{\sum_{k=1}^{B}\exp(s_{ik})}
+\log\frac{\exp(s_{ii})}{\sum_{k=1}^{B}\exp(s_{ki})}
\Bigg].}
\end{equation}
$B$ is the number of pairs participating in one contrastive calculation. With cross-GPU
gathering, seven GPUs and $32$ pairs per GPU give $B=224$. Eighteen accumulation steps give
a $4{,}032$-pair optimizer batch, not a $4{,}032$-pair negative pool. The implementation's
contrastive helper additionally divides the returned loss by the distributed world size;
this affects optimizer scaling but not the positive or negative sets in the equation. The matched-data
expansion experiment instead uses an equal-caption multiple-positive contrastive objective
and is therefore not described exactly by the diagonal-only equation above.


\subsection{Expansion with preserved embedding functions}\label{sec:cl}
A deployed encoder must absorb new modalities without disturbing the ones it serves.
The proposed technique\ does this by \emph{growing the expert bank on a frozen base}.

\begin{definition}[Integration operator]
Given the proposed technique\ with modality experts $\{f_j\}_{j=1}^{E}$ and a new modality $m^\star$,
integration $\mathcal{A}$ (i) appends a new expert $f_{E+1}$ whose parameters are
independent of the existing experts; (ii) freezes all parameters on the existing image
and text embedding paths, including embeddings, attention blocks, normalization layers,
shared FFNs, existing modality experts, and projection heads; and (iii) trains only the
new expert using paired image--text data and a contrastive objective. The reported fundus
and OCT extensions use no additional Stage-1 masked pretraining. The new expert comprises
modality-specific FFNs across the routed layers. Integration assigns
$r(m^\star)=E+1$ while preserving all existing routing assignments and embedding computations.
\end{definition}

\begin{proposition}[Preservation of existing embedding functions]
\label{prop:zero}
Let $g_\theta^{\mathrm v}$ and $g_\theta^{\mathrm t}$ denote the image and text
embedding functions before expansion, and let $\theta'$ denote the parameters
after expansion. If expansion preserves all original parameters, persistent
state, routing, and inference operations, then
\begin{align}
g_{\theta'}^{\mathrm v}(\bm{x},m)
&= g_\theta^{\mathrm v}(\bm{x},m),
&& \forall\, \bm{x},\; m\in\mathcal{M}_{\mathrm{old}}, \\
g_{\theta'}^{\mathrm t}(\bm{t})
&= g_\theta^{\mathrm t}(\bm{t}),
&& \forall\, \bm{t}.
\end{align}
\end{proposition}
\begin{proof}[Proof sketch]
The embedding of an input with modality $m$ depends only on parameters on its
\emph{computation path}: shared attention and embeddings, the shared experts
$s^{(\ell)}$, and---by the hard router (Eq.~\ref{eq:route}) with $r(m)\!=\!j\le E$
unchanged---exactly one old expert $f_j$ per DA-MoE layer. Every one of these lies in
the frozen base; the only trainable parameters ($f_{E+1}$) do not affect the embedding of
an old-modality input. The text pathway also remains unchanged. With the same persistent
state, routing, and inference operations, both existing embedding functions are therefore
preserved. For the 6-to-24-layer conversion, this argument additionally requires the
stated function-preserving remapping.
\end{proof}

Proposition~\ref{prop:zero} states preservation by construction under the specified
conditions. Methods that alter the effective mapping used by existing modalities, including
shared LoRA adapters, do not automatically satisfy these conditions. The guarantee arises
from preserving existing computation paths, not from routing alone. Unchanged predictions
and metrics additionally require unchanged
preprocessing, scoring rules, candidate sets, and downstream classifiers. Bitwise equality is
an empirical verification on the evaluated inputs, not an unconditional consequence for every
possible output. \S\ref{sec:forget} reports this verification, and
Section~\ref{sec:fundus} tests whether the new experts learn useful representations.

\section{Data and evaluation protocols}\label{sec:evaluation}
\subsection{Training data and model configurations}
The seven-modality core uses $19{,}393{,}075$ unpaired image rows for masked
pretraining and $3{,}229{,}073$ image--text pairs for alignment. Its modalities
are pathology, CT, MRI, ultrasound, X-ray, dermatology, and endoscopy. The fundus
and OCT experts use $88{,}143$ and $19{,}396$ additional image--text pairs,
respectively, without further masked pretraining. This gives $3{,}336{,}612$
pairs across core alignment and both full-data extensions.
Appendix~\ref{app:leak} lists data sources and evaluation separation decisions.

The main model uses a ViT-Large/16 backbone. The evaluated nine-modality
assembly stores $1.280$B parameters, while a single modality's selected image
path contains $355.40$M parameters, excluding the text encoder and unselected
experts. The compared vision--language baselines are BMC (ViT-B/16, $196$M
parameters), BMCA (ViT-L/14, $454$M), ConceptCLIP (SigLIP-SO400M, $543$M),
and UniMed-CLIP (ViT-B/16, $196$M). MerMED-FM (ViT-B, $86$M) is included for
linear probing. The proposed technique\ has the largest stored model in this comparison;
Appendix~\ref{app:params} distinguishes stored parameters from selected image
pathways and gives the implementation settings.

The tokenizer uses six visual teachers: UNI for pathology, MedDINOv3 for CT,
USFM for ultrasound, RAD-DINO for X-ray, GastroNet-5M for endoscopy, and BMC for
MRI and dermatology. Once tokenizer training is complete, the core and expert
training objectives use no additional teacher-feature or teacher-logit
distillation loss. Appendix~\ref{app:tokenizer-ablation} compares multi-teacher
and single-teacher tokenization under a limited, matched training budget.

\subsection{Transfer benchmarks and metrics}
The main evaluation comprises $11$ zero-shot classification benchmarks,
$10$ linear probes, and $4$ cross-modal retrieval benchmarks across nine
modalities. Classification and linear probing use balanced accuracy. Retrieval
uses cosine similarity with a fixed gallery per benchmark and reports the mean
of Recall@$1$, Recall@$10$, and Recall@$50$ across image-to-text and
text-to-image directions. Within each comparison table, models use common
prompts, probes, galleries, and scoring procedures. MerMED-FM is evaluated only
by linear probing because it has no text encoder.

We report means within each task family. They summarize the benchmarks in the
corresponding table and should not be compared across families as a common
metric. MerMED-FM's linear-probe mean covers only its seven available
benchmarks, so we also distinguish performance on that shared subset.

\subsection{Expansion experiments}
We use two designs with different purposes. The \emph{matched-data sequential}
experiment gives all methods $20{,}000$ fundus pairs followed by $16{,}000$ OCT
pairs, with three seeds and common validation-based checkpoint-selection
rules. Baselines adapt with LoRA; the proposed technique\ trains new routed experts. The
\emph{full-data assembly} combines independently trained fundus and OCT experts
from the same seven-modality base. The latter supplies the expanded model used
in the main transfer tables. It is not a matched-training comparison with
sequential LoRA. Section~\ref{sec:expansion-results} reports these experiments
separately.

\section{Transfer across medical imaging modalities}\label{sec:benchmarks}
\subsection{Zero-shot classification}\label{sec:zs}
The proposed technique\ wins $10$ of $11$ zero-shot benchmarks outright, for a mean of
$0.674$ ($+0.080$ over ConceptCLIP, the strongest baseline at $0.594$). The largest margins
occur on dermatology ($0.819$ vs.\ $0.389$), thyroid ultrasound ($0.744$ vs.\ $0.612$), and
endoscopic polyp detection ($0.824$ vs.\ $0.752$). The proposed technique\ also leads on the newly added
APTOS-DR5 fundus benchmark ($0.589$ vs.\ $0.457$). Its narrowest win is on ColorFundus10
($0.460$ vs.\ $0.459$ for UniMed). BMCA performs best on Pneumonia ($0.911$ vs.\
$0.722$ for the proposed technique). All models obtain low scores on the $19$-class PanNuke benchmark.

\begin{table}[tb]
\centering\small
\caption{\textbf{Zero-shot classification}, balanced accuracy. \best{Best},
\snd{second}. The proposed technique\ wins ten of eleven rows outright. BMCA leads on Pneumonia. $^\ddagger$PanNuke is a
$19$-class tissue panel on which every model scores well below saturation.}
\label{tab:zs}
\begin{tabular}{llccccc}
\toprule
Benchmark & Modality & \ours{OURS} & BMC & BMCA & ConceptCLIP &  UniMed-CLIP \\
\midrule
BreaKHis              & pathology   & \best{.665} & .532 & .553 & \snd{.637} & .601 \\
Pneumonia             & X-ray       & .722 & .732 & \best{.911} & \snd{.859} & .786 \\
CASIA-CXR             & X-ray       & \best{.720} & .630 & .590 & \snd{.680} & .670 \\
Thyroid               & ultrasound  & \best{.744} & .604 & .579 & \snd{.612} & .599 \\
EDD-polyp             & endoscopy   & \best{.824} & \snd{.752} & .644 & .714 & .623 \\
HAM10000              & dermatology & \best{.819} & .300 & .365 & \snd{.389} & .178 \\
LC25000               & pathology   & \best{.881} & .664 & .694 & \snd{.879} & .709 \\
PanNuke$^\ddagger$    & pathology   & \best{.246} & .071 & .093 & \snd{.142} & .112 \\
SARS-CoV-2 CT         & CT          & \best{.743} & .626 & .730 & .717 & \snd{.731} \\
APTOS-DR5             & fundus      & \best{.589} & .377 & .382 & \snd{.457} & .447 \\
ColorFundus10         & fundus      & \best{.460} & .375 & .419 & .445 & \snd{.459} \\
\midrule
\multicolumn{2}{l}{\textbf{Mean} (n=11)} & \ours{\best{.674}} & .515 & .542 & \snd{.594} & .538 \\
\bottomrule
\end{tabular}
\end{table}
\subsection{Linear probe}\label{sec:lp}
Table~\ref{tab:lp} evaluates frozen image features using logistic regression.
The proposed technique\ is best or tied-best on all ten benchmarks, with mean balanced accuracy
$0.856$. At the reported precision, MerMED-FM ties the proposed technique\ on CT-Kidney
($0.993$), and ConceptCLIP ties it on ColorFundus10 ($0.712$).

On the seven benchmarks available for every model, the proposed technique\ averages $0.906$,
compared with $0.867$ for MerMED-FM and $0.866$ for ConceptCLIP. MerMED-FM has
no reported fundus or OCT probe, so its mean covers this seven-benchmark subset.
ConceptCLIP's mean over all ten benchmarks is $0.824$. The different benchmark
coverage matters when interpreting the mean row.

\begin{table}[tb]
\centering\small
\caption{\textbf{Linear probe} on frozen features, $5$-fold cross-validation (seed $0$),
balanced accuracy. \best{Best}, \snd{second}. The proposed technique\ is best or tied-best on all ten rows
and has the highest mean over all ten benchmarks. Ties
are CT-Kidney, where the proposed technique\ and MerMED-FM both reach $0.993$, and ColorFundus10, where the proposed technique\
and ConceptCLIP both reach $0.712$. $^\dagger$MerMED-FM has no reported fundus/OCT probe;
its mean is computed over the original seven benchmarks, where the proposed technique\ scores $0.906$ on the
same subset.}
\label{tab:lp}
\begin{tabular}{lcccccc}
\toprule
Benchmark & \ours{OURS} & BMC & BMCA & ConceptCLIP & MerMED &  UniMed-CLIP \\
\midrule
HEST (pathology)                 & \best{.994} & .978 & \snd{.986} & .981 & .984 & .979 \\
BrainTumor (MRI)                 & \best{.972} & .866 & .894 & \snd{.926} & .917 & .903 \\
CT-Kidney (CT)                   & \best{.993} & .950 & .967 & .986 & \best{.993} & .964 \\
TN5000 (ultrasound)              & \best{.802} & .761 & .768 & \snd{.783} & .782 & .762 \\
RadImageNet (MRI)                & \best{.956} & .940 & .942 & \snd{.953} & .932 & .949 \\
EDD2020 (endoscopy)              & \best{.825} & .767 & .808 & .820 & \snd{.824} & .763 \\
HAM10000 (dermatology)           & \best{.797} & .519 & .559 & .610 & \snd{.626} & .515 \\
APTOS-DR5 (fundus)               & \best{.665} & .610 & \snd{.654} & .624 & --$^\dagger$ & .620 \\
ColorFundus10 (fundus)           & \best{.712} & .653 & .700 & \best{.712} & --$^\dagger$ & .689 \\
OCTDL (OCT)                      & \best{.838} & .710 & .814 & \snd{.837} & --$^\dagger$ & .750 \\
\midrule
\textbf{Mean} (n=10)             & \ours{\best{.856}} & .776 & .809 & \snd{.824} & --$^\dagger$ & .793 \\
\bottomrule
\end{tabular}
\end{table}

\subsection{Cross-modal retrieval}\label{sec:ret}
Under the same retrieval protocol, the proposed technique\ achieves the highest mean score ($20.70$ vs.\
$12.09$, a difference of $8.61$; Table~\ref{tab:ret}) and leads all four benchmarks. The
largest gains occur on OpenI ($35.88$ vs.\ $21.97$ for UniMed) and MR-RATE ($11.52$ vs.\
$5.68$ for BMCA). The margins are smaller on derm7pt ($9.64$ vs.\ $8.74$ for
ConceptCLIP) and PathGen-600 ($25.75$ vs.\ $25.61$ for ConceptCLIP).

\begin{table}[tb]
\centering\small
\caption{\textbf{Cross-modal retrieval}, mean $R@\{1,10,50\}$ over image$\to$text and
text$\to$image, with one fixed gallery per benchmark for all models. All benchmarks use
cosine similarity; PathGen-600 treats rows with identical captions as positives.
\best{Best}, \snd{second}. The proposed technique\ leads all four rows and the mean.}
\label{tab:ret}
\begin{tabular}{llccccc}
\toprule
Benchmark & Modality & \ours{OURS} & BMC & BMCA & ConceptCLIP & UniMed-CLIP \\
\midrule
OpenI          & X-ray       & \best{35.88} & 2.60  & 4.23  & 5.47  & \snd{21.97} \\
MR-RATE        & MRI         & \best{11.52} & 4.72  & \snd{5.68}  & 5.12  & 5.12 \\
derm7pt        & dermatology & \best{9.64}  & 6.00  & 7.37  & \snd{8.74}  & 3.30 \\
PathGen-600    & pathology   & \best{25.75} & 17.64 & 19.39 & \snd{25.61} & 17.97 \\
\midrule
\multicolumn{2}{l}{\textbf{Mean} (n=4)} & \ours{\best{20.70}} & 7.74 & 9.17 & 11.24 & \snd{12.09} \\
\bottomrule
\end{tabular}
\end{table}

Figure~\ref{fig:radar}(c) summarizes performance across all nine modalities. The smallest gaps
occur on OCT and CT: the proposed technique\ and ConceptCLIP differ by $0.001$ on OCT and $0.017$ on CT.
The largest gaps occur on dermatology and X-ray.

\section{Modality expansion and representation retention}
\label{sec:expansion-results}
\subsection{Sequential adaptation with matched data}\label{sec:matched-main}
We first assess learning and retention with the same adaptation data for all
methods: $20{,}000$ fundus pairs followed by $16{,}000$ OCT pairs. Each stage
uses the same validation-based checkpoint-selection rules, and results are
averaged over three seeds. The baselines use LoRA while the proposed technique\ trains new
modality experts with the existing paths frozen. Model size, pretraining
history, adaptation capacity, and computational cost remain different.
Appendix~\ref{app:ft} gives the protocol and definitions of gain and forgetting.

Table~\ref{tab:learn-forget} shows that all methods learn from the new data.
The proposed technique\ gains $14.95$ percentage points on fundus retrieval and $9.31$ on OCT,
with no measured loss on the preceding fundus or PathGen tasks. The baselines
lose $5.33$--$15.29$ points on fundus after OCT adaptation and
$2.97$--$6.11$ points on PathGen across the sequence. These outcomes assess
retention under this protocol rather than every possible capability.

\begin{table}[t]
\centering\small
\caption{\textbf{Learning and forgetting under sequential fundus$\to$OCT adaptation.}
All methods use the same $20{,}000$ fundus pairs followed by the same $16{,}000$ OCT pairs,
three training seeds, and the same validation-based checkpoint-selection rules. The baselines
use LoRA, whereas the proposed technique trains new experts while keeping its existing pathways frozen.
Values are mean percentage-point changes across the three seeds. Fundus and OCT gain measure
improvement after training on each new modality. Fundus forgetting is the decrease after OCT
training, and PathGen forgetting is the decrease from the initial model to the final model.}
\label{tab:learn-forget}
\resizebox{\linewidth}{!}{%
\begin{tabular}{@{}lrrrr@{}}
\toprule
Model & Fundus gain & OCT gain & Fundus forgetting & PathGen forgetting \\
\midrule
BMC  & 13.95 & 11.58 &  9.85 & 6.11 \\
BMCA   & 15.47 & 10.53 &  7.59 & 3.22 \\
ConceptCLIP & 13.78 &  9.01 &  5.33 & 5.39 \\
UniMed-CLIP & 11.69 & 10.50 & 15.29 & 2.97 \\
Proposed technique\ ViT-L  & $14.95$ & $9.31$ & \textbf{0.00} & \textbf{0.00} \\
\bottomrule
\end{tabular}
}
\end{table}

The trajectories in Figure~\ref{fig:matched_seq} show the tradeoff between
learning the new modality and retaining earlier tasks. The proposed technique's fundus retrieval
rises from $5.22\%$ to $20.17\%$ and remains fixed after OCT training. Its OCT
score rises from $4.28\%$ to $13.59\%$, below the adapted baselines' final
$15.46\%$--$15.88\%$. Derm7pt retrieval remains at $9.63\%$, with byte-identical
saved features within each seed. Thus, retaining earlier representations does
not imply the highest score on every newly adapted task.

\begin{figure}[tbp]
\centering
\includegraphics[trim=6.0pt 4.5pt 6.0pt 7.5pt, clip, width=\linewidth]{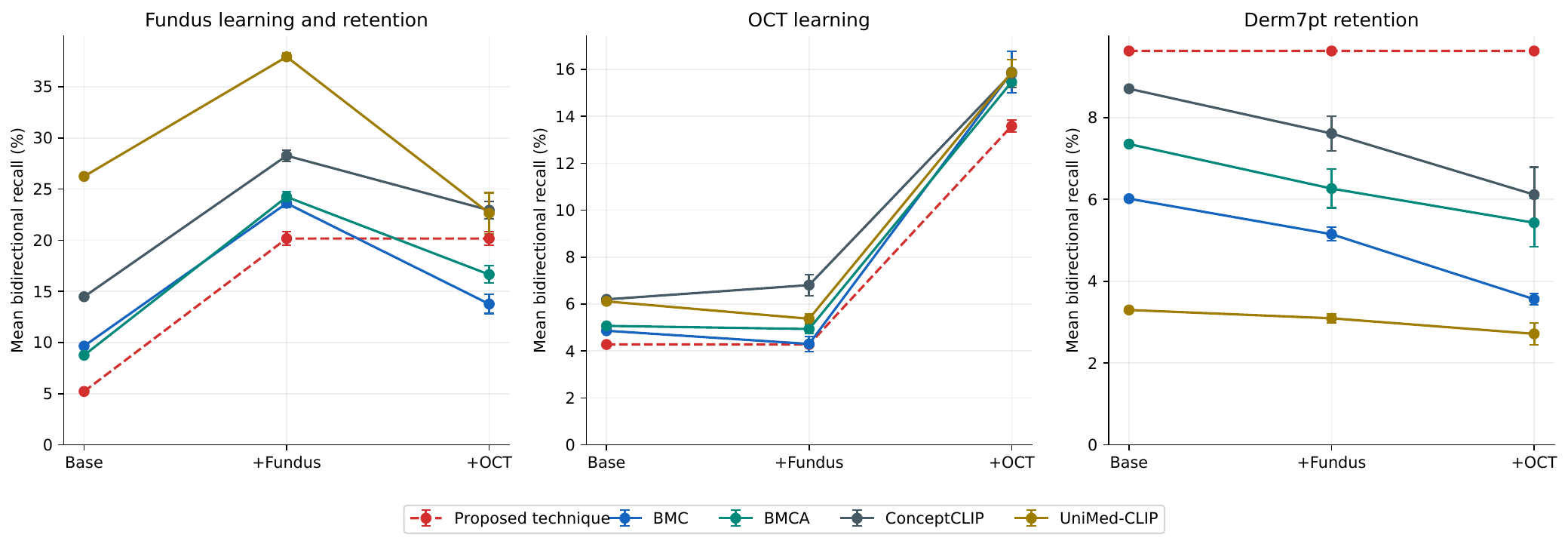}
\caption{\textbf{Learning new modalities while retaining existing capabilities.} Models use $20{,}000$ fundus pairs followed by $16{,}000$ OCT pairs. Left: fundus learning and retention. Middle: OCT learning. Right: retention on Derm7pt. Points show mean bidirectional recall and error bars show standard deviation across three training seeds. Baselines use LoRA; the proposed technique trains new experts with existing pathways frozen. This experiment uses training subsets, not the full training corpus.}
\label{fig:matched_seq}
\end{figure}

\subsection{Full-data expert assembly}\label{sec:fundus}
The assembled experts learn the new modalities: in Tables~\ref{tab:zs} and~\ref{tab:lp},
the proposed technique\ is best or tied-best on all five fundus/OCT rows. We also measure pathology retention on
PathGen-600. The proposed technique\ remains at $25.75\%$ mean bidirectional recall after both additions, with
byte-identical saved image and text features. Four baselines adapted by sequential LoRA lose
$2.97$--$6.11$ points (Table~\ref{tab:pathgen_lora}). This comparison is descriptive: the proposed technique's
full-data experts are trained independently and assembled, whereas the baselines use
sequential LoRA on matched-data subsets. The matched-data results above and Appendix~\ref{app:matched} provide the sequential comparison and per-stage results.

\begin{table}[t]
\centering\small
\caption{\textbf{PathGen pathology retention} after fundus$\to$OCT adaptation. Baselines
use sequential LoRA and report the mean over three seeds; the proposed technique\ uses one assembled-expert
sequence. Full trajectories are in Figure~\ref{fig:pathgen_trajectory}.}
\label{tab:pathgen_lora}
\begin{tabular}{lr}
  \toprule
  Model (adaptation) & $\Delta$ PathGen \\
  \midrule
  \ours{Proposed technique\ ViT-L (routed expert)} & \zero \\
  BMCA $+$ LoRA    & $-3.22$ \\
  BMC $+$ LoRA   & $-6.11$ \\
  UniMed-CLIP $+$ LoRA  & $-2.97$ \\
  ConceptCLIP $+$ LoRA  & $-5.39$ \\
  \bottomrule
\end{tabular}
\end{table}

\subsection{Checks on existing representations}\label{sec:forget}
We add independently trained fundus and OCT experts to the frozen ViT-Large core and test an
old modality after each addition. On a balanced $4096$-image PCam probe, the core,
$+$fundus, and $+$fundus$+$OCT checkpoints give identical AUROC ($0.8691$), balanced accuracy
($64.26\%$), features, cosine scores, and predictions. Parameter and seven-modality retention
checks are reported in Appendices~\ref{app:params} and~\ref{app:retention7}.

The full-data assembly also retains the evaluated retrieval results across
all seven base modalities. Appendix~\ref{app:retention7} reports the individual
benchmarks and feature checks. These assembly checks complement the matched-data
sequence; they involve independently trained experts and should be read as a
separate experiment.

\section{Ablations and supporting analyses}\label{sec:abl}
\subsection{Modality experts under a matched training budget}
A controlled ViT-B comparison uses identical shared random initialization,
training data, tokenizer, optimization settings, and training seed. The routed
model improves $15$ of $18$ benchmarks relative to the variant without
modality experts, ties on BreaKHis zero-shot classification, and scores lower
on Pneumonia zero-shot classification and HAM10000 linear probing.
Mean gains are $0.0433$ in zero-shot balanced accuracy, $0.0201$ in
linear-probe balanced accuracy, and $0.1330$ percentage points in retrieval.
Appendix~\ref{app:ablation} reports all benchmark rows and settings.

Both models train for five epochs per stage and use the final checkpoint,
without early stopping or validation-based checkpoint selection. This comparison
supports the contribution of modality-specific capacity under a limited,
matched training budget. It does not establish the same effect at convergence
or isolate routing from the additional expert parameters.

\subsection{Multi-teacher visual tokenization}
The tokenizer comparison holds the ViT-B architecture, initialization, data
subsets, and training schedules fixed while changing the visual teachers.
Multi-teacher supervision improves mean zero-shot balanced accuracy from
$32.74\%$ to $35.60\%$, linear-probe balanced accuracy from $74.56\%$ to
$76.34\%$, and retrieval from $2.17\%$ to $2.35\%$.
Appendix~\ref{app:tokenizer-ablation} provides the full results. Improvements
vary by benchmark, including two linear-probe decreases. The experiment uses
one seed and a short schedule, so these means do not establish statistical
significance or the effect under full-data training.

\subsection{Qualitative evidence}
The appendix collects two kinds of visual evidence. Before-and-after retrieval
examples in Appendix~\ref{app:qual} show individual ranking changes through
fundus and OCT integration. Similarity-gradient visualizations in
Appendices~\ref{app:saliency} and~\ref{app:comparative-saliency} compare spatial
responses to finding descriptions and paired reports. For the linear-probe
examples, the saliency target is image--prompt similarity rather than the
image-only classifier's class score. These selected visualizations complement
the quantitative evaluations and do not measure localization performance.

\section{Discussion and limitations}\label{sec:discussion}
The results support a distinction between acquiring new representations and
preserving an existing embedding function. New experts provide trainable
capacity for an added modality, while frozen original paths preserve earlier
embeddings under the conditions in Proposition~\ref{prop:zero}. Whether the
new expert reaches the strongest task performance remains an empirical
question. In the matched-data experiment, the proposed technique\ learns both fundus and OCT
while retaining earlier tasks, but its OCT retrieval score is below those of
the adapted baselines. Retention and new-task performance therefore need to be
assessed together.

The broader baseline comparisons do not isolate a single source of the
performance differences. Models differ in architecture, pretraining data,
parameter count, and optimization. The proposed technique\ uses a large bank of stored experts,
and the matched-data experiment does not match trainable parameter counts or
computational cost. The ViT-B routing and tokenizer experiments provide more
controlled evidence, but they use limited subsets and short training schedules.
Their results do not establish the size of the same effects after converged
ViT-L training.

Our expansion rule assumes that the modality is supplied at inference and
that existing preprocessing, persistent state, routing, and inference
operations remain fixed. Unchanged embeddings imply unchanged task outputs
only when scoring rules, candidate sets, and downstream classifiers are also
fixed. The feature-identity checks cover evaluated inputs. A change to the
assembled model's learned logit-scale parameter is separate from those checks,
so preserved features do not establish identity of every scaled logit.

Evaluation also has limits. Some linear-probe splits operate at the image or
data-row level rather than the patient or slide level, including HEST's
image-level split. The retention benchmarks sample capabilities across the
base modalities, rather than testing every possible downstream use. The
qualitative examples visualize similarity responses without a quantitative
localization evaluation. These experiments assess representation learning;
clinical use would require prospective validation in the intended setting.

\section{Related work}\label{sec:related}
\textbf{Biomedical vision--language models.} CLIP \citep{radford2021clip} inspired biomedical
vision--language models such as BMC, BMCA, and ConceptCLIP. These models support
text-prompted classification and image--text retrieval, but do not explicitly separate the
computation paths of different imaging modalities. Vision-only models such as MerMED-FM
cover several modalities, but lack a text interface. Another strategy is to connect existing
domain-specific models: CTF \citep{chen2026ctf} integrates separately trained
radiology and pathology encoders through clinical concepts and cross-domain
prompting. In contrast, the proposed technique\ combines broad modality coverage and a language
interface within one routed encoder.

\textbf{Multiway pretraining and visual tokenization.} The proposed technique\ builds on the Multiway
Transformer of BEiT-3 \citep{wang2023beit3} and the VQ-KD tokenizer of BEiT-v2
\citep{peng2022beitv2}. MUSK \citep{xiang2025musk} applies this framework to pathology using
masked image and language modeling followed by contrastive alignment. The proposed technique\ extends the design
to multiple modalities through one tokenizer supervised by six visual teachers and visual
FFNs with modality-specific experts. The teachers supervise tokenizer training; after the
tokenizer is frozen, the VLM receives no additional teacher-feature or teacher-logit loss.

\textbf{Mixture-of-Experts and continual expansion.} MoE models route inputs through
selected parameters \citep{shazeer2017moe,fedus2022switch}. The proposed technique\ uses deterministic
modality routing and therefore needs no load-balancing loss. Unlike learned token routing,
each image uses a known modality label. LoRA \citep{hu2022lora}, adapters
\citep{houlsby2019adapters}, regularization \citep{kirkpatrick2017ewc}, and rehearsal can
reduce interference, while progressive networks preserve earlier computations through model
expansion \citep{rusu2016progressive}. MedQwen \citep{medqwen2026} routes low-rank experts in
a generative medical VLM; the proposed technique\ instead uses full-rank modality experts in a contrastive
encoder. Its preservation guarantee follows from keeping existing computation paths fixed,
not from routing alone.

\section{Conclusion}\label{sec:conclusion}
The proposed technique\ combines masked pretraining on unpaired medical images with contrastive
image--text alignment and expandable modality experts. The resulting model
supports nine imaging modalities and achieves the best or tied-best result in
$24$ of $25$ classification, linear-probe, and retrieval settings. In a
separate matched-data sequential experiment, new experts improve fundus and
OCT retrieval while the evaluated preceding fundus and pathology tasks remain
unchanged. The preservation argument follows from retaining the original
computation paths under explicit inference conditions. Together, these results
support modular expert expansion as a way to extend medical vision--language
representations while retaining existing embeddings.

\clearpage
\bibliographystyle{plainnat}
\bibliography{references}

@inproceedings{radford2021clip,
  title     = {Learning Transferable Visual Models From Natural Language Supervision},
  author    = {Radford, Alec and Kim, Jong Wook and Hallacy, Chris and Ramesh, Aditya
               and Goh, Gabriel and Agarwal, Sandhini and Sastry, Girish and Askell, Amanda
               and Mishkin, Pamela and Clark, Jack and Krueger, Gretchen and Sutskever, Ilya},
  booktitle = {International Conference on Machine Learning (ICML)},
  year      = {2021}
}

@article{zhang2023biomedclip,
  title   = {BiomedCLIP: a multimodal biomedical foundation model pretrained from
             fifteen million scientific image--text pairs},
  author  = {Zhang, Sheng and Xu, Yanbo and Usuyama, Naoto and Xu, Hanwen and Bagga, Jaspreet
             and Tinn, Robert and Preston, Sam and Rao, Rajesh and Wei, Mu and Valluri, Naveen
             and Wong, Cliff and Tupini, Andrea and Wang, Yu and Mazzola, Matt and Shukla, Swadheen
             and Liden, Lars and Gao, Jianfeng and Lungren, Matthew P. and Naumann, Tristan
             and Wang, Sheng and Poon, Hoifung},
  journal = {NEJM AI},
  year    = {2025}
}

@inproceedings{biomedica2024,
  title     = {BIOMEDICA: An Open Biomedical Image--Caption Archive, Dataset,
               and Vision--Language Models Derived from Scientific Literature},
  author    = {Lozano, Alejandro and Sun, Min Woo and Burgess, James and others},
  booktitle = {IEEE/CVF Conference on Computer Vision and Pattern Recognition (CVPR)},
  year      = {2025}
}

@article{nie2026conceptclip,
  title={An explainable biomedical foundation model via large-scale concept-enhanced vision-language pre-training},
  author={Nie, Yuxiang and He, Sunan and Bie, Yequan and Wang, Yihui and Chen, Zhixuan and Yang, Shu and Cai, Zhiyuan and Wu, Linshan and Wang, Hongmei and Wang, Xi and Cheng, Ngai Shing and Luo, Luyang and Wu, Mingxiang and Jin, Haibo and Wu, Xian and Chan, Ronald Cheong Kin and Lau, Yuk Ming and Zhang, Zhengyu and Xiao, Sushan and Yang, Can and Zhao, Yinghua and Duan, Xiaohui and Zhang, Li and Liang, Li and Zheng, Yefeng and Rajpurkar, Pranav and Chen, Hao},
  journal={Nature Biomedical Engineering},
  year={2026},
}

@article{xiang2025musk,
  title   = {A vision--language foundation model for precision oncology},
  author  = {Xiang, Jinxi and Wang, Xiyue and Zhang, Xiaoming and Xi, Yinghua
             and Eweje, Feyisope and Chen, Yijiang and Li, Yuchen and Bergstrom, Colin
             and Gopaulchan, Matthew and Kim, Ted and others},
  journal = {Nature},
  volume  = {638},
  pages   = {769--778},
  year    = {2025}
}

@article{zhou2026mermed,
  title={MerMED-FM: Multimodal, Multi-Disease Medical Imaging Foundation Model},
  author={Zhou, Yang and Quek, Chrystie Wan Ning and Zhou, Jun and Wang, Yan and Bai, Yang and Gutierrez, Laura and Ke, Yuhe and Yao, Jie and Teo, Zhen Ling and Ting, Darren Shu Jeng and others},
  journal={The Lancet Digital Health},
  volume={8},
  number={7},
  year={2026},
  publisher={Elsevier}
}

@inproceedings{medqwen2026,
  title     = {Sparse Spectral LoRA: Routed Experts for Medical VLMs},
  author    = {Nejati Manzari, Omid and Asgariandehkordi, Hojat and Koleilat, Taha
               and Xiao, Yiming and Rivaz, Hassan},
  booktitle = {Proceedings of the IEEE/CVF Conference on Computer Vision and
               Pattern Recognition (CVPR)},
  pages     = {35351--35362},
  year      = {2026}
}

@inproceedings{wang2023beit3,
  title     = {Image as a Foreign Language: BEiT Pretraining for Vision and
               Vision-Language Tasks},
  author    = {Wang, Wenhui and Bao, Hangbo and Dong, Li and Bjorck, Johan and Peng, Zhiliang
               and Liu, Qiang and Aggarwal, Kriti and Mohammed, Owais Khan and Singhal, Saksham
               and Som, Subhojit and Wei, Furu},
  booktitle = {IEEE/CVF Conference on Computer Vision and Pattern Recognition (CVPR)},
  year      = {2023}
}

@inproceedings{peng2022beitv2,
  title     = {BEiT v2: Masked Image Modeling with Vector-Quantized Visual Tokenizers},
  author    = {Peng, Zhiliang and Dong, Li and Bao, Hangbo and Ye, Qixiang and Wei, Furu},
  booktitle = {arXiv preprint arXiv:2208.06366},
  year      = {2022}
}

@inproceedings{shazeer2017moe,
  title     = {Outrageously Large Neural Networks: The Sparsely-Gated
               Mixture-of-Experts Layer},
  author    = {Shazeer, Noam and Mirhoseini, Azalia and Maziarz, Krzysztof and Davis, Andy
               and Le, Quoc and Hinton, Geoffrey and Dean, Jeff},
  booktitle = {International Conference on Learning Representations (ICLR)},
  year      = {2017}
}

@article{fedus2022switch,
  title   = {Switch Transformers: Scaling to Trillion Parameter Models with
             Simple and Efficient Sparsity},
  author  = {Fedus, William and Zoph, Barret and Shazeer, Noam},
  journal = {Journal of Machine Learning Research},
  volume  = {23},
  number  = {120},
  pages   = {1--39},
  year    = {2022}
}

@inproceedings{hu2022lora,
  title     = {LoRA: Low-Rank Adaptation of Large Language Models},
  author    = {Hu, Edward J. and Shen, Yelong and Wallis, Phillip and Allen-Zhu, Zeyuan
               and Li, Yuanzhi and Wang, Shean and Wang, Lu and Chen, Weizhu},
  booktitle = {International Conference on Learning Representations (ICLR)},
  year      = {2022}
}

@inproceedings{houlsby2019adapters,
  title     = {Parameter-Efficient Transfer Learning for NLP},
  author    = {Houlsby, Neil and Giurgiu, Andrei and Jastrzebski, Stanislaw and Morrone, Bruna
               and De Laroussilhe, Quentin and Gesmundo, Andrea and Attariyan, Mona
               and Gelly, Sylvain},
  booktitle = {International Conference on Machine Learning (ICML)},
  year      = {2019}
}

@article{rusu2016progressive,
  title   = {Progressive Neural Networks},
  author  = {Rusu, Andrei A. and Rabinowitz, Neil C. and Desjardins, Guillaume
             and Soyer, Hubert and Kirkpatrick, James and Kavukcuoglu, Koray
             and Pascanu, Razvan and Hadsell, Raia},
  journal = {arXiv preprint arXiv:1606.04671},
  year    = {2016}
}

@article{kirkpatrick2017ewc,
  title   = {Overcoming Catastrophic Forgetting in Neural Networks},
  author  = {Kirkpatrick, James and Pascanu, Razvan and Rabinowitz, Neil and Veness, Joel
             and Desjardins, Guillaume and Rusu, Andrei A. and Milan, Kieran and Quan, John
             and Ramalho, Tiago and Grabska-Barwinska, Agnieszka and others},
  journal = {Proceedings of the National Academy of Sciences},
  volume  = {114},
  number  = {13},
  pages   = {3521--3526},
  year    = {2017}
}

@article{chen2024uni,
  title   = {Towards a general-purpose foundation model for computational pathology},
  author  = {Chen, Richard J. and Ding, Tong and Lu, Ming Y. and Williamson, Drew F. K.
             and Jaume, Guillaume and Song, Andrew H. and Chen, Bowen and Zhang, Andrew
             and Shao, Daniel and Shaban, Muhammad and others},
  journal = {Nature Medicine},
  volume  = {30},
  pages   = {850--862},
  year    = {2024}
}

@article{meddinov3,
  title   = {MedDINOv3: How to adapt vision foundation models for medical
             image segmentation?},
  author  = {Li, Yuheng and Wu, Yizhou and Lai, Yuxiang and Hu, Mingzhe and
             Yang, Xiaofeng},
  journal = {arXiv preprint arXiv:2509.02379},
  year    = {2025}
}

@article{tschandl2018ham10000,
  title   = {The HAM10000 dataset, a large collection of multi-source dermatoscopic
             images of common pigmented skin lesions},
  author  = {Tschandl, Philipp and Rosendahl, Cliff and Kittler, Harald},
  journal = {Scientific Data},
  volume  = {5},
  pages   = {180161},
  year    = {2018}
}

@inproceedings{jaume2024hest,
  title     = {HEST-1k: A Dataset for Spatial Transcriptomics and Histology Image Analysis},
  author    = {Jaume, Guillaume and Doucet, Paul and Song, Andrew H. and Lu, Ming Y.
               and Almagro-P{\'e}rez, Cristina and Wagner, Sophia J. and Vaidya, Anurag J.
               and Chen, Richard J. and Williamson, Drew F. K. and Kim, Ahrong and Mahmood, Faisal},
  booktitle = {Advances in Neural Information Processing Systems (NeurIPS)},
  year      = {2024}
}

@inproceedings{xie2025medtrinity,
  title     = {MedTrinity-25M: A Large-scale Multimodal Dataset with Multigranular
               Annotations for Medicine},
  author    = {Xie, Yunfei and Zhou, Ce and Gao, Lang and Wu, Juncheng and Li, Xianhang
               and Zhou, Hong-Yu and Liu, Sheng and Xing, Lei and Zou, James
               and Xie, Cihang and Zhou, Yuyin},
  booktitle = {International Conference on Learning Representations (ICLR)},
  year      = {2025}
}

@article{soares2020sarscov2,
  title   = {{SARS-CoV-2} CT-scan dataset: A large dataset of real
             patients CT scans for {SARS-CoV-2} identification},
  author  = {Soares, Eduardo and Angelov, Plamen and Biaso, Sarah
             and Higa Froes, Michele and Kanda Abe, Daniel},
  journal = {medRxiv},
  year    = {2020},
  doi     = {10.1101/2020.04.24.20078584}
}

@inproceedings{chen2026ctf,
  title     = {Bridging Radiology and Pathology Foundation Models via Concept-Based
               Multimodal Co-Adaptation},
  author    = {Chen, Yihang and Huang, Yanyan and Wang, Fuying and Yeung, Maximus and
               Jiang, Yuming and Wang, Shujun and Yu, Lequan},
  booktitle = {International Conference on Learning Representations (ICLR)},
  year      = {2026}
}

@article{khattak2024unimed,
  title   = {UniMed-CLIP: Towards a Unified Image-Text Pretraining Paradigm for Diverse
             Medical Imaging Modalities},
  author  = {Khattak, Muhammad Uzair and Kunhimon, Shahina and Naseer, Muzammal and
             Khan, Salman and Khan, Fahad Shahbaz},
  journal = {arXiv preprint arXiv:2412.10372},
  year    = {2024}
}

@article{weinstein2013cancer,
  title   = {The Cancer Genome Atlas Pan-Cancer Analysis Project},
  author  = {Weinstein, John N. and Collisson, Eric A. and Mills, Gordon B. and Shaw, Kenna R.
             and Ozenberger, Brad A. and Ellrott, Kyle and Shmulevich, Ilya and Sander, Chris
             and Stuart, Joshua M.},
  journal = {Nature Genetics},
  volume  = {45},
  pages   = {1113--1120},
  year    = {2013}
}

@inproceedings{ikezogwo2023quilt,
  title     = {Quilt-1M: One Million Image--Text Pairs for Histopathology},
  author    = {Ikezogwo, Wisdom Oluchi and Seyfioglu, Mehmet Saygin and Ghezloo, Fatemeh
               and Geva, Dylan Stefan Chan and Mohammed, Fatwir Sheikh and Anand, Pavan Kumar
               and Krishna, Ranjay and Shapiro, Linda},
  booktitle = {Advances in Neural Information Processing Systems (NeurIPS Datasets and
               Benchmarks Track)},
  volume    = {36},
  year      = {2023}
}

@inproceedings{sun2024pathasst,
  title     = {PathAsst: A Generative Foundation AI Assistant Towards Artificial General
               Intelligence of Pathology},
  author    = {Sun, Yuxuan and Zhu, Chenglu and Zheng, Sunyi and Zhang, Kai and Sun, Lin
               and Shui, Zhongyi and Zhang, Yunlong and Li, Honglin and Yang, Lin},
  booktitle = {Proceedings of the AAAI Conference on Artificial Intelligence},
  volume    = {38},
  pages     = {5034--5042},
  year      = {2024}
}

@inproceedings{irvin2019chexpert,
  title     = {CheXpert: A Large Chest Radiograph Dataset with Uncertainty Labels
               and Expert Comparison},
  author    = {Irvin, Jeremy and Rajpurkar, Pranav and Ko, Michael and Yu, Yifan
               and Ciurea-Ilcus, Silviana and Chute, Chris and Marklund, Henrik
               and Haghgoo, Behzad and Ball, Robyn and Shpanskaya, Katie and Seekins, Jayne
               and Mong, David A. and Halabi, Safwan S. and Sandberg, Jesse K. and Jones, Ricky
               and Larson, David B. and Langlotz, Curtis P. and Patel, Bhavik N.
               and Lungren, Matthew P. and Ng, Andrew Y.},
  booktitle = {Proceedings of the AAAI Conference on Artificial Intelligence},
  volume    = {33},
  pages     = {590--597},
  year      = {2019}
}

@inproceedings{wang2017chestxray8,
  title     = {ChestX-ray8: Hospital-Scale Chest X-Ray Database and Benchmarks on
               Weakly-Supervised Classification and Localization of Common Thorax Diseases},
  author    = {Wang, Xiaosong and Peng, Yifan and Lu, Le and Lu, Zhiyong
               and Bagheri, Mohammadhadi and Summers, Ronald M.},
  booktitle = {Proceedings of the IEEE Conference on Computer Vision and Pattern
               Recognition (CVPR)},
  year      = {2017}
}

@article{borgli2020hyperkvasir,
  title   = {HyperKvasir, a Comprehensive Multi-class Image and Video Dataset for
             Gastrointestinal Endoscopy},
  author  = {Borgli, Hanna and Thambawita, Vajira and Smedsrud, Pia H. and Hicks, Steven
             and Jha, Debesh and Eskeland, Sigrun L. and Randel, Kristin Ranheim
             and Pogorelov, Konstantin and Lux, Mathias and Nguyen, Duc Tien Dang
             and Johansen, Dag and Griwodz, Carsten and Stensland, H{\aa}kon K.
             and Garcia-Ceja, Enrique and Schmidt, Peter T. and Hammer, Hugo L.
             and Riegler, Michael A. and Halvorsen, P{\aa}l and de Lange, Thomas},
  journal = {Scientific Data},
  volume  = {7},
  pages   = {283},
  year    = {2020}
}

@article{smedsrud2021kvasircapsule,
  title   = {Kvasir-Capsule, a Video Capsule Endoscopy Dataset},
  author  = {Smedsrud, Pia H. and Thambawita, Vajira and Hicks, Steven A.
             and Gjestang, Henrik and Nedrejord, Oda Olsen and N{\ae}ss, Espen
             and Borgli, Hanna and Jha, Debesh and Berstad, Tor Jan Derek
             and Eskeland, Sigrun L. and Lux, Mathias and Espeland, H{\aa}vard
             and Petlund, Andreas and Nguyen, Duc Tien Dang and Garcia-Ceja, Enrique
             and Johansen, Dag and Schmidt, Peter T. and Toth, Ervin and Hammer, Hugo L.
             and de Lange, Thomas and Riegler, Michael A. and Halvorsen, P{\aa}l},
  journal = {Scientific Data},
  volume  = {8},
  pages   = {142},
  year    = {2021}
}

@article{nwoye2022rendezvous,
  title   = {Rendezvous: Attention Mechanisms for the Recognition of Surgical Action
             Triplets in Endoscopic Videos},
  author  = {Nwoye, Chinedu Innocent and Yu, Tong and Gonzalez, Cristians
             and Seeliger, Barbara and Mascagni, Pietro and Mutter, Didier
             and Marescaux, Jacques and Padoy, Nicolas},
  journal = {Medical Image Analysis},
  volume  = {78},
  pages   = {102433},
  year    = {2022}
}

@article{ruckert2024rocov2,
  title   = {{ROCOv2}: Radiology Objects in {COntext} Version 2, an Updated
             Multimodal Image Dataset},
  author  = {R{\"u}ckert, Johannes and Bloch, Louise and Br{\"u}ngel, Raphael
             and Idrissi-Yaghir, Ahmad and Sch{\"a}fer, Henning and Schmidt, Cynthia S.
             and Koitka, Sven and Pelka, Obioma and Ben Abacha, Asma
             and Seco de Herrera, Alba G. and M{\"u}ller, Henning and Horn, Peter A.
             and Nensa, Felix and Friedrich, Christoph M.},
  journal = {Scientific Data},
  volume  = {11},
  pages   = {688},
  year    = {2024},
  doi     = {10.1038/s41597-024-03496-6}
}

@inproceedings{yao2025retcop,
  title     = {Continual Retinal Vision-Language Pre-training upon Incremental
               Imaging Modalities},
  author    = {Yao, Yuang and Wu, Ruiqi and Zhou, Yi and Zhou, Tao},
  booktitle = {International Conference on Medical Image Computing and
               Computer-Assisted Intervention (MICCAI)},
  year      = {2025}
}

@article{boers2024gastronet,
  title   = {Foundation models in gastrointestinal endoscopic {AI}: Impact of architecture, pre-training approach and data efficiency},
  author  = {Boers, T. G. W. and Fockens, K. N. and van der Putten, J. A. and Jaspers, T. J. M. and Kusters, C. H. J. and Jukema, J. B. and Jong, M. R. and Struyvenberg, M. R. and de Groof, J. and Bergman, J. J. and de With, P. H. N. and van der Sommen, F.},
  journal = {Medical Image Analysis},
  volume  = {98},
  pages   = {103298},
  year    = {2024},
  doi     = {10.1016/j.media.2024.103298}
}

@inproceedings{gautam2025kvasir,
  title     = {Kvasir-VQA-x1: A multimodal dataset for medical reasoning and robust {MedVQA} in gastrointestinal endoscopy},
  author    = {Gautam, Sushant and Riegler, Michael A. and Halvorsen, P{\aa}l},
  booktitle = {Data Engineering in Medical Imaging},
  publisher = {Springer},
  year      = {2025},
  doi       = {10.1007/978-3-032-08009-7_6}
}

@misc{gow2026mimicivecho,
  title        = {{MIMIC-IV-ECHO}: Echocardiogram matched subset, version 1.0},
  author       = {Gow, Brian and others},
  year         = {2026},
  howpublished = {PhysioNet},
  doi          = {10.13026/nrjh-5r77}
}

@inproceedings{huang2021deepopht,
  title     = {{DeepOpht}: Medical report generation for retinal images via deep models and visual explanation},
  author    = {Huang, Jia-Hong and others},
  booktitle = {IEEE Winter Conference on Applications of Computer Vision},
  pages     = {2442--2452},
  year      = {2021},
}

@misc{isic_archive,
  title        = {{ISIC} Archive},
  author       = {{International Skin Imaging Collaboration}},
  year         = {2026},
  note         = {Accessed September 20, 2026},
  howpublished = {\url{https://www.isic-archive.com/blank-1}}
}

@article{jiao2024usfm,
  title   = {{USFM}: A universal ultrasound foundation model generalized to tasks and organs towards label efficient image analysis},
  author  = {Jiao, Jing and Zhou, Jin and Li, Xiaokang and Xia, Menghua and Huang, Yi and Huang, Lihong and Wang, Na and Zhang, Xiaofan and Zhou, Shichong and Wang, Yuanyuan and Guo, Yi},
  journal = {arXiv preprint arXiv:2401.00153},
  year    = {2024}
}

@inproceedings{jin2026ultrasoundclip,
  title     = {Ultrasound-{CLIP}: Semantic-aware contrastive pre-training for ultrasound image--text understanding},
  author    = {Jin, Jiayun and Chai, Haolong and Huang, Xueying and Guo, Xiaoqing and Zheng, Zengwei and Zhou, Zhan and Wang, Junmei and Wang, Xinyu and Liu, Jie and Zhou, Binbin},
  booktitle = {IEEE/CVF Conference on Computer Vision and Pattern Recognition},
  year      = {2026},
  note      = {Dataset: \url{https://huggingface.co/datasets/JJY-0823/US-365K}}
}

@misc{johnson2024mimiccxr,
  title        = {{MIMIC-CXR-JPG}---chest radiographs with structured labels, version 2.1.0},
  author       = {Johnson, Alistair and Lungren, Matthew and Peng, Yifan and Lu, Zhiyong and Mark, Roger and Berkowitz, Steven and Horng, Steven},
  year         = {2024},
  howpublished = {PhysioNet},
  doi          = {10.13026/jsn5-t979}
}

@inproceedings{li2026gmai,
  title     = {{GMAI-VL} \& {GMAI-VL-5.5M}: A large vision--language model and a comprehensive multimodal dataset towards general medical {AI}},
  author    = {Li, T. and others},
  booktitle = {Proceedings of the AAAI Conference on Artificial Intelligence},
  volume    = {40},
  pages     = {23177--23185},
  year      = {2026},
  url       = {https://github.com/uni-medical/GMAI-VL}
}

@inproceedings{lin2023pmc,
  title     = {{PMC-CLIP}: Contrastive language--image pre-training using biomedical documents},
  author    = {Lin, Weixiong and Zhao, Ziheng and Zhang, Xiaoman and Wu, Chaoyi and Zhang, Ya and Wang, Yanfeng and Xie, Weidi},
  booktitle = {Medical Image Computing and Computer Assisted Intervention},
  pages     = {525--536},
  year      = {2023},
  url       = {https://arxiv.org/abs/2303.07240}
}

@inproceedings{ma2021ldpolyp,
  title     = {{LDPolypVideo} benchmark: A large-scale colonoscopy video dataset of diverse polyps},
  author    = {Ma, Yinghao and Chen, Xianhao and Cheng, Kai and Li, Yang and Sun, Bin},
  booktitle = {Medical Image Computing and Computer Assisted Intervention},
  year      = {2021},
  url       = {https://miccai2021.org/openaccess/paperlinks/2021/09/01/268-Paper1167.html}
}

@article{marzola2021stmus,
  title   = {Deep learning segmentation of transverse musculoskeletal ultrasound images for neuromuscular disease assessment},
  author  = {Marzola, Francesco and others},
  journal = {Computers in Biology and Medicine},
  volume  = {135},
  pages   = {104623},
  year    = {2021},
  doi     = {10.1016/j.compbiomed.2021.104623}
}

@article{menze2015brats,
  title   = {The multimodal brain tumor image segmentation benchmark ({BRATS})},
  author  = {Menze, Bjoern H. and others},
  journal = {IEEE Transactions on Medical Imaging},
  volume  = {34},
  number  = {10},
  pages   = {1993--2024},
  year    = {2015},
  doi     = {10.1109/TMI.2014.2377694}
}

@article{perezgarcia2024raddino,
  title   = {{RAD-DINO}: Exploring scalable medical image encoders beyond text supervision},
  author  = {P{\'e}rez-Garc{\'i}a, Fernando and others},
  journal = {arXiv preprint arXiv:2401.10815},
  year    = {2024}
}

@misc{pmc_oa,
  title        = {{PMC} Open Access Subset},
  author       = {{National Library of Medicine}},
  year         = {2026},
  note         = {Accessed September 20, 2026},
  howpublished = {\url{https://pmc.ncbi.nlm.nih.gov/tools/openftlist/}}
}

@misc{qiaiuncc_oct_summary,
  title        = {{OCT-summary-Dataset}},
  author       = {{QIAIUNCC}},
  year         = {2026},
  note         = {Accessed September 20, 2026},
  howpublished = {\url{https://huggingface.co/datasets/QIAIUNCC/OCT-summary-Dataset}}
}

@inproceedings{tyagi2024nerve,
  title     = {Nerve block target localization and needle guidance for autonomous robotic ultrasound guided regional anesthesia},
  author    = {Tyagi, Ayush and others},
  booktitle = {IEEE/RSJ International Conference on Intelligent Robots and Systems},
  year      = {2024},
  url       = {https://arxiv.org/abs/2308.03717},
  note      = {Dataset: \url{https://github.com/Regional-US/brachial_plexus}}
}

@inproceedings{wang2021echocp,
  title     = {{EchoCP}: An echocardiography dataset in contrast transthoracic echocardiography for patent foramen ovale diagnosis},
  author    = {Wang, T. and Li, Z. and Huang, M. and Zhuang, J. and Bi, S. and Zhang, J. and Shi, Y. and Fei, H. and Xu, X.},
  booktitle = {Medical Image Computing and Computer Assisted Intervention},
  year      = {2021},
  url       = {https://arxiv.org/abs/2105.08267}
}

@inproceedings{wang2023endofm,
  title     = {Foundation model for endoscopy video analysis via large-scale self-supervised pre-train},
  author    = {Wang, Zheng and Liu, Chang and Zhang, Shijie and Dou, Qi},
  booktitle = {Medical Image Computing and Computer Assisted Intervention},
  pages     = {101--111},
  year      = {2023},
  url       = {https://github.com/med-air/Endo-FM}
}

@article{gamper2020pannuke,
  title={Pannuke dataset extension, insights and baselines},
  author={Gamper, Jevgenij and Koohbanani, Navid Alemi and Benes, Ksenija and Graham, Simon and Jahanifar, Mostafa and Khurram, Syed Ali and Azam, Ayesha and Hewitt, Katherine and Rajpoot, Nasir},
  journal={arXiv preprint arXiv:2003.10778},
  year={2020}
}

@article{butsanets2025radimagenet,
  title={RadImageNet-VQA: A Large-Scale CT and MRI Dataset for Radiologic Visual Question Answering},
  author={Butsanets, L{\'e}o and Corbi{\`e}re, Charles and Khlaut, Julien and Manceron, Pierre and Dancette, Corentin},
  journal={arXiv preprint arXiv:2512.17396},
  year={2025}
}

@article{demner2016preparing,
  title={Preparing a collection of radiology examinations for distribution and retrieval},
  author={Demner-Fushman, Dina and Kohli, Marc D and Rosenman, Marc B and Shooshan, Sonya E and Rodriguez, Laritza and Antani, Sameer and Thoma, George R and McDonald, Clement J},
  journal={Journal of the American Medical Informatics Association},
  volume={23},
  number={2},
  pages={304--310},
  year={2016},
  publisher={Oxford University Press}
}

@misc{forithmus2026mrrate,
  title={{MR-RATE}: A Vision-Language Foundation Model and Dataset for Magnetic Resonance Imaging},
  author={{Forithmus Team} and {University of Zurich} and {NVIDIA}},
  howpublished={\url{https://huggingface.co/datasets/Forithmus/MR-RATE}},
  year={2026},
  note={Accessed: 2026}
}

@inproceedings{sun2025pathgen,
  title={PathGen-1.6M: 1.6 Million Pathology Image-text Pairs Generation through Multi-agent Collaboration},
  author={Sun, Yuxuan and others},
  booktitle={International Conference on Learning Representations (ICLR)},
  year={2025}
}

@article{kawahara2019seven,
  title={Seven-point checklist and skin lesion classification using multitask multimodal neural nets},
  author={Kawahara, Jeremy and Daneshvar, Sara and Argenziano, Giuseppe and Hamarneh, Ghassan},
  journal={IEEE Journal of Biomedical and Health Informatics},
  volume={23},
  number={2},
  pages={538--546},
  year={2019},
  publisher={IEEE}
}

@article{yu2026endoclip,
  title         = {A Report-Grounded Vision--Language Foundation Model
                   for Colonoscopy from 280,000 Routine Reports},
  author        = {Yu, Jia and Zhu, Yan and He, Yili and Wang, Zilong
                   and Jiang, Xinyang and Fu, Peiyao and Yang, Ruijie
                   and Chen, Tianyi and Li, Siyuan and Wang, Zhihua
                   and Wu, Fei and Li, Quanlin and Yang, Xian
                   and Zhou, Pinghong and Wang, Shuo},
  journal       = {arXiv preprint arXiv:2607.28466},
  year          = {2026},
  eprint        = {2607.28466},
  archivePrefix = {arXiv},
  primaryClass  = {cs.CV},
  url           = {https://arxiv.org/abs/2607.28466}
}

@article{li2025ultrasound,
  title   = {Ultrasound Report Generation With Cross-Modality Feature Alignment via Unsupervised Guidance},
  author  = {Li, Jun and Su, Tongkun and Zhao, Baoliang and
             Lv, Faqin and Wang, Qiong and Navab, Nassir and
             Hu, Ying and Jiang, Zhongliang},
  journal = {IEEE Transactions on Medical Imaging},
  year    = {2025},
  volume  = {44},
  number  = {1},
  pages   = {19--30},
  doi     = {10.1109/TMI.2024.3424978}
}

@article{spanhol2016dataset,
  title={A dataset for breast cancer histopathological image classification},
  author={Spanhol, Fabio A and Oliveira, Luiz S and Petitjean, Caroline and Heutte, Laurent},
  journal={IEEE Transactions on Biomedical Engineering},
  volume={63},
  number={7},
  pages={1455--1462},
  year={2016},
  publisher={IEEE}
}

@article{borkowski2019lung,
  title={Lung and Colon Cancer Histopathological Image Dataset ({LC25000})},
  author={Borkowski, Andrew A. and Bui, Marilyn M. and Thomas, L. Brannon and Wilson, Catherine P. and DeLand, Lauren A. and Mastorides, Stephen M.},
  journal={arXiv preprint arXiv:1912.12142},
  year={2019}
}

@article{kermany2018identifying,
  title   = {Identifying Medical Diagnoses and Treatable Diseases
             by Image-Based Deep Learning},
  author  = {Kermany, Daniel S. and Goldbaum, Michael and
             Cai, Wenjia and others},
  journal = {Cell},
  year    = {2018},
  volume  = {172},
  number  = {5},
  pages   = {1122--1131.e9},
  doi     = {10.1016/j.cell.2018.02.010}
}

@article{metmer2024casia,
  title   = {An open chest {X-ray} dataset with benchmarks for
             automatic radiology report generation in {French}},
  author  = {Metmer, Hichem and Yang, Xiaoshan},
  journal = {Neurocomputing},
  year    = {2024},
  volume  = {609},
  pages   = {128478},
  doi     = {10.1016/j.neucom.2024.128478}
}
\clearpage
\appendix
\section*{Appendices}
The appendices provide data and implementation details, controlled comparisons,
adaptation and retention protocols, and qualitative image examples.
\section{Datasets and evaluation tasks}\label{app:leak}
We evaluate zero-shot classification, linear probing, and image--text retrieval across nine imaging modalities: X-ray, CT, MRI, ultrasound, pathology, dermatology, endoscopy, fundus photography, and OCT. Within each experiment, all models use the same evaluation set and scoring procedure.

\textbf{Training data.} The reported ViT-L model begins with a seven-modality core. We train this core on $19{,}393{,}075$ image rows for masked pretraining and $3{,}229{,}073$ image--text pairs for alignment. We then train the fundus and OCT experts using $88{,}143$ and $19{,}396$ additional image--text pairs, respectively, without further masked pretraining. Across core alignment and both extension stages, training uses $3{,}336{,}612$ pairs. Table~\ref{tab:training-data} gives the counts and sources for both training stages.

\begin{table}[p]\centering\small
\caption{Training data for the ViT-L core and modality extensions. Stage~1 uses image rows for masked pretraining, and Stage~2 uses image--text pairs for alignment. The pathology Stage~1 count is measured in extracted patches. A dash indicates that no additional Stage~1 training is used for an extension.}
\label{tab:training-data}
\setlength{\tabcolsep}{3pt}
\renewcommand{\arraystretch}{1.2}
\begin{tabularx}{\linewidth}{@{}lrr>{\raggedright\arraybackslash}X@{}}\toprule
\textbf{Modality} & \textbf{Stage 1 rows} & \textbf{Stage 2 pairs} & \textbf{Sources: Stage 1 $\rightarrow$ Stage 2} \\ \midrule
Pathology & 5,956,670\textsuperscript{$\dagger$} & 1,139,291 & TCGA patches~\citep{weinstein2013cancer} $\rightarrow$ Quilt-1M~\citep{ikezogwo2023quilt}; PathCap~\citep{sun2024pathasst} \\[2pt]
Endoscopy & 4,864,701 & 158,672 & Endo-FM private~\citep{wang2023endofm}; Hyper-Kvasir~\citep{borgli2020hyperkvasir}; Kvasir-Capsule~\citep{smedsrud2021kvasircapsule}; CholecT45~\citep{nwoye2022rendezvous}; LDPolypVideo~\citep{ma2021ldpolyp}; Colonoscopic $\rightarrow$ Kvasir-VQA-x1~\citep{gautam2025kvasir} \\[2pt]
MRI & 3,877,303 & 463,588 & MedTrinity MRI~\citep{xie2025medtrinity}, including BraTS~\citep{menze2015brats} $\rightarrow$ MedTrinity MRI~\citep{xie2025medtrinity} \\[2pt]
CT & 3,336,658 & 347,842 & MedTrinity CT~\citep{xie2025medtrinity} $\rightarrow$ MedTrinity CT~\citep{xie2025medtrinity} \\[2pt]
X-ray & 621,616 & 212,498 & MIMIC-CXR-JPG~\citep{johnson2024mimiccxr}; CheXpert~\citep{irvin2019chexpert}; NIH ChestX-ray14~\citep{wang2017chestxray8} $\rightarrow$ CheXpert~\citep{irvin2019chexpert} \\[2pt]
Ultrasound & 193,308 & 364,362 & MIMIC-IV-ECHO~\citep{gow2026mimicivecho}; brachial plexus~\citep{tyagi2024nerve}; STMUS~\citep{marzola2021stmus}; EchoCP~\citep{wang2021echocp} $\rightarrow$ US-365K~\citep{jin2026ultrasoundclip} \\[2pt]
Dermatology & 542,819 & 542,820 & ISIC, with known HAM10000 IDs excluded $\rightarrow$ ISIC metadata-derived captions~\citep{isic_archive} \\[2pt]
\midrule
\textbf{Seven-modality core} & \textbf{19,393,075\textsuperscript{$\dagger$}} & \textbf{3,229,073} & \\
Fundus extension & -- & 88,143 & DeepEyeNet~\citep{huang2021deepopht}; GMAI-VL-5.5M~\citep{li2026gmai}; MedConcept-23M~\citep{nie2026conceptclip}; PMC-derived pairs~\citep{lin2023pmc,pmc_oa} \\[2pt]
OCT extension & -- & 19,396 & OCT-summary-Dataset~\citep{qiaiuncc_oct_summary} \\
\midrule
\textbf{Total} & \textbf{19,393,075\textsuperscript{$\dagger$}} & \textbf{3,336,612} & \\
\bottomrule\end{tabularx}
\end{table}

\textbf{Classification and linear probing.} PanNuke ~\citep{gamper2020pannuke} contains 19 tissue classes. HEST ~\citep{jaume2024hest} contains eight tissue classes derived from patch captions. BrainTumor contains four MRI classes, and the RadImageNet-VQA ~\citep{butsanets2025radimagenet} subset contains 574 MRI images labeled by anatomy. The pneumonia benchmark contains pediatric chest radiographs. SARS-CoV-2 CT-scan evaluates binary COVID classification on 2,481 CT images~\citep{soares2020sarscov2}.

\textbf{Retrieval.} The four main retrieval benchmarks are OpenI ~\citep{demner2016preparing} for X-ray, MR-RATE ~\citep{forithmus2026mrrate}for MRI, Derm7pt ~\citep{kawahara2019seven} for dermatology, and PathGen-600 ~\citep{sun2025pathgen} for pathology. They contain $3{,}826$, $505$, $1{,}011$, and $600$ image--text pairs, respectively. The broader retention experiment also uses Liver-US ~\citep{li2025ultrasound} for ultrasound and EndoReport ~\citep{yu2026endoclip} for endoscopy. All retrieval evaluations rank normalized image and text embeddings using cosine similarity. Each experiment section provides its positive-matching and aggregation details.

\textbf{Linear-probe splits.} The ViT-B linear-probe evaluation uses stratified cross-validation. HEST is split at the image level rather than the slide level. Section~\ref{app:params} provides full details.

\section{Method and implementation details}\label{app:params}
\textbf{Parameter counts.} The nine-modality ViT-L model used by the proposed technique contains $1.280$ billion stored parameters, including the text encoder and projection heads. A single modality's image pathway contains $355.40$ million parameters. This pathway count includes the shared visual pathway and the selected expert at each routed layer, but excludes the text encoder and unselected experts. Table~\ref{tab:params} compares the proposed technique with the baseline models.

\begin{table}[ht]\centering\footnotesize
\caption{Architecture and parameter counts in millions ($10^6$). Image path denotes selected-path image parameters for the proposed technique and visual-module parameters for the baselines. Total denotes the full stored inference model. MerMED is vision-only.}
\label{tab:params}
\setlength{\tabcolsep}{3pt}
\resizebox{\linewidth}{!}{%
\begin{tabular}{@{}llrrr@{}}\toprule
Model & Vision backbone & Input & Image path (M) & Total (M) \\ \midrule
Proposed technique ViT-L (9-mod) & Multiway ViT-L/16 + DA-MoE & 384 & 355.40 & 1,280.00 \\
BMC & ViT-B/16 & 224 & 86.19 & 195.90 \\
BMCA & ViT-L/14 & 224 & 303.97 & 454.16 \\
ConceptCLIP & SigLIP SO400M/14 & 384 & 428.23 & 542.87 \\
UniMed-CLIP & ViT-B/16 & 224 & 86.19 & 195.90 \\
MerMED (inference backbone) & ViT-B/16 & 224 & 85.80 & 85.80 \\
\bottomrule\end{tabular}
}
\end{table}

\textbf{Tokenizer and two-stage training.}
We train one VQ-KD tokenizer with $8{,}192$ visual codes and freeze it
for both ViT-B and ViT-L training. Its targets come from six
modality-specific visual teachers: UNI for pathology, MedDINOv3 for CT,
USFM for ultrasound, RAD-DINO for X-ray, GastroNet-5M for endoscopy,
and BMC for MRI and dermatology. Stage 1 performs masked-image
and masked-language prediction using separate image and text forward
passes. Stage 2 uses a symmetric contrastive loss to align projected
and normalized image and text embeddings. Both stages use modality
routing.

\textbf{Core training schedule.} Masked pretraining runs for $130{,}000$ steps, including $2{,}000$ warmup steps. Image--text alignment runs for $15$ epochs, including one warmup epoch. The maximum text length is $100$ tokens in both stages.

\textbf{Stage 2 objective and optimization.} We optimize the ViT-L core using the symmetric diagonal InfoNCE objective in Section~\ref{sec:train}. The text encoder remains trainable, and its learning rate is $0.1$ times the vision learning rate. The learned logit scale is $\gamma=\min\{\exp(a),100\}$, giving an effective temperature of $1/\gamma$. Cross-GPU gathering over seven GPUs with $32$ pairs per GPU gives $B=224$ pairs in each contrastive calculation. We accumulate gradients over $18$ such microbatches, so each optimizer update uses $32\times7\times18=4{,}032$ pairs, but the negative pool remains $224$ pairs per contrastive calculation. The contrastive helper additionally divides its returned loss by the distributed world size; this changes optimizer scaling, not the positive or negative sets. The matched-data expansion experiment instead uses the equal-caption multiple-positive objective specified in Table~\ref{tab:matched-settings}.

\textbf{Domain-aware routing.} All image modalities share the image-attention parameters. The text pathway uses separate attention parameters. At each routed image layer, we add the shared image FFN output to the output of the expert selected by the given modality ID:
\[
\operatorname{FFN}(x,m)=\operatorname{FFN}_{\mathrm{shared}}(x)+\operatorname{Expert}_{m}(x).
\]
Here, $m$ denotes the image modality. Only the selected expert contributes to the layer output. However, this formulation does not claim that the implementation avoids every computation associated with unselected experts. Paired image--text training supervises the image expert; it does not add or train a modality-specific text expert.

\textbf{Learning a new modality with a frozen text encoder.} Each training pair is processed by both encoders. The image passes through the frozen shared visual pathway and the trainable expert for the new modality. The corresponding caption passes through the frozen text encoder. A contrastive loss aligns the resulting image and text embeddings. We compute text embeddings for every batch, but we do not update the text encoder. Only the image-expert FFNs for the new modality are updated. In this way, the new expert learns to map images into the existing text embedding space without requiring a new text expert.

\textbf{Preserving existing pathways.} During new-expert training, we freeze the shared visual parameters, all existing experts, the text encoder, and the projection heads. Existing modalities continue to use their original routes. When persistent state, routing, preprocessing, inference operations, scoring rules, candidate sets, and downstream classifiers also remain unchanged, the existing image and text embedding functions are preserved. The original 24-layer ViT-L core routes layers $\{4,8,12,16,20,24\}$, whereas the full-data expansion model uses routed FFNs in all 24 layers. When we convert the remaining layers, FFN remapping and zero-output branches preserve the original computation for existing modalities. A new modality expert therefore contains FFN parameters at every routed layer, rather than a single FFN block. On the evaluated tasks, we verify that the saved features and predictions are bitwise identical. This is an empirical verification on those inputs, not an unconditional claim about every possible output. Scaled similarity logits need not be identical because the assembled checkpoint sequence includes a separate change to the learned logit-scale parameter.

\textbf{Full-data ViT-L expert training.} We train the fundus and OCT experts independently from the same seven-modality base using $88{,}143$ and $19{,}396$ image--text pairs, respectively. Neither expert receives additional masked pretraining. We select the fundus checkpoint at epoch $8$ and the OCT checkpoint at epoch $15$, when their training losses plateau. The final full-data model uses these checkpoints. These runs are separate from the matched-data experiment, which uses subsets of $20{,}000$ fundus pairs and $16{,}000$ OCT pairs.

\begin{table}[ht]\centering\small
\caption{Full-data ViT-L fundus and OCT expert training. Configured epochs are the training budgets; selected epochs identify the checkpoints used in the final full-data assembly.}
\label{tab:full-expert-settings}
\resizebox{\linewidth}{!}{%
\begin{tabular}{@{}lcc@{}}\toprule
Setting & Fundus & OCT \\ \midrule
Training pairs & 88,143 & 19,396 \\
Initialization & Seven-modality base & Seven-modality base \\
Additional Stage 1 & None & None \\
Learning rate & $10^{-4}$ & $10^{-4}$ \\
Global batch size & 112 & 112 \\
Per-GPU batch size & 16 & 16 \\
Configured epochs & 10 & 40 \\
Selected checkpoint epoch & 8 & 15 \\
Selection criterion & Training-loss plateau & Training-loss plateau \\
\bottomrule\end{tabular}
}
\end{table}

\textbf{Matched-data expansion experiment.} This experiment starts from the seven-modality ViT-L checkpoint and keeps its original routed-layer structure. The $20{,}000$ fundus pairs are selected from the available $88{,}143$-pair fundus training dataset, and the $16{,}000$ OCT pairs are selected from the available $19{,}396$-pair OCT training dataset. We initialize the new fundus and OCT expert FFNs from the shared FFNs. At each stage, only the expert for the current modality is trainable. The shared parameters, previously trained experts, text pathway, and learned logit-scale parameter remain frozen. After selecting the fundus checkpoint, we train the OCT expert while keeping the fundus expert fixed. Its equal-caption multiple-positive contrastive loss differs from the diagonal-only core Stage-2 equation. This sequential experiment differs from the final model assembly, which combines independently trained full-data experts.

\begin{table}[h]\centering\small
\caption{Settings for the proposed technique's matched-data expansion experiment. The $20{,}000$ fundus pairs are selected from the available paired fundus training dataset; the $16{,}000$ OCT pairs are selected from the available paired OCT training dataset.}
\label{tab:matched-settings}
\resizebox{\linewidth}{!}{%
\begin{tabular}{@{}lp{8cm}@{}}\toprule
Setting & Value \\ \midrule
Training pairs & Fundus: $20{,}000$; OCT: $16{,}000$ \\
Seeds & $0,1,2$ \\
Optimizer & AdamW; weight decay $0.01$ \\
Batch size & $32$ image--text pairs \\
Candidate learning rates & $10^{-5}$ and $3\times10^{-5}$ \\
Training length & Minimum $8$, maximum $20$ epochs per candidate \\
Early stopping & Four epochs without a cumulative validation gain of $0.1$ percentage points relative to the stopping reference \\
Checkpoint selection & Highest current-task validation mean bidirectional Recall@$\{1,10,50\}$; first maximum on ties \\
Image preprocessing & Bicubic resize and center crop to $384\times384$; tensor values in $[0,1]$ \\
Text input & SentencePiece; maximum length $256$ \\
Loss & Symmetric contrastive loss with identical captions treated as multiple positives \\
Logit scale & Frozen at the initial checkpoint value \\
\bottomrule\end{tabular}
}
\end{table}

\textbf{ViT-B classification comparison.} We resize and center-crop each image to $384\times384$. For zero-shot classification, we compute the cosine similarity between normalized projected image embeddings and class-text embeddings. We encode and normalize each class prompt separately, average the prompt embeddings within each class, and normalize the result again. The maximum text length is $100$ tokens. Both models use the same procedure for each benchmark.

For the six linear-probe benchmarks, we use five-fold stratified cross-validation with shuffling and seed $0$. The input features are the global image representations before projection and normalization. Within each fold, we fit a StandardScaler on the training split and then fit logistic regression with $C=1$, balanced class weights, the L-BFGS solver, and a maximum of $3{,}000$ iterations. We report the mean balanced accuracy across folds. The folds are split by data row, not by patient or slide.

\textbf{Scope of identity checks.} The unchanged saved features establish identity only for the evaluated inputs. Other checkpoint changes, such as the logit-scale change in the independently assembled sequence, are separate from these feature checks. A finite benchmark cannot establish that every output is identical for every possible input.

\section{ViT-B comparison with and without DA-MoE}\label{app:ablation}
We compare a ViT-B implementation of the proposed technique with and without modality experts. Both models use the same benchmark-specific evaluation procedure. This comparison is separate from the ViT-L results in the main text.

Both ViT-B variants are trained from random initialization without loading a pretrained vision--language checkpoint. They start with identical image and text weights and use the same frozen pretrained visual tokenizer and training subsets. The routed model's expert output projections are initialized to zero, so the experts initially leave the shared pathway unchanged. Both variants use the same optimization settings and a fixed budget of five epochs for Stage 1 and five epochs for Stage 2. During Stage 2, the text learning rate is $0.1$ times the vision learning rate. We evaluate the final Stage-2 checkpoint for each model without early stopping or validation-based checkpoint selection. Table~\ref{tab:ablation-settings} provides the full configuration. This experiment compares training with and without modality experts under a limited, matched training budget. We did not train a ViT-L model without experts.

\begin{table}[h]\centering\small
\caption{Settings for the ViT-B comparison with and without DA-MoE (Section~\ref{app:ablation}).}
\label{tab:ablation-settings}
\resizebox{\linewidth}{!}{%
\begin{tabular}{@{}lp{8.5cm}@{}}\toprule
Setting & Value \\ \midrule
Shared initialization & Identical randomly initialized shared weights; no pretrained backbone checkpoint \\
Expert initialization & Shared-FFN initialization with zero output projections; routed model only \\
Visual tokenizer & Same pretrained, frozen VQ-KD tokenizer for both models \\
Stage-1 images & $69{,}902$ across seven modalities; 98 validation-overlapping image paths excluded \\
Stage-1 text & $3{,}229{,}073$ unpaired texts, sampled separately from images \\
Stage-1 duration & Five image-data epochs; $1{,}560$ optimizer updates \\
Stage-1 batch & Seven GPUs $\times$ 32 images and 32 independent texts per GPU \\
Stage-1 masking & $40\%$ image patches; $15\%$ text tokens \\
Stage-1 learning rate & Peak $2\times10^{-4}$, minimum $10^{-5}$; 156-update warmup, then cosine decay \\
Stage-1 objective & Equal-weighted sum of mean masked-image and mean masked-language cross-entropy losses \\
Stage-2 pairs & $164{,}605$ training / $8{,}666$ validation \\
Stage-2 effective batch & $4{,}032$ pairs: 448-pair distributed microbatch $\times$ 9 accumulation steps \\
Stage-2 learning rate & Peak $10^{-4}$; text learning-rate multiplier $0.1$ \\
Stage-2 schedule & Five epochs; one-epoch warmup, then cosine decay \\
Checkpoint used for evaluation & Final Stage-2 epoch 5 for both variants \\
Logit scale & Trainable; initial effective temperature $0.02$; scale capped at $100$ \\
Stage-2 objective & Symmetric contrastive loss; auxiliary MLM, load balancing, and distillation inactive \\
Image resolution / text limit & $384\times384$ / 100 tokens \\
\bottomrule\end{tabular}
}
\end{table}

Both variants use the same fixed training budget. Validation scores are monitored but do not determine stopping or checkpoint selection. Five epochs do not necessarily establish convergence.

\begin{table}[h]\centering\small
\caption{\textbf{Comparison with and without DA-MoE, the proposed technique (ViT-B) vs.\ no-MoE ViT-B.} Zero-shot and linear probe
are balanced accuracy; retrieval is mean $R@\{1,10,50\}$ over both directions. All retrieval benchmarks use cosine similarity. PathGen-600 treats rows with identical captions as positives.
Both variants share the same random initialization, training data, optimization settings, and training seed (Table~\ref{tab:ablation-settings}); \best{best} per row.}
\label{tab:ablation-full}
\resizebox{\linewidth}{!}{%
\begin{tabular}{llccc}
\toprule
Benchmark & Modality & Proposed technique\ ViT-B & no-MoE ViT-B & $\Delta$ \\
\midrule
\multicolumn{5}{l}{\emph{Zero-shot} (n=8 benchmarks)}\\
HAM10000            & dermatology & \best{.1401} & .1276 & $+.0125$ \\
LC25000             & pathology   & \best{.2000} & .0652 & $+.1348$ \\
Thyroid             & ultrasound  & \best{.5091} & .4608 & $+.0483$ \\
BreaKHis            & pathology   & \best{.5000} & \best{.5000} & $.0000$ \\
CASIA-CXR           & X-ray       & \best{.3200} & .2000 & $+.1200$ \\
Pneumonia           & X-ray       & .5107 & \best{.5346} & $-.0239$ \\
PanNuke (19 tissues)& pathology   & \best{.0526} & .0499 & $+.0027$ \\
SARS-CoV-2 CT-scan  & CT          & \best{.4996} & .4480 & $+.0516$ \\
\quad\emph{zero-shot mean} & & \best{.3415} & .2983 & $+.0433$ \\
\midrule
\multicolumn{5}{l}{\emph{Linear probe} (n=6 benchmarks)}\\
BrainTumor          & MRI         & \best{.8683} & .8568 & $+.0115$ \\
HAM10000            & dermatology & .4289 & \best{.4325} & $-.0036$ \\
RadImageNet         & MRI         & \best{.7582} & .6908 & $+.0674$ \\
TN5000              & ultrasound  & \best{.7667} & .7475 & $+.0192$ \\
CT-Kidney           & CT          & \best{.9355} & .9142 & $+.0213$ \\
HEST                & pathology   & \best{.9911} & .9863 & $+.0048$ \\
\quad\emph{linear-probe mean} & & \best{.7915} & .7714 & $+.0201$ \\
\midrule
\multicolumn{5}{l}{\emph{Cross-modal retrieval} (n=4 benchmarks)}\\
OpenI               & X-ray       & \best{.5881} & .5837 & $+.0044$ \\
MR-RATE             & MRI         & \best{3.5644} & 3.1584 & $+.4060$ \\
derm7pt             & dermatology & \best{1.9618} & 1.8958 & $+.0660$ \\
PathGen-600         & pathology   & \best{3.5000} & 3.4444 & $+.0556$ \\
\quad\emph{retrieval mean} & & \best{2.4036} & 2.2706 & $+.1330$ \\
\bottomrule
\end{tabular}
}
\end{table}

The proposed technique performs better on 15 of the 18 benchmarks, ties on BreaKHis ~\citep{spanhol2016dataset} zero-shot classification, and performs worse on Pneumonia ~\citep{kermany2018identifying} zero-shot classification and ~\citep{tschandl2018ham10000} HAM10000 linear probing. On average, DA-MoE improves zero-shot balanced accuracy by $0.0433$, linear-probe balanced accuracy by $0.0201$, and retrieval by $0.1330$ percentage points. The largest improvements occur on LC25000 ~\citep{borkowski2019lung} and CASIA-CXR ~\citep{metmer2024casia} for zero-shot classification, RadImageNet for linear probing, and MR-RATE for retrieval.

\section{Multi-teacher visual tokenization}
\label{app:tokenizer-ablation}
We compare two VQ-KD tokenizers. The multi-teacher tokenizer uses six
pretrained visual teachers across seven modalities: UNI for
pathology~\citep{chen2024uni}, MedDINOv3 for
CT~\citep{meddinov3}, USFM for
ultrasound~\citep{jiao2024usfm}, RAD-DINO for
X-ray~\citep{perezgarcia2024raddino}, GastroNet-5M for
endoscopy~\citep{boers2024gastronet}, and BMC for MRI and
dermatology~\citep{zhang2023biomedclip}. The single-teacher tokenizer
uses BMC for all seven modalities. Both tokenizers use the same
architecture and a vocabulary of $8{,}192$ visual codes.

Both conditions use the same subset of $69{,}902$ unpaired images and $164{,}605$ image--text pairs. They also use the same ViT-B DA-MoE architecture, random initialization and hyperparameters. Tokenizer training, masked pretraining, and image--text alignment each run for five epochs. We evaluate the final alignment checkpoints on eight zero-shot classification, six linear-probe, and four image--text retrieval benchmarks.

\begin{table}[ht]
\centering\small
\caption{Comparison of multi-teacher and single-teacher visual tokenization. Zero-shot and linear-probe results are balanced accuracy (\%). Retrieval results are mean bidirectional Recall@$\{1,10,50\}$ (\%). Means are recomputed from the benchmark values shown in the table. The best result in each row is in bold.}
\label{tab:tokenizer-ablation}
\setlength{\tabcolsep}{5pt}
\begin{tabular}{@{}llrr@{}}
\toprule
Task & Benchmark & Multi-teacher & Single-teacher \\
\midrule
\multirow{9}{*}{Zero-shot}
& HAM10000 & \best{14.29} & 14.11 \\
& LC25000 & 20.00 & 20.00 \\
& Thyroid & \best{53.85} & 50.51 \\
& BreaKHis & 50.00 & 50.00 \\
& CASIA-CXR & \best{24.00} & 22.00 \\
& Pneumonia & 50.00 & 50.00 \\
& PanNuke & \best{5.97} & 5.26 \\
& SARS-CoV-2 CT-scan & \best{66.66} & 50.00 \\
& \emph{Mean} & \best{35.60} & 32.74 \\
\midrule
\multirow{7}{*}{Linear probe}
& BrainTumor & \best{85.83} & 81.78 \\
& HAM10000 & \best{39.48} & 34.02 \\
& RadImageNet & \best{71.30} & 67.94 \\
& TN5000 & 72.76 & \best{73.93} \\
& CT-Kidney & 89.61 & \best{90.69} \\
& HEST & \best{99.05} & 98.97 \\
& \emph{Mean} & \best{76.34} & 74.56 \\
\midrule
\multirow{5}{*}{Retrieval}
& OpenI & \best{0.49} & 0.43 \\
& MR-RATE & 3.96 & 3.96 \\
& Derm7pt & \best{1.67} & 1.06 \\
& PathGen-600 & \best{3.28} & 3.22 \\
& \emph{Mean} & \best{2.35} & 2.17 \\
\bottomrule
\end{tabular}
\end{table}

Multi-teacher supervision improves the mean zero-shot balanced accuracy by $2.86$ percentage points, from $32.74\%$ to $35.60\%$. It also improves the mean linear-probe balanced accuracy by $1.78$ points, from $74.56\%$ to $76.34\%$, and the mean retrieval score by $0.18$ points, from $2.17\%$ to $2.35\%$. The largest zero-shot improvement occurs on SARS-CoV-2 CT-scan. At the benchmark level, multi-teacher supervision performs better on five zero-shot tasks, four linear-probe tasks, and three retrieval tasks; it performs worse on two linear-probe tasks and ties on the remaining tasks.

Overall, multi-teacher supervision improves all three mean scores in this experiment, but the gains vary across benchmarks. Because the experiment uses one seed, a limited training subset, and a short schedule, it does not establish statistical significance or performance under converged full-data training.

\section{Matched-data adaptation: setup and results}\label{app:ft}
We compare the proposed technique (ViT-L) with BMC, BMCA, ConceptCLIP,
and UniMed-CLIP using the same sequential adaptation data:
$20{,}000$ fundus image--text pairs followed by $16{,}000$ OCT pairs.
For every model, OCT adaptation starts from the selected fundus
checkpoint. This subset experiment measures learning and retention and
is separate from the independently trained full-data expert assembly.
Fundus and OCT are newly added expert pathways for the proposed technique, although
the baselines may have encountered these modalities during pretraining.

The baselines use LoRA, whereas the proposed technique trains new modality experts.
We initialize the proposed technique from its seven-modality ViT-L checkpoint and
initialize the fundus and OCT experts by copying the shared feed-forward
layers. We first train only the fundus expert and then train only the OCT
expert while keeping the fundus expert fixed. The shared backbone,
original experts, text encoder, projection heads, and contrastive
temperature remain frozen throughout. The frozen text encoder processes
each caption, and the contrastive loss trains the new visual expert to
align images with the corresponding text representations.

For each baseline, we use LoRA with rank $16$, $\alpha=32$, a batch size of $32$, and a symmetric contrastive loss. We select the learning rate from $\{10^{-5},3\times10^{-5}\}$ using validation performance on the current task. Each stage runs for at most $20$ epochs. We report the mean across three training seeds.

Evaluation uses fixed sets of $588$ fundus images, $576$ OCT images, and $600$ PathGen report rows. We define the reported changes as follows. Fundus gain is the score after fundus adaptation minus the base score. OCT gain is the score after OCT adaptation minus the score before OCT adaptation. Fundus forgetting is the fundus score after fundus adaptation minus the fundus score after OCT adaptation. PathGen forgetting is the base PathGen score minus the final PathGen score. All changes are measured in percentage points. A positive forgetting value indicates a performance decrease.

All methods improve on the new fundus and OCT tasks. After OCT adaptation, the baselines lose some fundus performance and also perform worse on PathGen. The proposed technique shows no forgetting on fundus or PathGen because its existing pathways remain frozen.

\section{Matched-data learning and retention}\label{app:matched}
Figure~\ref{fig:matched_seq} presents additional results from the matched-data experiment summarized in Table~\ref{tab:learn-forget}. All methods use the same evaluation galleries and checkpoint-selection rules. The adaptation data are matched, but model size, pretraining history, trainable parameter count, and computational cost are not.

Fundus and OCT are new expert pathways for the proposed technique. In contrast, the baselines may have encountered these modalities during pretraining. Therefore, the stage labels refer to adaptation and do not necessarily indicate a model's first exposure to the modality.

\textbf{Fundus.} After fundus training, the proposed technique improves from $5.22\%$ to $20.17\%$. Its fundus score does not change after OCT integration. All baselines also improve after fundus training, but their fundus scores decrease after OCT training.

\textbf{OCT.} Before OCT training, the proposed technique scores $4.28\%$. After training, it reaches $13.59\%$. The baselines reach scores between $15.46\%$ and $15.88\%$. Thus, the proposed technique learns OCT while preserving earlier capabilities, although it does not achieve the highest OCT score in this subset experiment. The initial fundus and OCT scores are obtained with newly initialized experts.

\textbf{Dermatology.} On the fixed Derm7pt gallery of $1{,}011$ rows, the proposed technique remains at $9.63\%$ after adding both modalities. Within each seed, the saved image and text features are byte-identical across stages. The baseline scores decrease from $6.02$ to $3.56$ for BMC, from $7.35$ to $5.43$ for BMCA, from $8.70$ to $6.12$ for ConceptCLIP, and from $3.30$ to $2.71$ for UniMed-CLIP.

We report mean Recall@$\{1,10,50\}$ across both retrieval directions. Derm7pt, fundus, and OCT all use cosine similarity. Within each benchmark, all methods and stages use the same evaluation protocol.

\section{Retention across seven base modalities}\label{app:retention7}
\begin{figure}[t]
\centering
\includegraphics[trim=6.2pt 32pt 5.8pt 5.8pt, clip, width=0.78\linewidth]{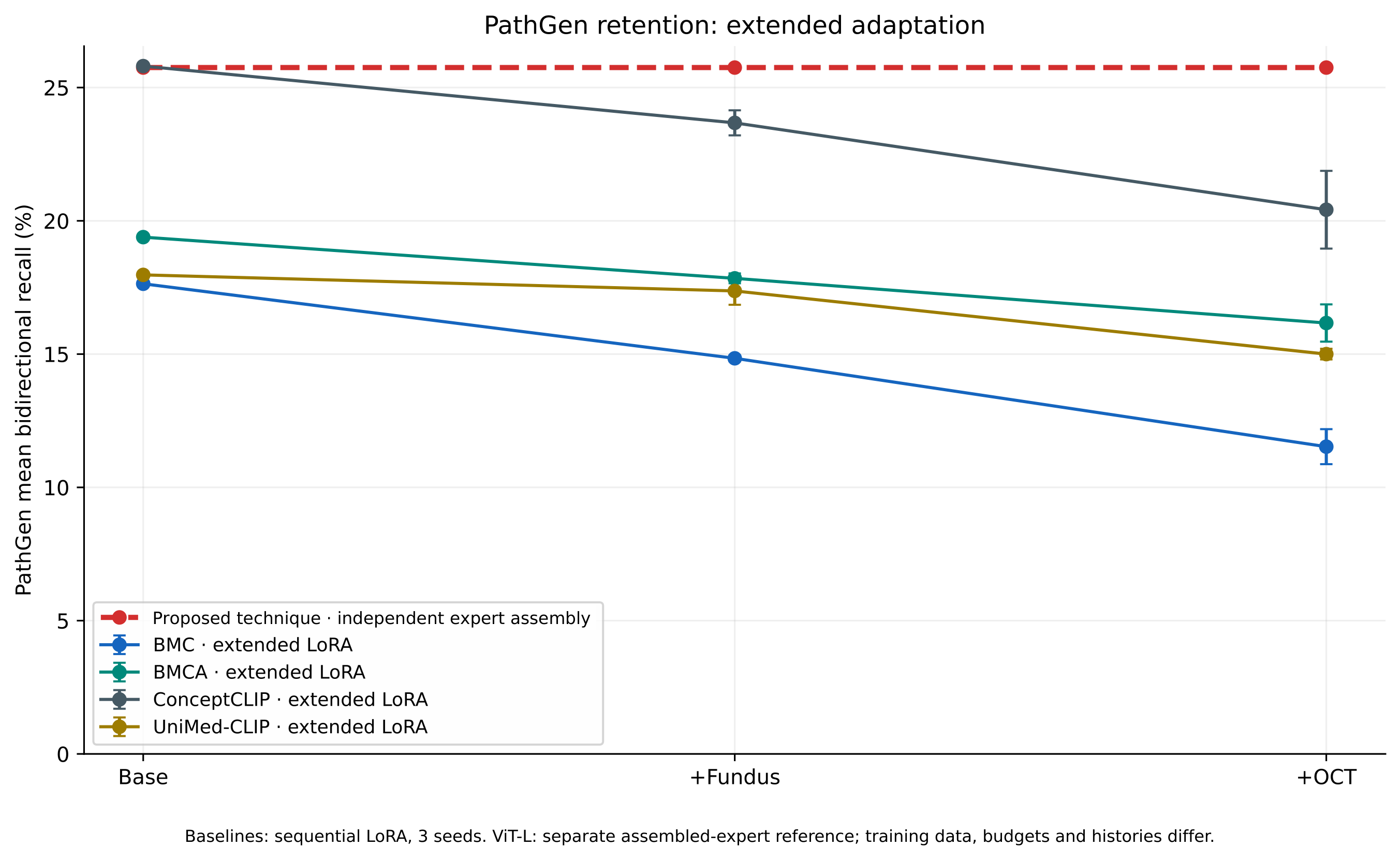}
\caption{\textbf{PathGen retention through fundus--OCT adaptation.} Mean bidirectional
Recall@$\{1,10,50\}$ at the base checkpoint, after fundus, and after OCT. Baseline error bars
show standard deviation over three seeds. The proposed technique's independently trained experts leave the
pathology pathway unchanged. Training history, data exposure, and adaptation capacity differ
between the proposed technique\ and the sequential-LoRA baselines.}
\label{fig:pathgen_trajectory}
\end{figure}

We evaluate one benchmark for each of the seven original modalities before and after adding fundus and OCT. For the baselines, we reuse the sequential LoRA runs described in Section~\ref{app:ft}. For the proposed technique, we evaluate three assembled checkpoints: the original core, the core with the fundus expert, and the core with both the fundus and OCT experts. The two new experts were trained independently from the frozen core and then combined. Therefore, this procedure differs from both the sequential baseline adaptation and the three-seed matched-data experiment.

We use OpenI ($3{,}826$ rows), MR-RATE ($505$), Derm7pt ($1{,}011$), PathGen-600 ($600$), ROCO-CT ($2{,}026$ pairs), Liver-US ($680$), and EndoReport ($217$). ROCO-CT is a CT-only subset of the ROCOv2 test split \citep{ruckert2024rocov2}. We select this subset using the dataset's modality annotations and remove records with missing or non-CT modality labels. Each image is paired with a caption from a biomedical publication. All seven benchmarks measure bidirectional retrieval recall. Inputs, prompts, and scoring procedures remain fixed across stages.

For each benchmark, we report $\Delta=S_{\mathrm{after\ OCT}}-S_{\mathrm{base}}$ in percentage points. A negative value indicates a performance decrease, and a positive value indicates an improvement. All retrieval benchmarks use cosine similarity. The retrieval score averages Recall@$1$, Recall@$10$, and Recall@$50$ across both directions. 

\begin{table}[h]\centering\small
\caption{\textbf{Full base-modality retention}, $\Delta$ from original checkpoint to after
fundus+OCT integration (negative = forgetting, positive = improvement). Baseline
values are the mean sequential-LoRA run over three seeds; ViT-L values for the proposed technique come from one
checkpoint sequence (these values do not estimate training-seed variance). All seven rows
use bidirectional retrieval recall.}
\label{tab:retention7}
\resizebox{\linewidth}{!}{%
\begin{tabular}{@{}llrrrrr@{}}
\toprule
Benchmark & Modality & Proposed technique\ & BMC & BMCA & ConceptCLIP & UniMed \\
\midrule
OpenI          & X-ray       & 0.00 & $-1.71$  & $-0.70$  & $-1.85$  & $-14.70$ \\
MR-RATE        & MRI         & 0.00 & $-0.20$  & $-0.42$  & $-0.23$  & $+0.57$  \\
Derm7pt        & dermatology & 0.00 & $-2.46$  & $-1.92$  & $-2.59$  & $-0.58$  \\
PathGen-600    & pathology   & 0.00 & $-6.11$  & $-3.22$  & $-5.39$  & $-2.97$  \\
ROCO-CT        & CT          & 0.00 & $-33.26$ & $-13.42$ & $-10.86$ & $-22.58$ \\
Liver-US       & ultrasound  & 0.00 & $-0.44$  & $-0.44$  & $-0.54$  & $-0.96$  \\
EndoReport     & endoscopy   & 0.00 & $-4.46$  & $-3.62$  & $-1.43$  & $-4.08$  \\
\midrule
\textbf{Mean} ($n=7$) & & 0.00 & $-6.95$ & $-3.39$ & $-3.27$ & $-6.47$ \\
\bottomrule
\end{tabular}
}
\end{table}

The baseline scores decrease in $27$ of the $28$ model--benchmark combinations. UniMed-CLIP improves on MR-RATE by $0.57$ points. Every other baseline result decreases, and all four baselines decrease on ROCO-CT. The proposed technique remains unchanged on all seven benchmarks. The mean row summarizes percentage-point changes across different tasks; it is not a single shared evaluation metric.

The saved the proposed technique image and text features are byte-identical across the evaluated checkpoints, and the ROCO-CT retrieval results are unchanged. This verifies preservation on the tested inputs. It does not mean that every parameter or scaled similarity logit is unchanged because the learned logit-scale parameter differs between the original core and the fundus checkpoint. We also do not claim that every baseline decrease is statistically significant or that these seven tasks cover every capability of the models.

\section{Qualitative examples}\label{app:qual}
The following figures show one retrieval example for each original modality in
Table~\ref{tab:retention7}, before expansion and after adding fundus and OCT. Every model uses
the same query, gallery, preprocessing, and cosine-similarity scoring at both stages. Each
figure shows the three highest-ranked reports. Green marks a report that is correct under the
benchmark's matching rule; an unmarked report may still be clinically relevant. These examples
illustrate individual ranking changes. We also include a
separate CT-Kidney classification example (Figure~\ref{fig:qual-ct}); it is not part of the
aggregate retrieval results in Table~\ref{tab:retention7}.

\begin{figure}[p]\centering
\includegraphics[width=0.95\linewidth,height=0.69\textheight,keepaspectratio]{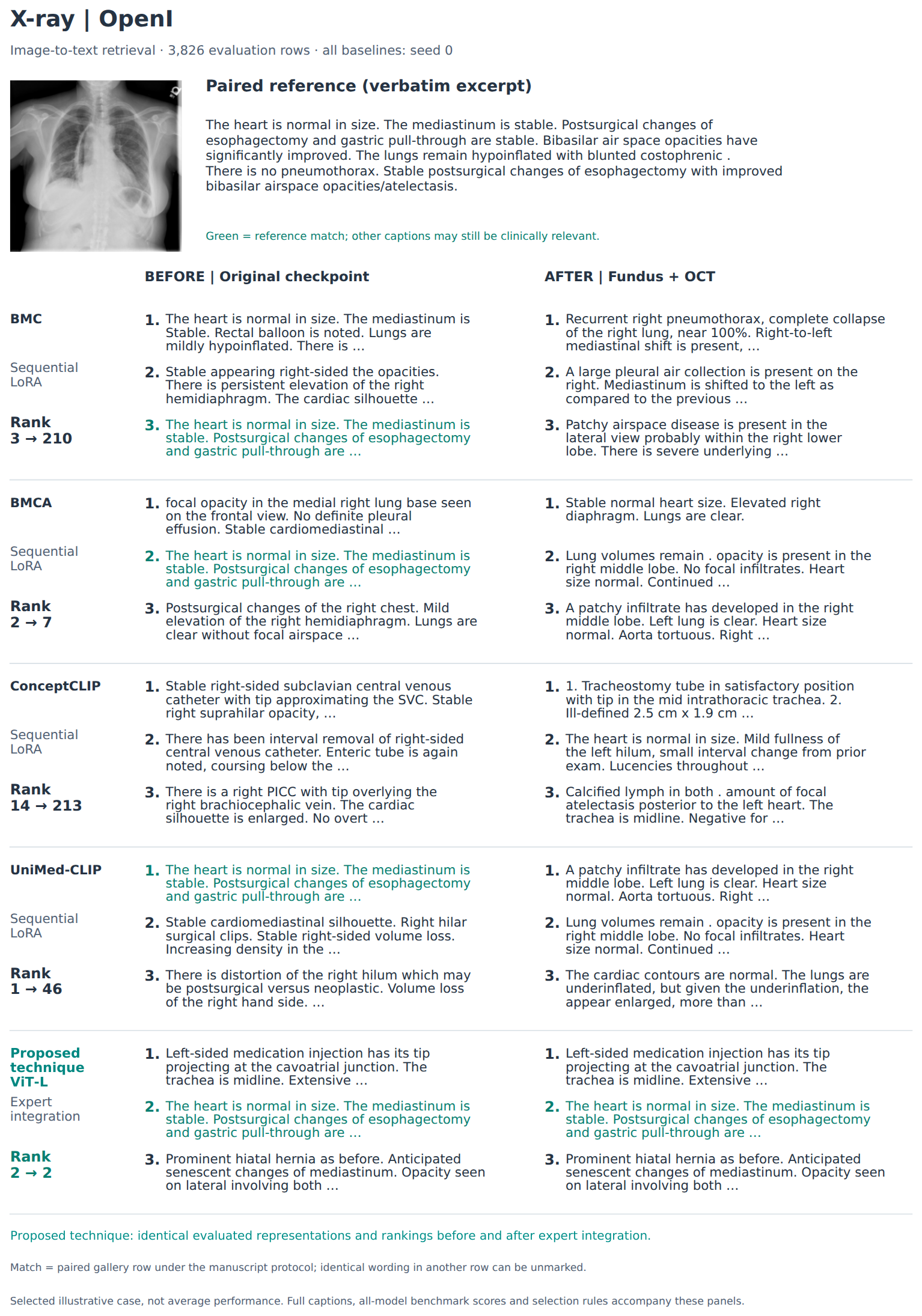}
\caption{\textbf{X-ray (OpenI) qualitative example.} Top-three image-to-report results for the
same query radiograph before expansion (original checkpoint) and after fundus and OCT
integration. Baseline results use seed~$0$. Green marks the reference report. The rank of the benchmark-positive report is: the proposed technique\
$2\!\to\!2$ (unchanged; saved representations identical); BMC $3\!\to\!210$; BMCA
$2\!\to\!7$; ConceptCLIP $14\!\to\!213$; UniMed-CLIP $1\!\to\!46$.}
\label{fig:qual-xray}
\end{figure}

\begin{figure}[p]\centering
\includegraphics[width=0.95\linewidth,height=0.69\textheight,keepaspectratio]{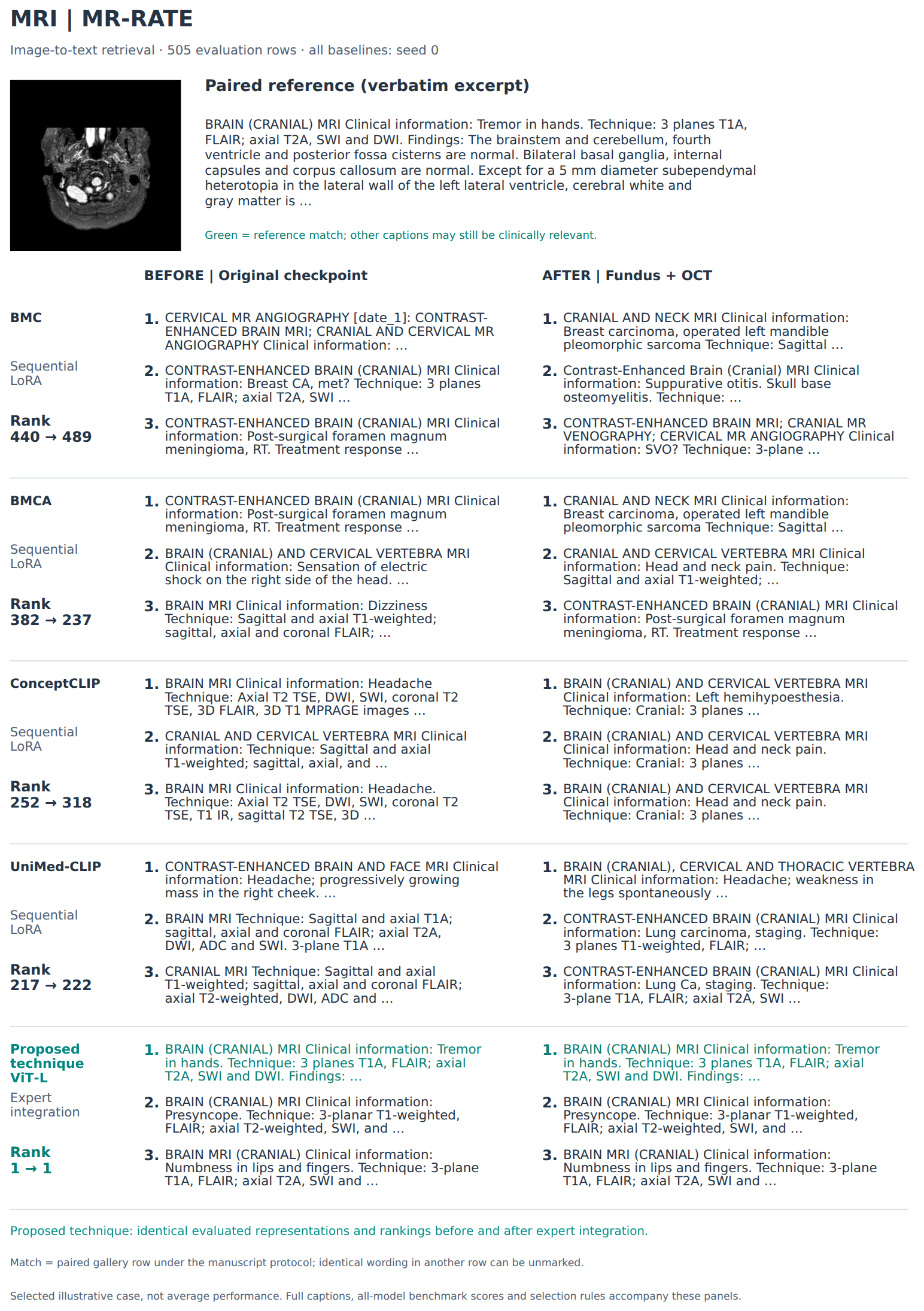}
\caption{\textbf{MRI (MR-RATE) qualitative example.} We use the same before-and-after format as
Figure~\ref{fig:qual-xray} and apply MR-RATE scoring. The rank of the benchmark-positive report is: the proposed technique\ $1\!\to\!1$; BMC
$440\!\to\!489$; BMCA $382\!\to\!237$; ConceptCLIP $252\!\to\!318$; UniMed-CLIP
$217\!\to\!222$. For this query, no baseline ranks the correct report among the top three at
either checkpoint, whereas the proposed technique\ ranks it first at both checkpoints.}
\label{fig:qual-mri}
\end{figure}

\begin{figure}[p]\centering
\includegraphics[width=0.95\linewidth,height=0.69\textheight,keepaspectratio]{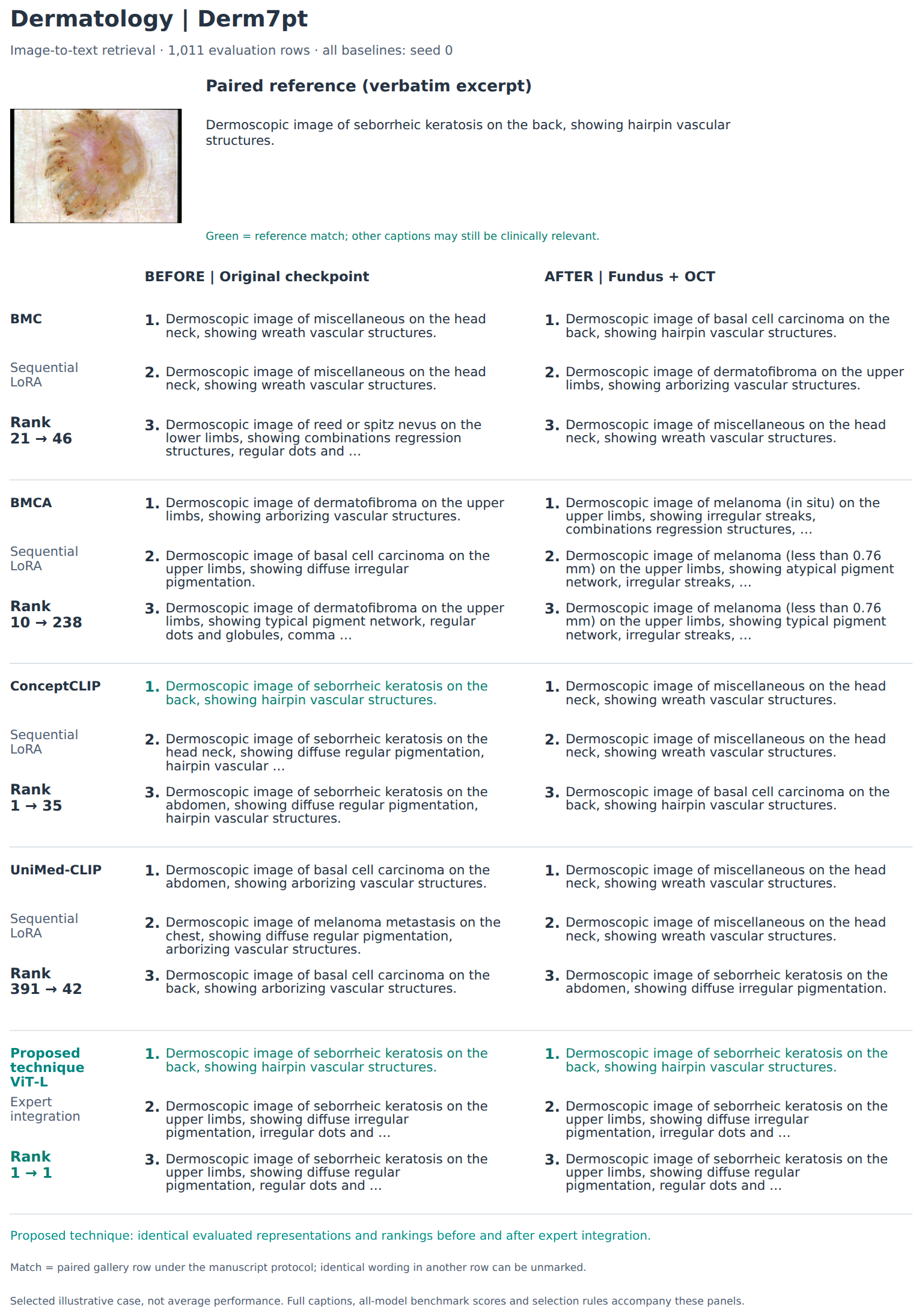}
\caption{\textbf{Dermatology (Derm7pt) qualitative example.} We use the same before-and-after format as
Figure~\ref{fig:qual-xray} and apply Derm7pt scoring. The rank of the benchmark-positive report is: the proposed technique\ $1\!\to\!1$; BMC
$21\!\to\!46$; BMCA $10\!\to\!238$; ConceptCLIP $1\!\to\!35$; UniMed-CLIP
$391\!\to\!42$. For this query, ConceptCLIP is the only baseline that ranks the correct report
among the top three before adaptation; no baseline does so afterward.}
\label{fig:qual-dermatology}
\end{figure}

\begin{figure}[p]\centering
\includegraphics[width=0.95\linewidth,height=0.69\textheight,keepaspectratio]{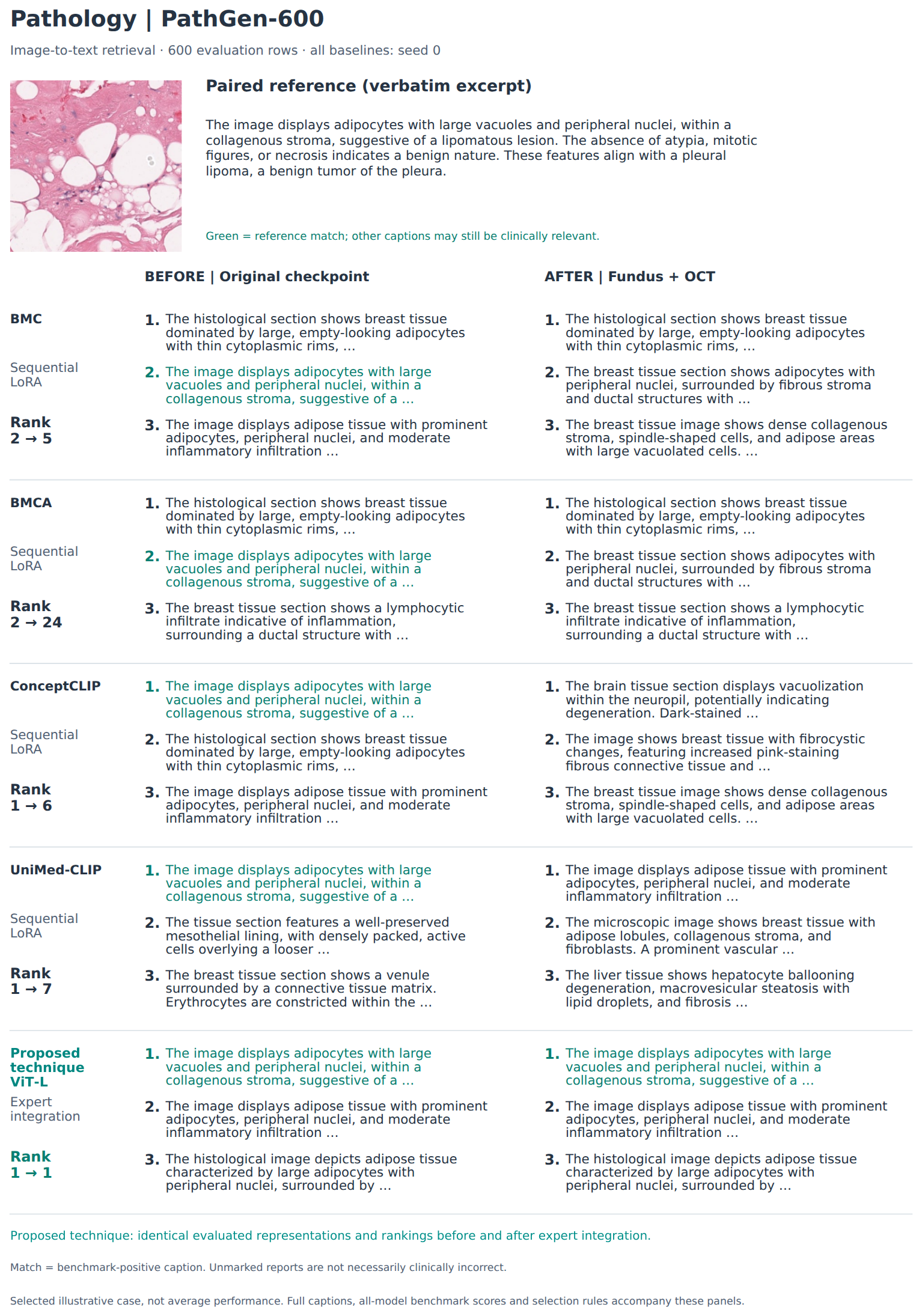}
\caption{\textbf{Pathology (PathGen-600) qualitative example.} We use the same before-and-after format as
Figure~\ref{fig:qual-xray} and apply PathGen-600 scoring. The rank of the benchmark-positive report is: the proposed technique\ $1\!\to\!1$; BMC
$2\!\to\!5$; BMCA $2\!\to\!24$; ConceptCLIP $1\!\to\!6$; UniMed-CLIP $1\!\to\!7$. All
four baselines place the correct report in the top three before adaptation; none do
afterward.}
\label{fig:qual-pathology}
\end{figure}

\begin{figure}[p]\centering
\includegraphics[width=0.95\linewidth,height=0.69\textheight,keepaspectratio]{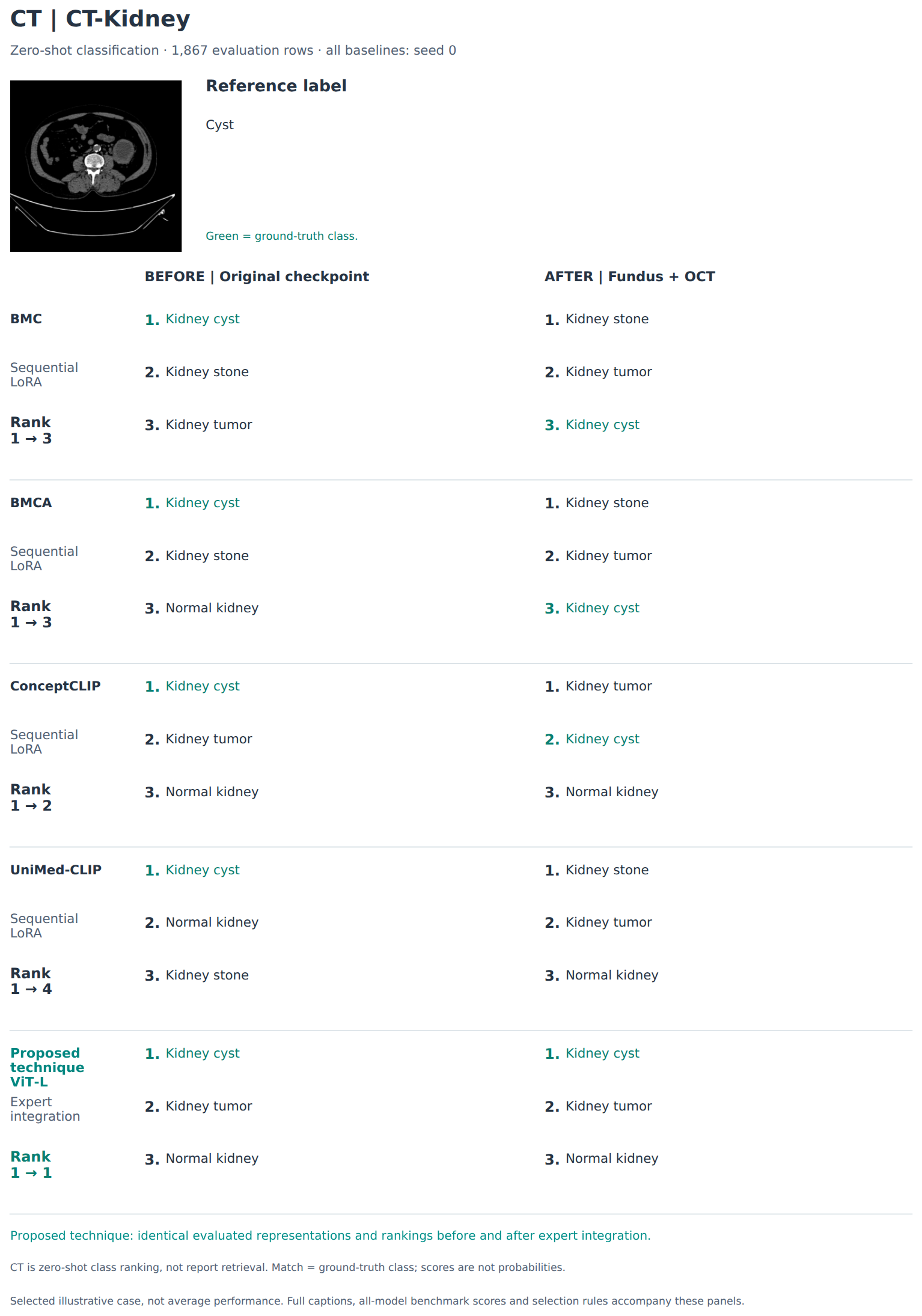}
\caption{\textbf{Additional CT-Kidney qualitative example.} This is a zero-shot classification
example, not a retrieval example, and it is not included in Table~\ref{tab:retention7}. We show the
top three predictions among four classes (kidney cyst, stone, tumor, and normal kidney) for the same query before and
after integration. Baseline results use seed~$0$. Green marks the ground-truth class (cyst). Its rank is: the proposed technique\
$1\!\to\!1$; BMC $1\!\to\!3$; BMCA $1\!\to\!3$; ConceptCLIP $1\!\to\!2$; UniMed-CLIP
$1\!\to\!4$. At the original checkpoint, every model ranks the ground-truth class first.
After adaptation, the rank decreases for every baseline on this query. However, on the full
CT-Kidney benchmark, BMCA's mean balanced accuracy improves after adaptation. CT-Kidney
is a separate zero-shot evaluation and is not the ROCO-CT retrieval benchmark in
Table~\ref{tab:retention7}. This difference shows why a single example should not be used to infer aggregate performance.}
\label{fig:qual-ct}
\end{figure}

\begin{figure}[p]\centering
\includegraphics[width=0.95\linewidth,height=0.69\textheight,keepaspectratio]{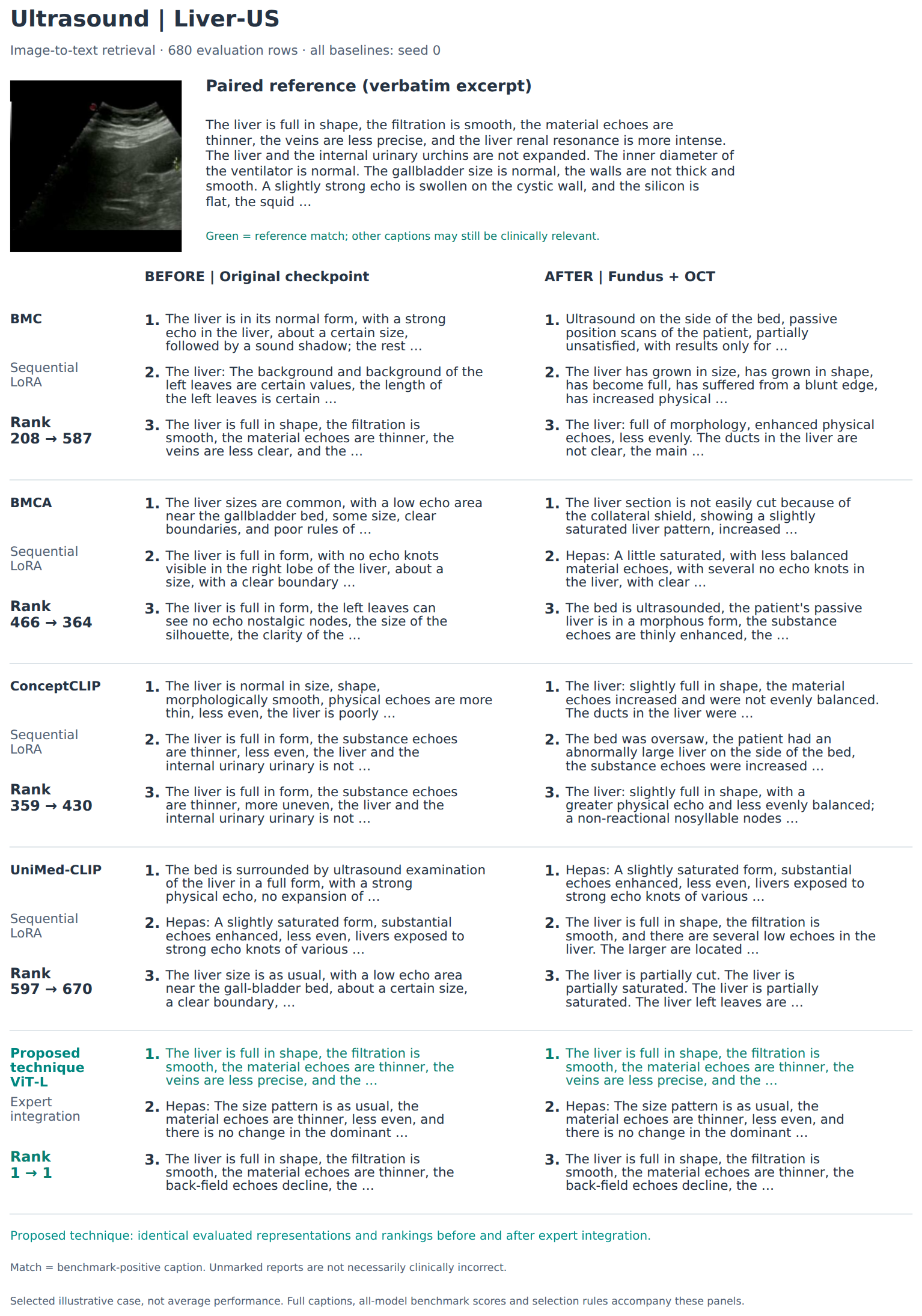}
\caption{\textbf{Ultrasound (Liver-US) qualitative example.} We use the same before-and-after format as
Figure~\ref{fig:qual-xray} and apply Liver-US scoring. The rank of the benchmark-positive report is: the proposed technique\ $1\!\to\!1$; BMC
$208\!\to\!587$; BMCA $466\!\to\!364$; ConceptCLIP $359\!\to\!430$; UniMed-CLIP
$597\!\to\!670$. No baseline is within the top three at either checkpoint on this query.}
\label{fig:qual-ultrasound}
\end{figure}

\begin{figure}[p]\centering
\includegraphics[width=0.95\linewidth,height=0.69\textheight,keepaspectratio]{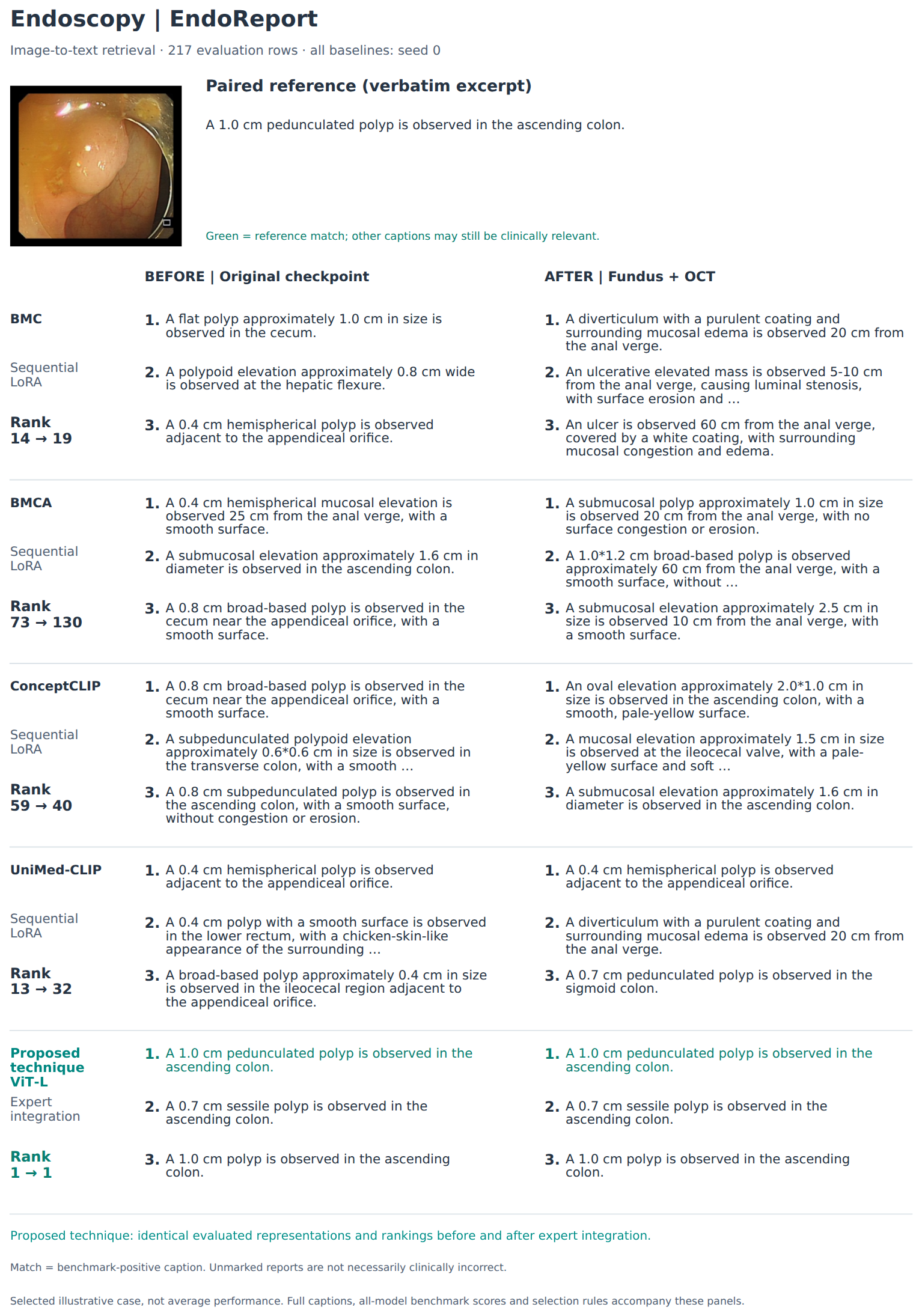}
\caption{\textbf{Endoscopy (EndoReport) qualitative example.} We use the same before-and-after format as
Figure~\ref{fig:qual-xray} and apply EndoReport scoring. The rank of the benchmark-positive report is: the proposed technique\ $1\!\to\!1$; BMC
$14\!\to\!19$; BMCA $73\!\to\!130$; ConceptCLIP $59\!\to\!40$; UniMed-CLIP
$13\!\to\!32$. No baseline is within the top three at either checkpoint on this query.}
\label{fig:qual-endoscopy}
\end{figure}

\section{Saliency overview across nine modalities}\label{app:saliency}
Figure~\ref{fig:saliency} shows similarity-gradient Grad-CAM maps for all nine modalities.
We back-propagate the image--text cosine similarity for each finding prompt through the routed
expert to the final block's patch tokens. The maps illustrate spatial responses to the specified prompts across the
modalities, including fundus and OCT. They are qualitative examples rather
than a quantitative localization evaluation. 

\begin{figure}[t]
\centering
\includegraphics[width=\linewidth]{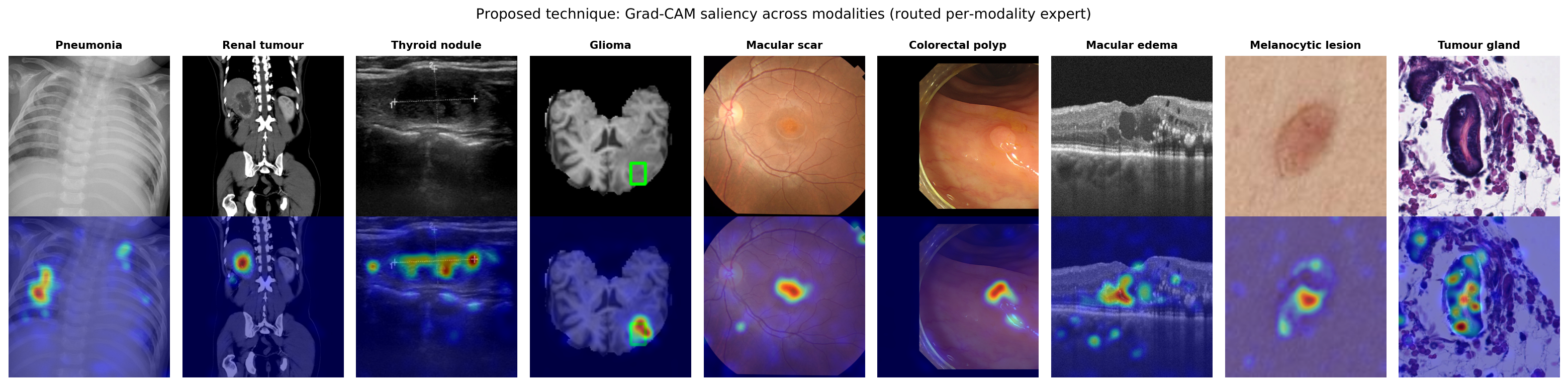}
\caption{\textbf{Qualitative saliency across nine modalities.} Similarity-gradient Grad-CAM
through each modality's routed expert; top row original, bottom row saliency. Left to right:
pneumonia, renal tumour, thyroid nodule, glioma on FLAIR with the ground-truth box, macular
scar, colorectal polyp, macular edema on OCT, melanocytic lesion, and malignant gland on
H\&E. Fundus and OCT use experts added to a frozen model.}
\label{fig:saliency}
\end{figure}

\clearpage
\section{Comparative image--text similarity-gradient patch saliency}
\label{app:comparative-saliency}

We supplement the quantitative evaluations with $66$ selected examples from $22$
datasets: $33$ zero-shot classification examples, $21$ linear-probe examples, and
$12$ image-to-text retrieval examples. Each dataset contributes three examples.
The columns compare the proposed technique, BMC, BMCA (labeled BIOMEDICA in the source
presentation), ConceptCLIP, and UniMed-CLIP. Each comparison uses the same
source image and text prompt for all five models, with each model's own image
preprocessing. Within an example, the upper row shows the image and the lower
row shows the saliency overlay. These selected cases illustrate differences in
spatial responses; they do not estimate dataset-level accuracy or retrieval
recall or characterize the models' failure cases. IDRiD is included as an additional qualitative linear-probe dataset and
is not a row in Table~\ref{tab:lp}.

\subsection{Saliency construction and task-specific prompts}
\label{app:comparative-saliency-method}
We compute image--text similarity-gradient patch saliency without updating model
weights. Let $s$ be the cosine similarity between the image and text embeddings,
and let $A_{p,c}$ be the patch features at the fourth transformer block from the
end, where $p$ indexes spatial patches and $c$ indexes channels. Backpropagating
$s$ to these features gives
\begin{equation}
H_p=\operatorname{ReLU}\!\left(
\sum_c A_{p,c}\frac{\partial s}{\partial A_{p,c}}
\right).
\label{eq:comparative-patch-saliency}
\end{equation}
Each displayed overlay combines $50\%$ image and $50\%$ heatmap. The proposed technique\ uses the
modality route associated with the image. This construction uses a patch-wise
feature--gradient product; the new comparisons use the fourth block from the
end, rather than the final-block maps in Figure~\ref{fig:saliency}. Heatmap
colors visualize spatial responses within an example and should not be read as
a common numerical scale across models.

\textbf{Zero-shot classification.} The heatmap prompt describes the ground-truth
class, including when the model predicts a different class. For example, the
first APTOS example uses ``Fundus photograph showing proliferative diabetic
retinopathy.'' The first Pneumonia CXR example uses ``Chest X-ray showing
pneumonia.'' The classification evaluation may average several prompts per
class, whereas each heatmap explains the single prompt used for that example.
Thus, the map is a response to the specified reference class, rather than
necessarily an explanation of the predicted class or the full classification
decision.

\textbf{Linear probing.} The linear classifier operates on image features and
uses no text. A class-related text prompt is introduced only to compute the
saliency map: the first Brain Tumor MRI example uses ``pituitary, medical
image'', and the first CT-Kidney example uses ``Tumor, medical image''.
The heatmaps explain image--prompt similarity. In particular, these maps are not
gradients of the linear-probe class score.

\textbf{Image-to-text retrieval.} The target text is the image's paired caption
or report, including the full report when available. The recorded diagnostic
label identifies the example but is not a replacement for the paired text.
The maps therefore show the spatial response to the reference text. They do not visualize
the content of a model's highest-ranked alternative report.

\subsection{Reading the comparative examples}
Figures~\ref{fig:saliency-aptos-2019}--\ref{fig:saliency-thyroid-us} show
zero-shot examples, Figures~\ref{fig:saliency-brain-tumor-mri}--\ref{fig:saliency-tn5000}
show linear-probe examples, and Figures~\ref{fig:saliency-derm7pt}--\ref{fig:saliency-pathgen-600}
show retrieval examples. Each caption identifies the three cases and explains
specific spatial responses. Interpret these patterns together with the target
text and the task-specific definitions above. Responses within an apparent
lesion can occur in several models and do not by themselves establish a
correct classification or a successful retrieval.

These visualizations provide descriptive evidence about the models' spatial
responses to text. They do not include a quantitative localization evaluation
or establish that a highlighted region is causally necessary for a prediction.
The selected examples complement the aggregate results in
Tables~\ref{tab:zs}, \ref{tab:lp}, and~\ref{tab:ret}; they do not establish that one
model localizes disease better across a complete dataset or isolate the effect
of routing from other differences between models.

\clearpage
\newcommand{\saliencycolumns}{%
\noindent\hspace*{0.071\linewidth}%
\makebox[0.176\linewidth][c]{\scriptsize\shortstack{Proposed\\technique}}%
\hspace*{0.009\linewidth}\makebox[0.176\linewidth][c]{\scriptsize BMC}%
\hspace*{0.009\linewidth}\makebox[0.176\linewidth][c]{\scriptsize BMCA}%
\hspace*{0.009\linewidth}\makebox[0.176\linewidth][c]{\scriptsize ConceptCLIP}%
\hspace*{0.009\linewidth}\makebox[0.176\linewidth][c]{\scriptsize UniMed-CLIP}\par}

\begin{figure}[p]
\centering\begin{minipage}{5.5in}
\centering
\saliencycolumns
\vspace{3pt}
{\footnotesize (a)}\\[1pt]
\includegraphics[width=\linewidth]{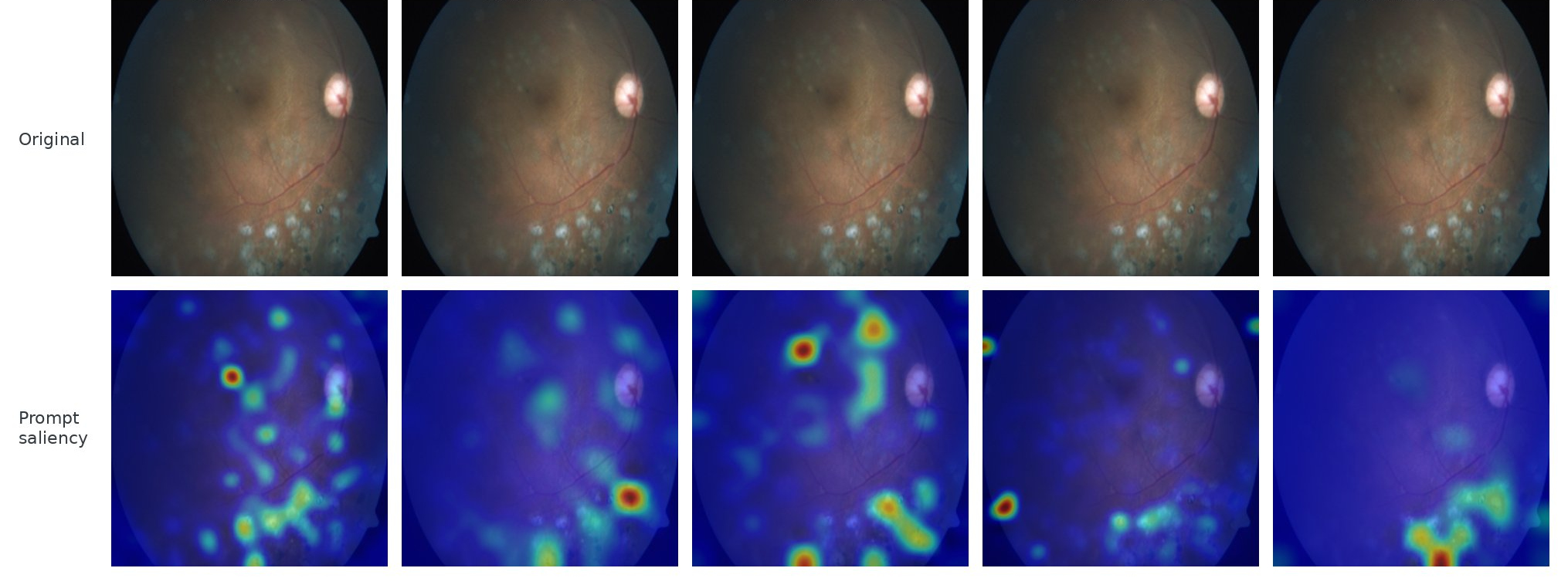}\\[4pt]
{\footnotesize (b)}\\[1pt]
\includegraphics[width=\linewidth]{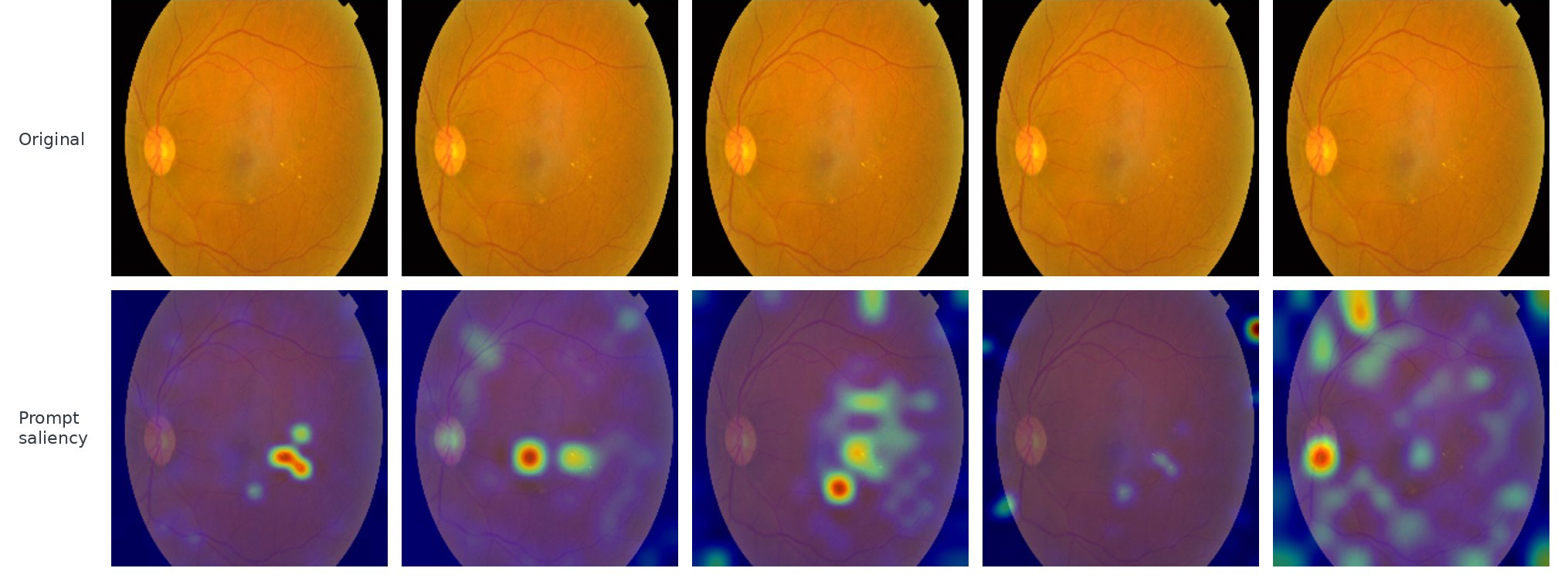}\\[4pt]
{\footnotesize (c)}\\[1pt]
\includegraphics[width=\linewidth]{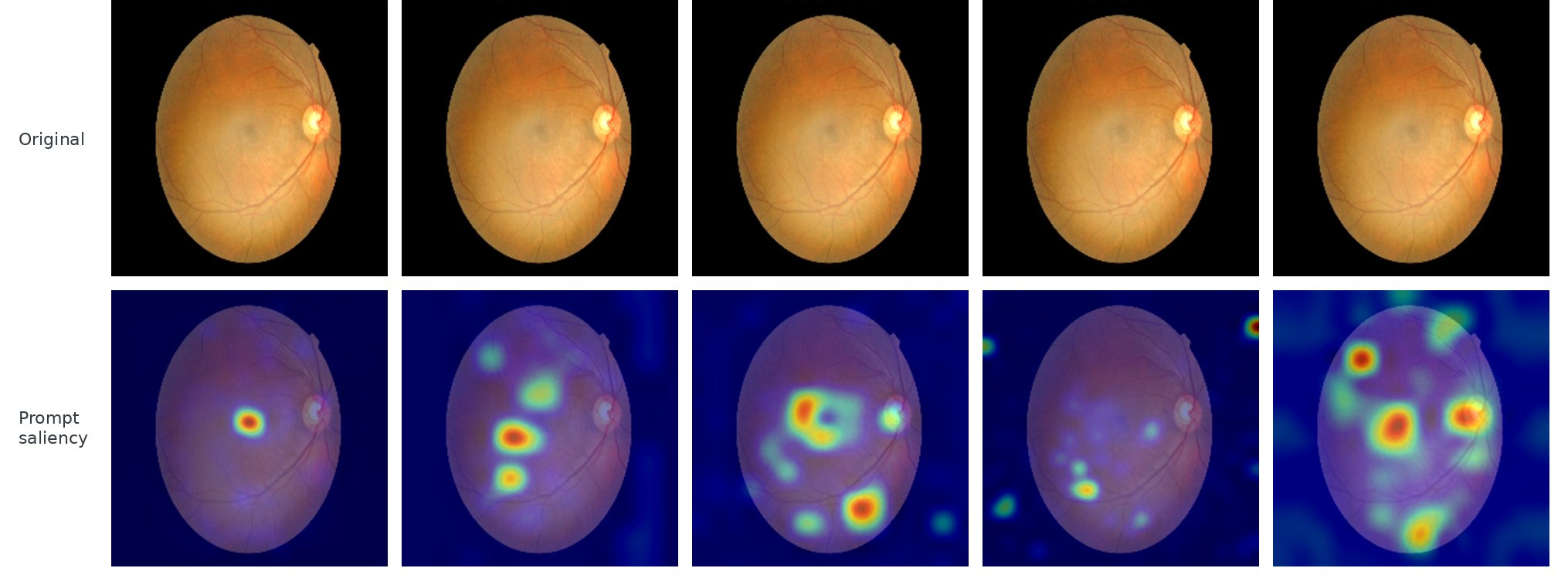}
\caption{\textbf{APTOS 2019 zero-shot classification.} Examples: (a) proliferative diabetic retinopathy, (b) moderate diabetic retinopathy, and (c) no diabetic retinopathy. In panel (a), the proposed technique highlights several locations within the retinal field, including the lower portion of the image. BMC and BMCA also respond in this area, while ConceptCLIP includes peripheral foci and UniMed-CLIP concentrates much of its response near the lower edge. The comparison shows how a common retinopathy prompt produces different spatial distributions. Heatmaps use ground-truth-class descriptions. Saliency construction follows Section~\ref{app:comparative-saliency-method}.}
\label{fig:saliency-aptos-2019}
\end{minipage}
\end{figure}
\clearpage

\begin{figure}[p]
\centering\begin{minipage}{5.5in}
\centering
\saliencycolumns
\vspace{3pt}
{\footnotesize (a)}\\[1pt]
\includegraphics[width=\linewidth]{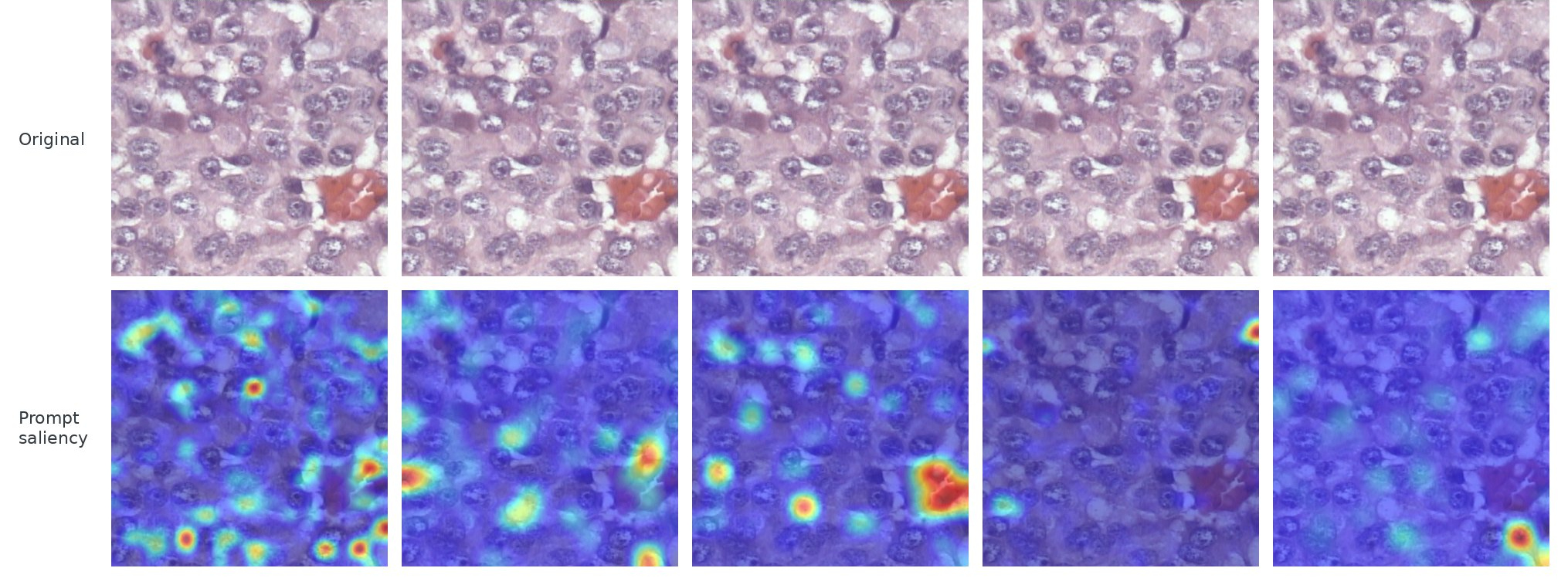}\\[4pt]
{\footnotesize (b)}\\[1pt]
\includegraphics[width=\linewidth]{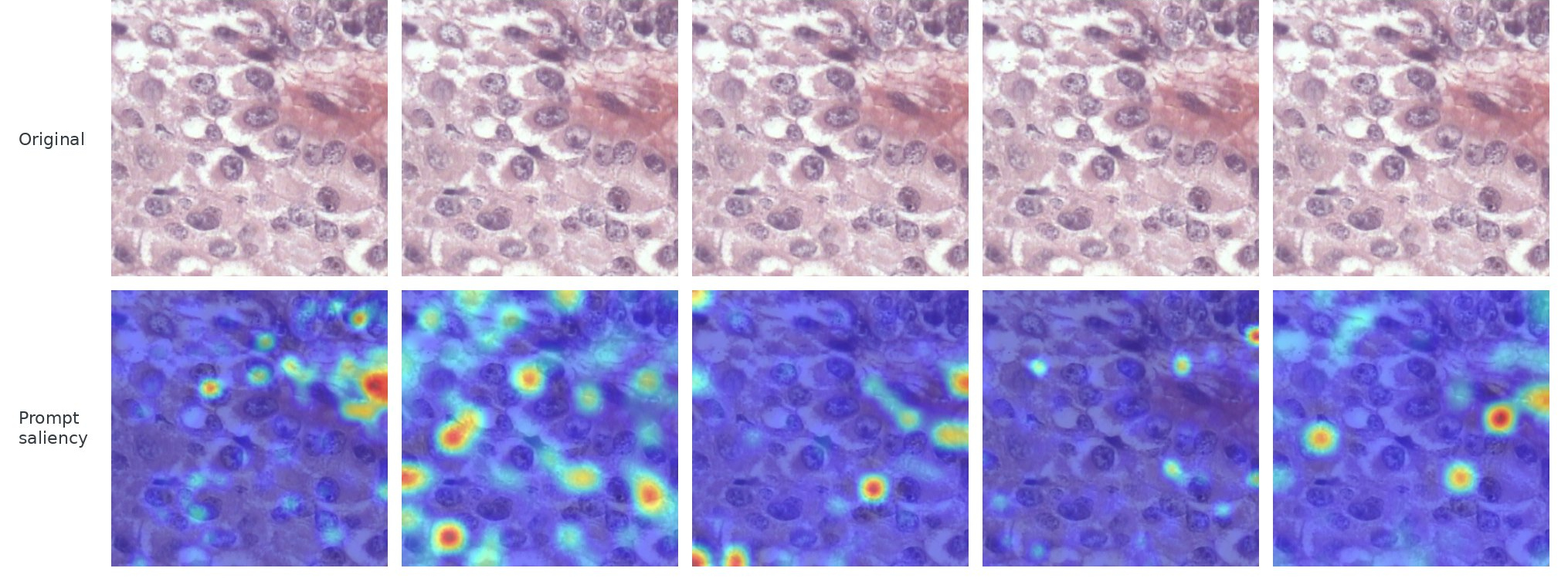}\\[4pt]
{\footnotesize (c)}\\[1pt]
\includegraphics[width=\linewidth]{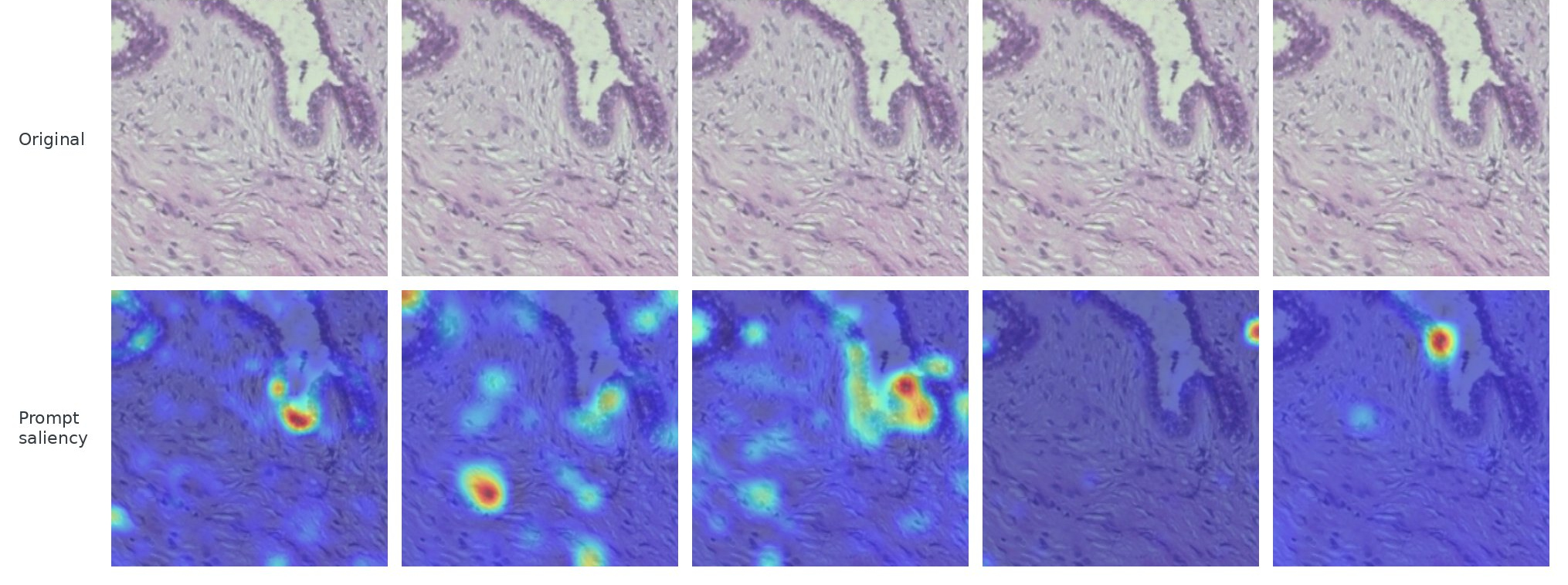}
\caption{\textbf{BreakHis zero-shot classification.} Examples: (a) malignant breast tumor, (b) malignant breast tumor, and (c) benign breast tumor. In panel (a), the proposed technique produces multiple foci across the cellular tissue, whereas BMC and BMCA emphasize a smaller set of locations. ConceptCLIP shows few prominent foci in this example, and UniMed-CLIP includes a response near the lower-right boundary. The overlays describe spatial responses to the tissue-class prompt without delineating a tumor boundary. Heatmaps use ground-truth-class descriptions. Saliency construction follows Section~\ref{app:comparative-saliency-method}.}
\label{fig:saliency-breakhis}
\end{minipage}
\end{figure}
\clearpage

\begin{figure}[p]
\centering\begin{minipage}{5.5in}
\centering
\saliencycolumns
\vspace{3pt}
{\footnotesize (a)}\\[1pt]
\includegraphics[width=\linewidth]{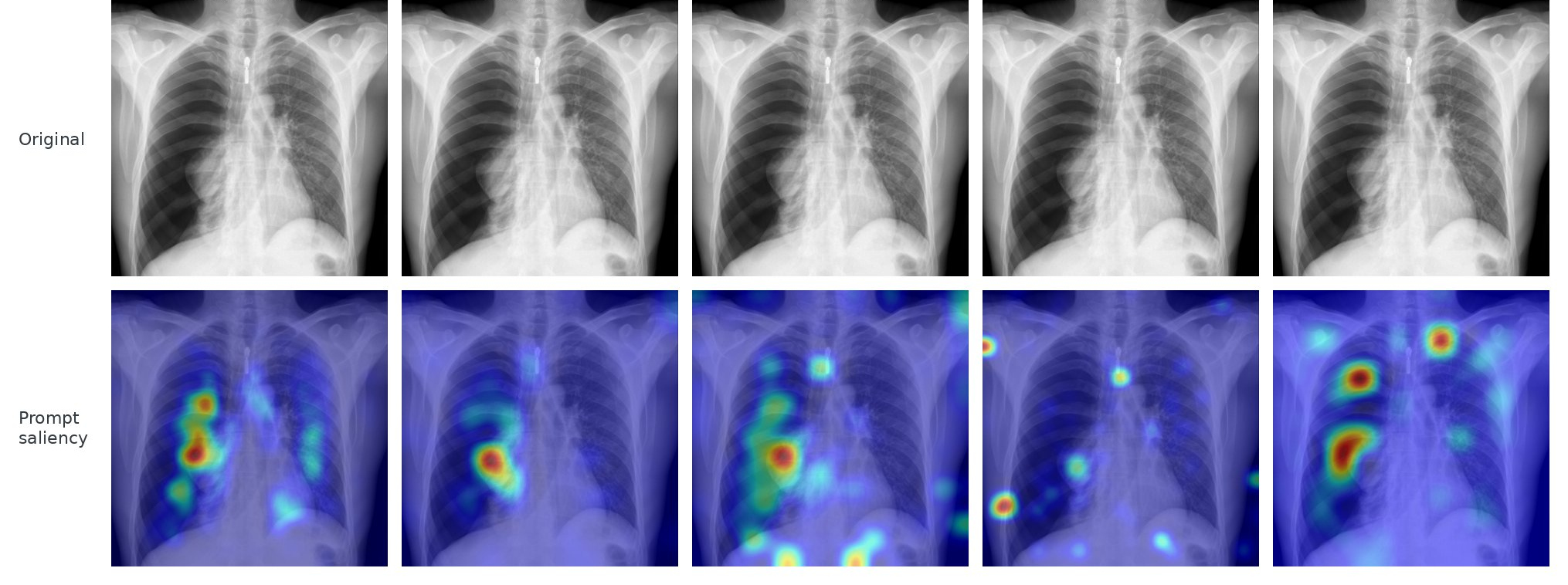}\\[4pt]
{\footnotesize (b)}\\[1pt]
\includegraphics[width=\linewidth]{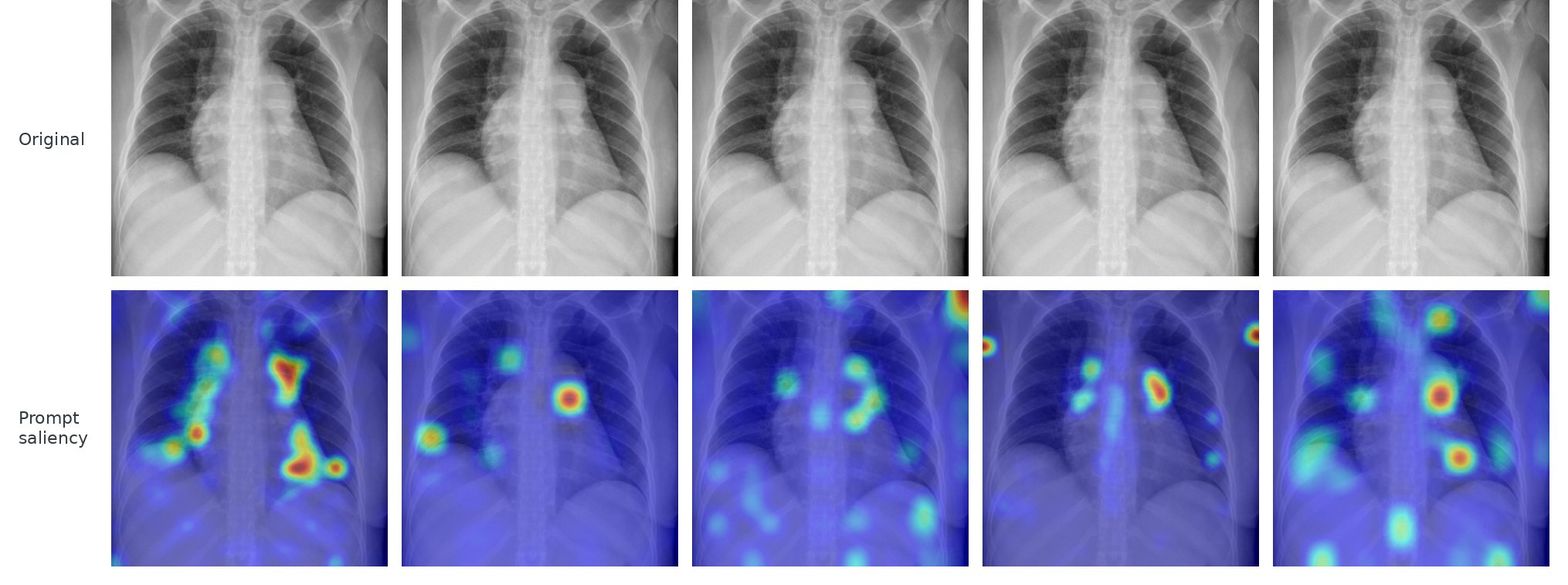}\\[4pt]
{\footnotesize (c)}\\[1pt]
\includegraphics[width=\linewidth]{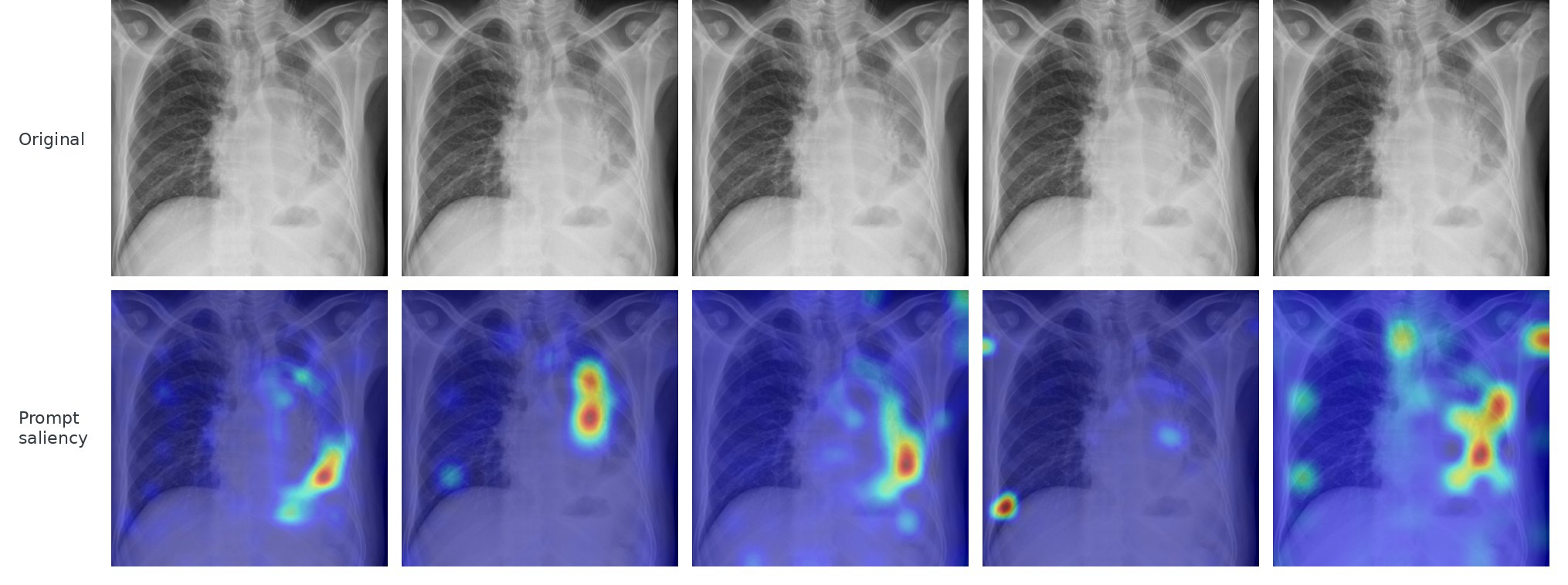}
\caption{\textbf{CASIA-CXR zero-shot classification.} Examples: (a) pneumothorax, (b) cardiomegaly, and (c) pleural effusion. In panel (a), the proposed technique, BMC, and BMCA show prominent foci within the lung field on the left side of the displayed image. ConceptCLIP distributes its foci differently, and UniMed-CLIP includes responses higher in the thorax. The differing maps illustrate sensitivity to the same finding description; they do not delineate the extent of pneumothorax. Heatmaps use ground-truth-class descriptions. Saliency construction follows Section~\ref{app:comparative-saliency-method}.}
\label{fig:saliency-casia-cxr}
\end{minipage}
\end{figure}
\clearpage

\begin{figure}[p]
\centering\begin{minipage}{5.5in}
\centering
\saliencycolumns
\vspace{3pt}
{\footnotesize (a)}\\[1pt]
\includegraphics[width=\linewidth]{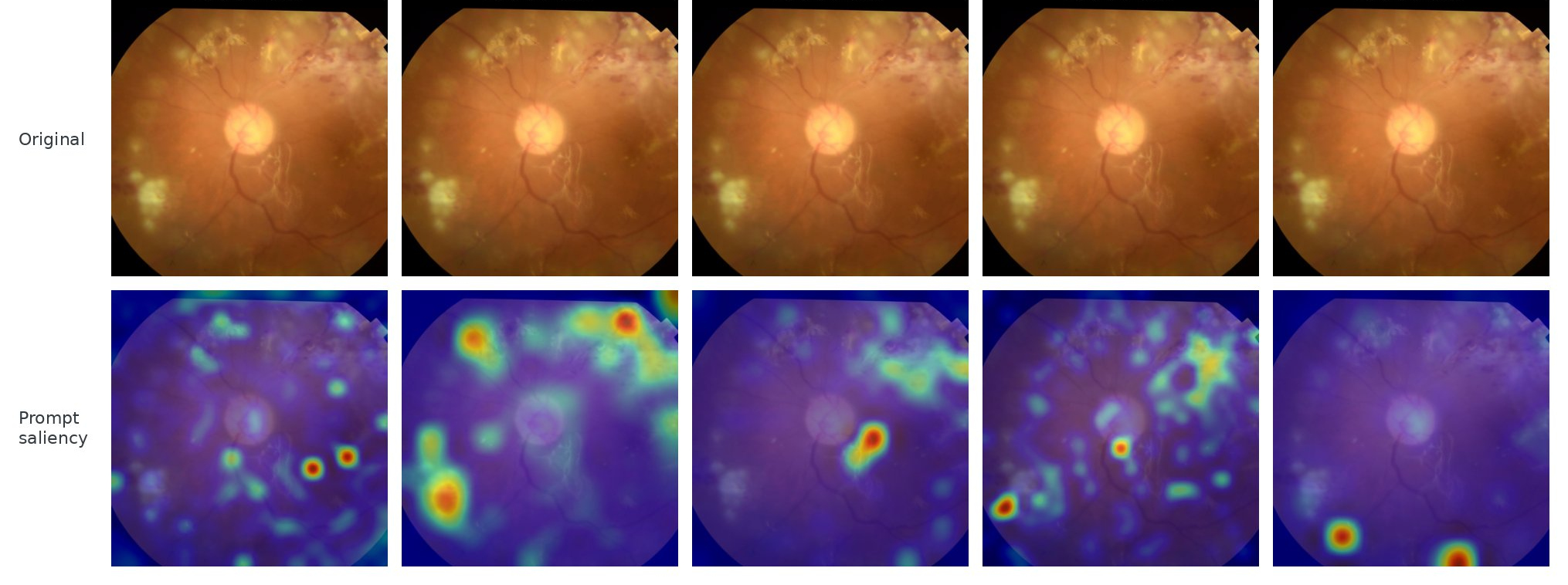}\\[4pt]
{\footnotesize (b)}\\[1pt]
\includegraphics[width=\linewidth]{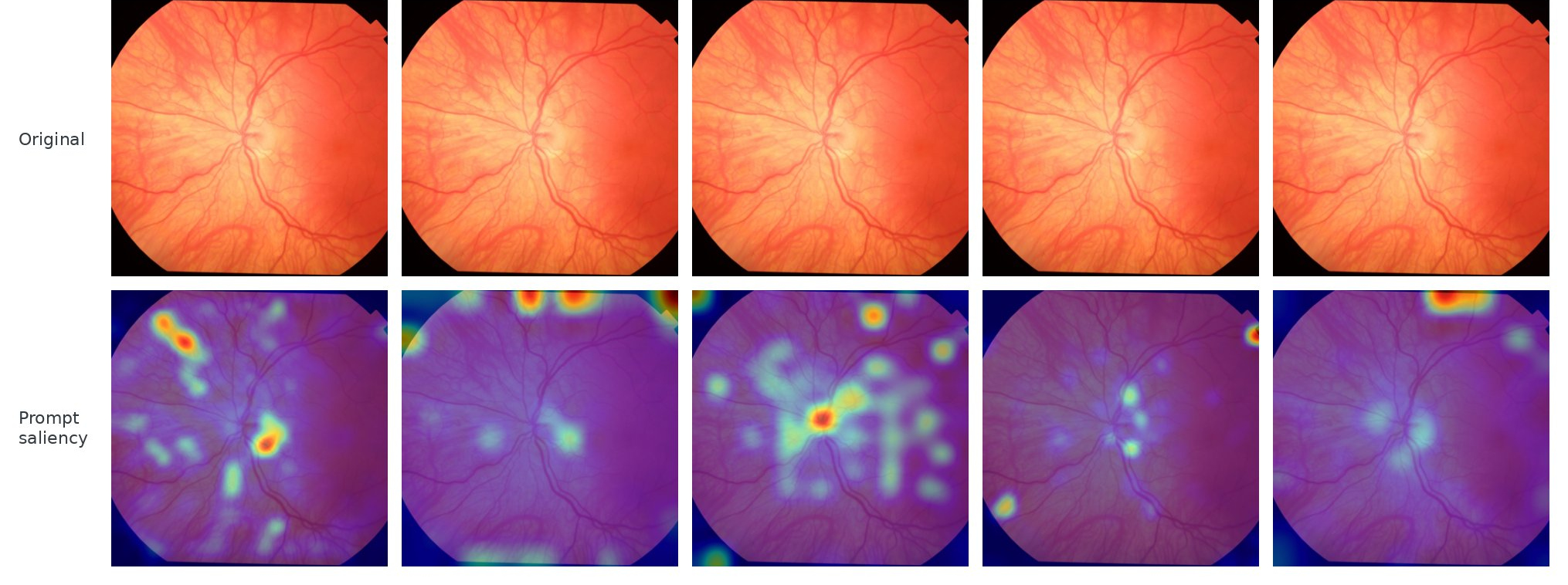}\\[4pt]
{\footnotesize (c)}\\[1pt]
\includegraphics[width=\linewidth]{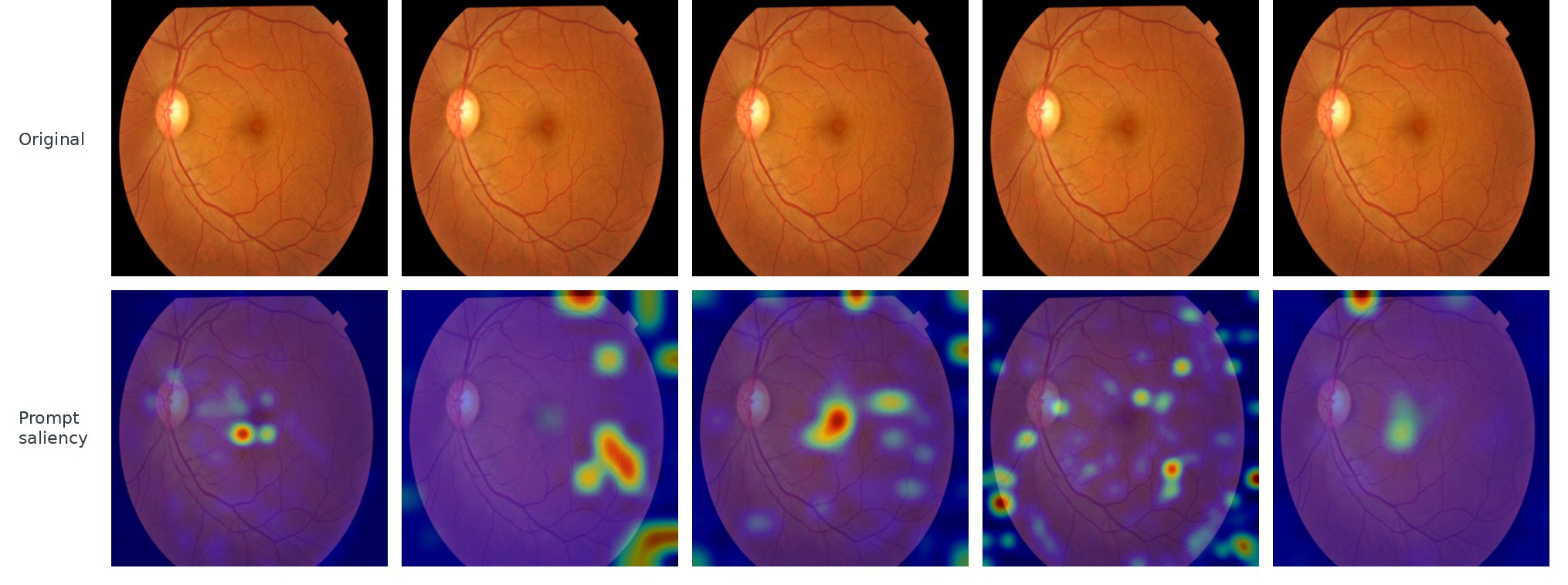}
\caption{\textbf{Color Fundus zero-shot classification.} Examples: (a) diabetic retinopathy, (b) myopia, and (c) healthy. In panel (a), the proposed technique produces several localized retinal responses, while BMC shows broader responses toward the upper and peripheral retinal field. BMCA includes a prominent focus near the optic disc, ConceptCLIP shows several scattered foci, and UniMed-CLIP emphasizes lower peripheral locations. The figure illustrates that a shared class description can elicit different responses within the same fundus image. Heatmaps use ground-truth-class descriptions. Saliency construction follows Section~\ref{app:comparative-saliency-method}.}
\label{fig:saliency-color-fundus}
\end{minipage}
\end{figure}
\clearpage

\begin{figure}[p]
\centering\begin{minipage}{5.5in}
\centering
\saliencycolumns
\vspace{3pt}
{\footnotesize (a)}\\[1pt]
\includegraphics[width=\linewidth]{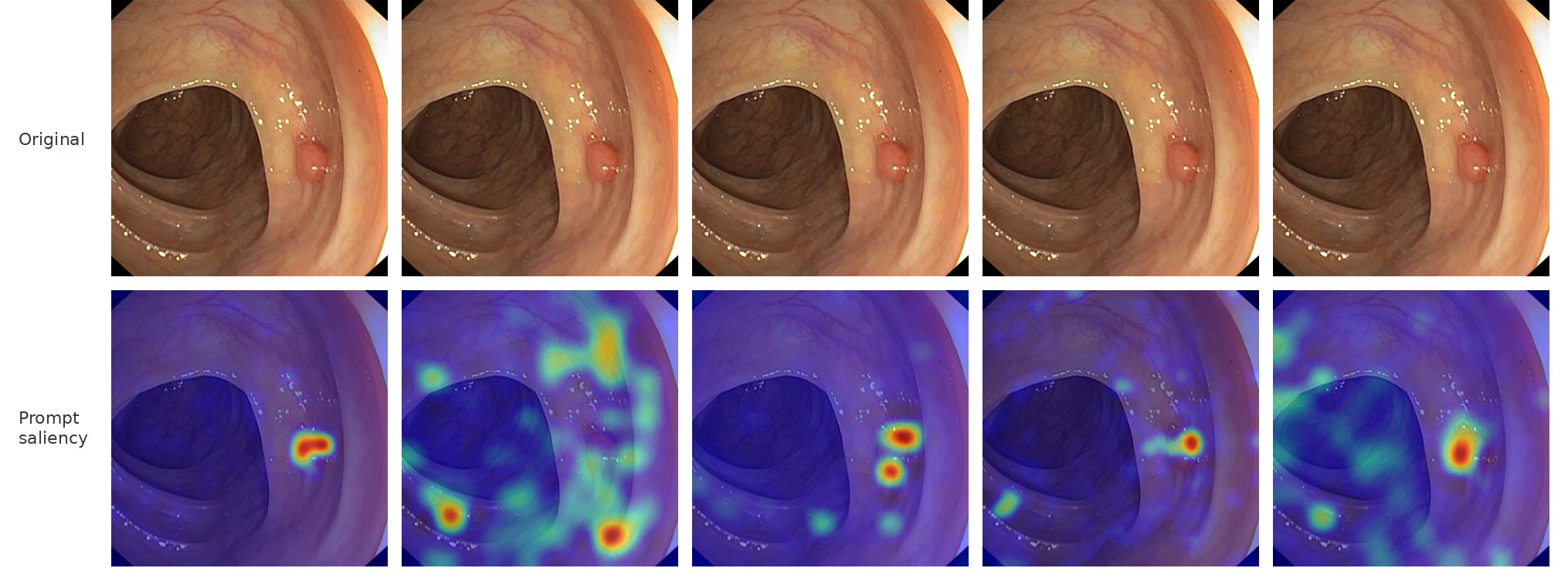}\\[4pt]
{\footnotesize (b)}\\[1pt]
\includegraphics[width=\linewidth]{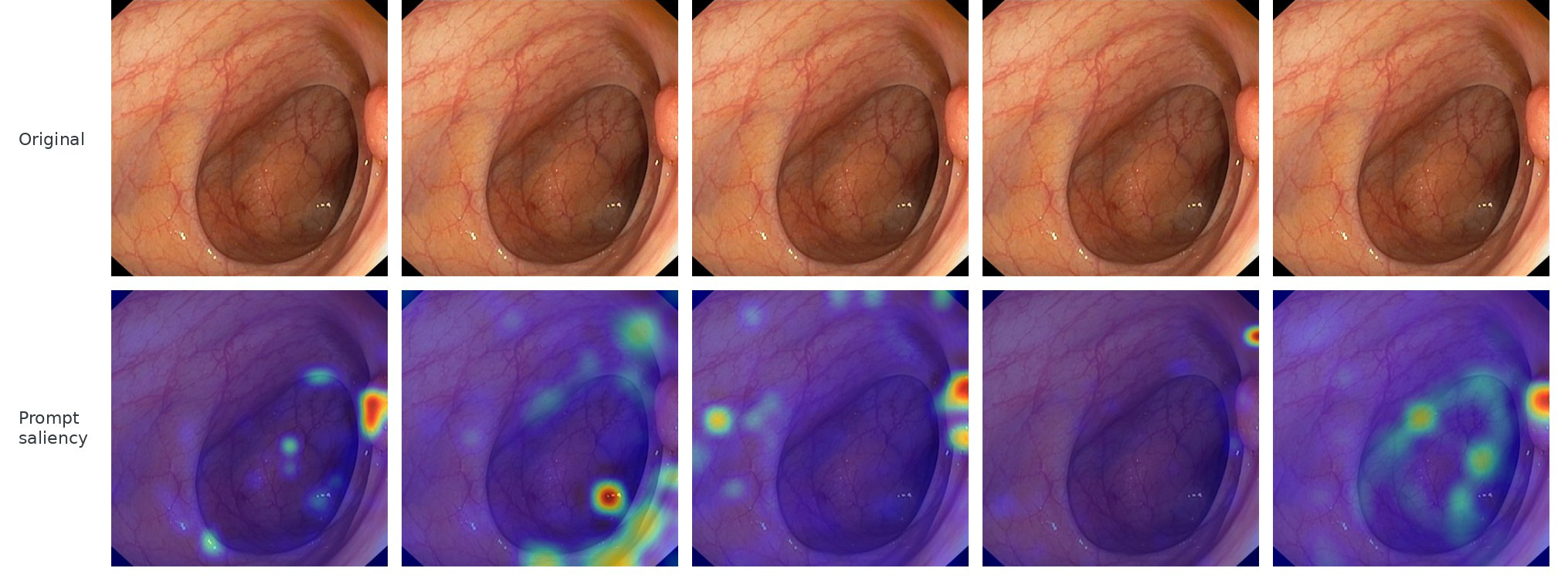}\\[4pt]
{\footnotesize (c)}\\[1pt]
\includegraphics[width=\linewidth]{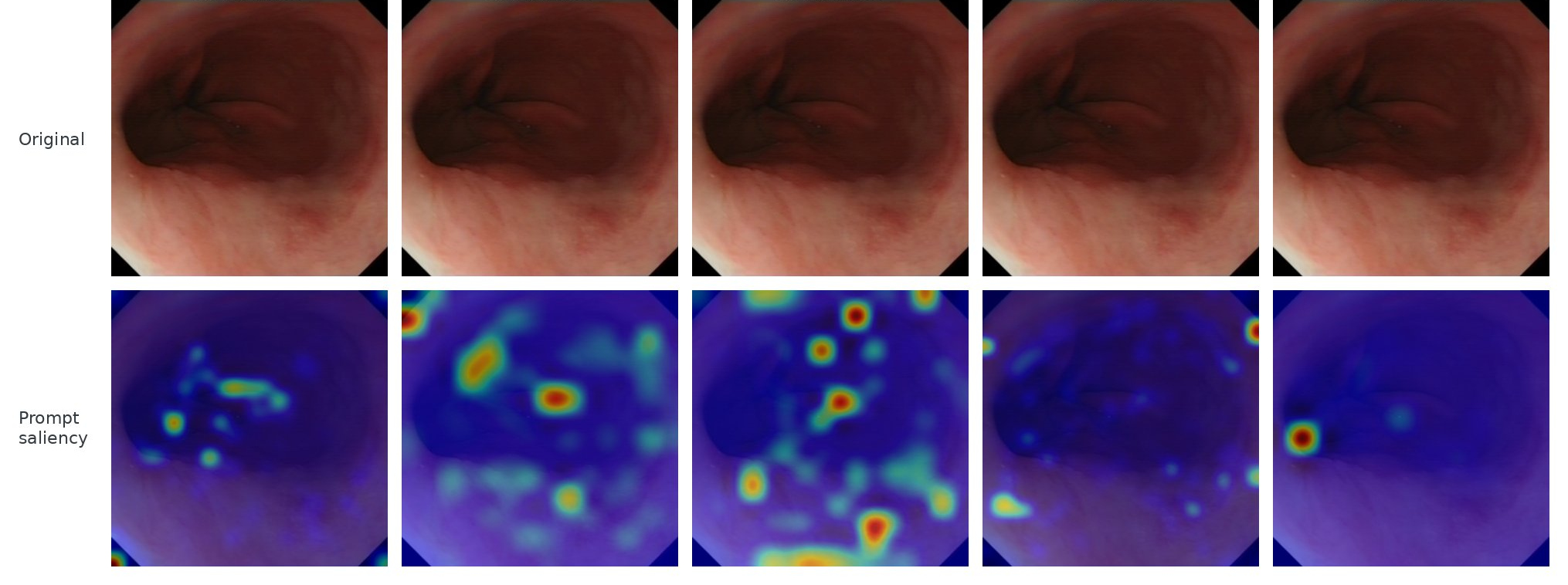}
\caption{\textbf{EDD2020 zero-shot classification.} Examples: (a) polyp-positive, (b) polyp-positive, and (c) polyp-negative. In panel (a), the proposed technique shows a compact response at the raised structure on the right side of the image. BMCA, ConceptCLIP, and UniMed-CLIP also include a response in this area, while BMC's overlay extends across more of the surrounding mucosa. This example contrasts localized and distributed responses to the polyp-related prompt. Heatmaps use ground-truth-class descriptions. Saliency construction follows Section~\ref{app:comparative-saliency-method}.}
\label{fig:saliency-edd2020}
\end{minipage}
\end{figure}
\clearpage

\begin{figure}[p]
\centering\begin{minipage}{5.5in}
\centering
\saliencycolumns
\vspace{3pt}
{\footnotesize (a)}\\[1pt]
\includegraphics[width=\linewidth]{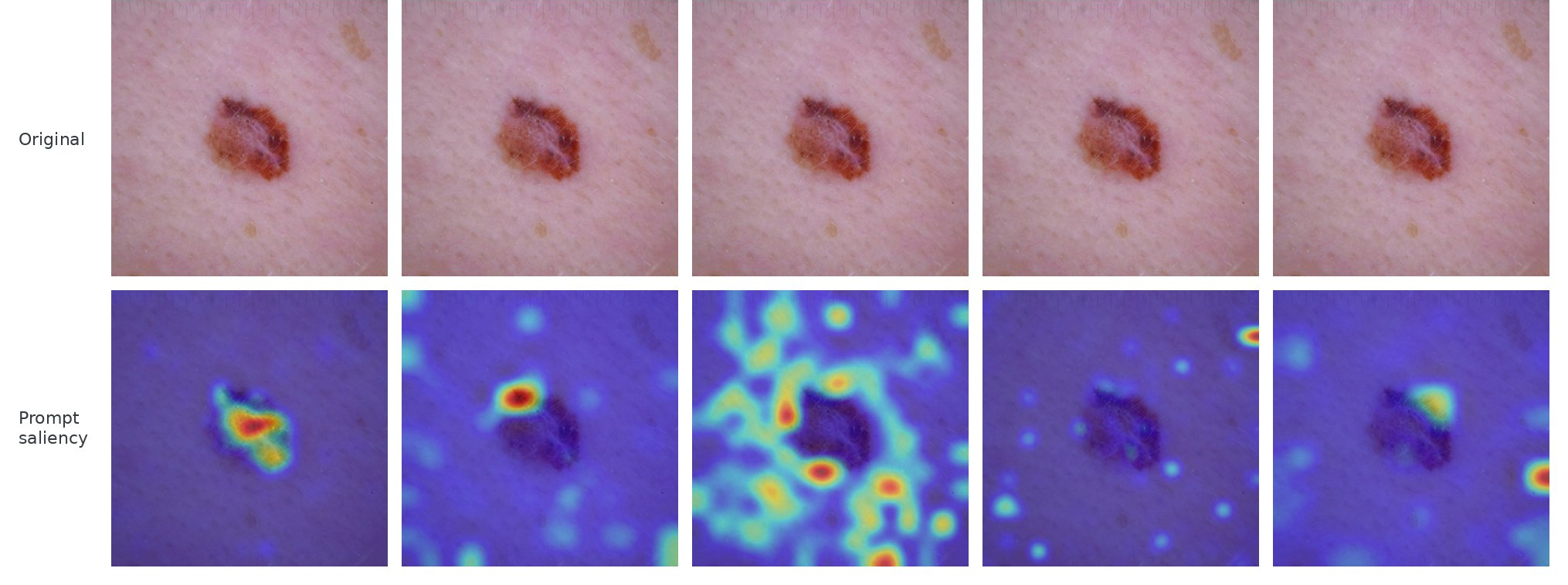}\\[4pt]
{\footnotesize (b)}\\[1pt]
\includegraphics[width=\linewidth]{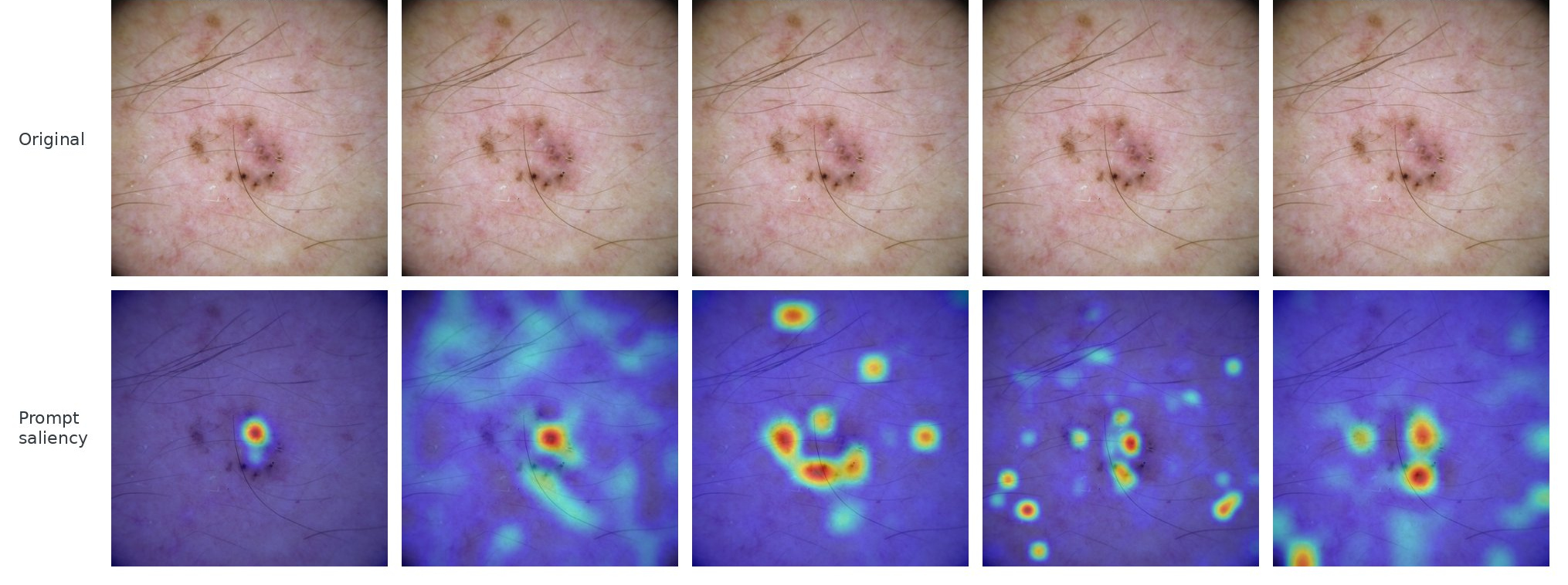}\\[4pt]
{\footnotesize (c)}\\[1pt]
\includegraphics[width=\linewidth]{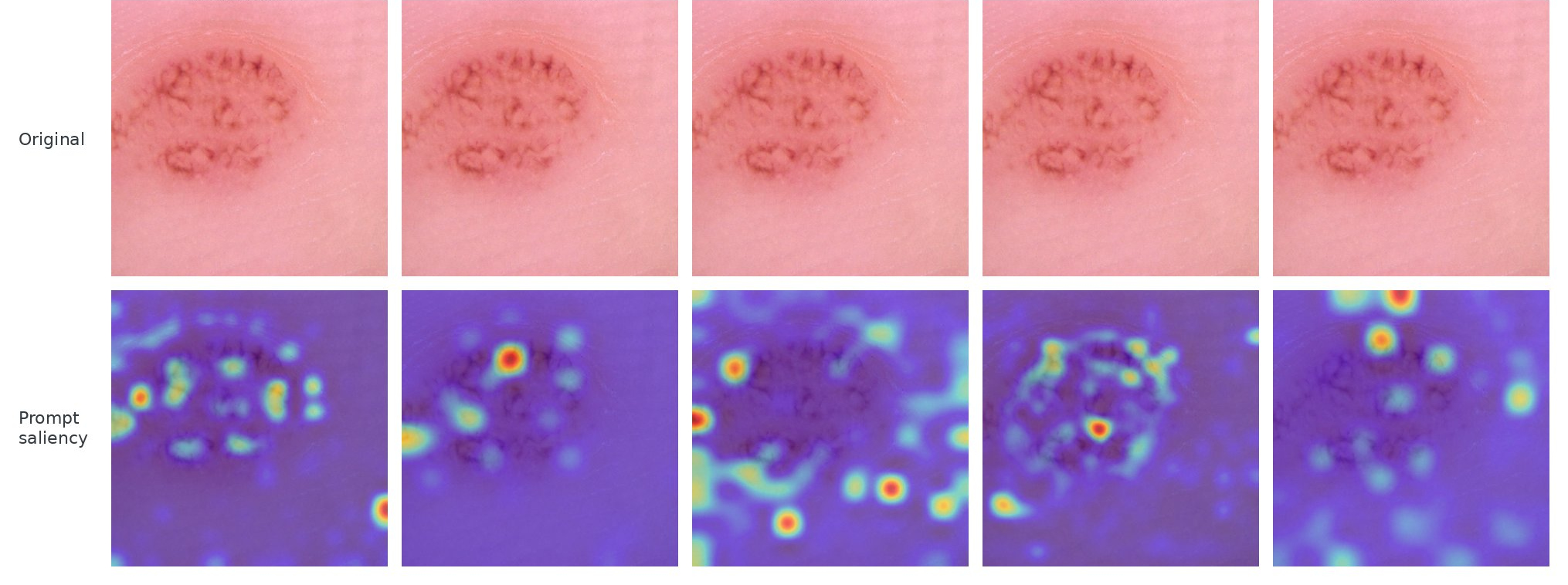}
\caption{\textbf{HAM10000 zero-shot classification.} Examples: (a) melanoma, (b) basal cell carcinoma, and (c) melanocytic nevus. In panel (a), the proposed technique shows a concentrated response over part of the pigmented lesion. BMC also highlights a lesion-associated area, while BMCA distributes responses over both the lesion and surrounding skin. ConceptCLIP and UniMed-CLIP show different, more limited foci. The overlays compare responses to the melanoma description without supplying a lesion segmentation. Heatmaps use ground-truth-class descriptions. Saliency construction follows Section~\ref{app:comparative-saliency-method}.}
\label{fig:saliency-ham10000}
\end{minipage}
\end{figure}
\clearpage

\begin{figure}[p]
\centering\begin{minipage}{5.5in}
\centering
\saliencycolumns
\vspace{3pt}
{\footnotesize (a)}\\[1pt]
\includegraphics[width=\linewidth]{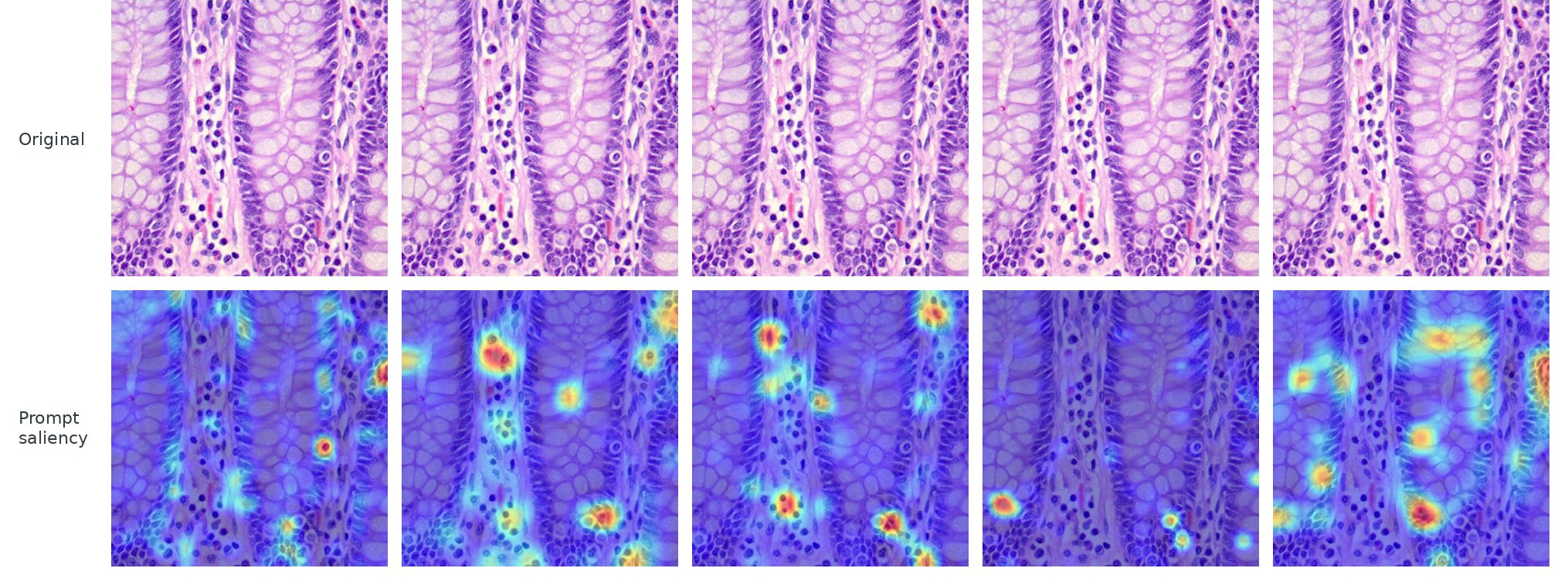}\\[4pt]
{\footnotesize (b)}\\[1pt]
\includegraphics[width=\linewidth]{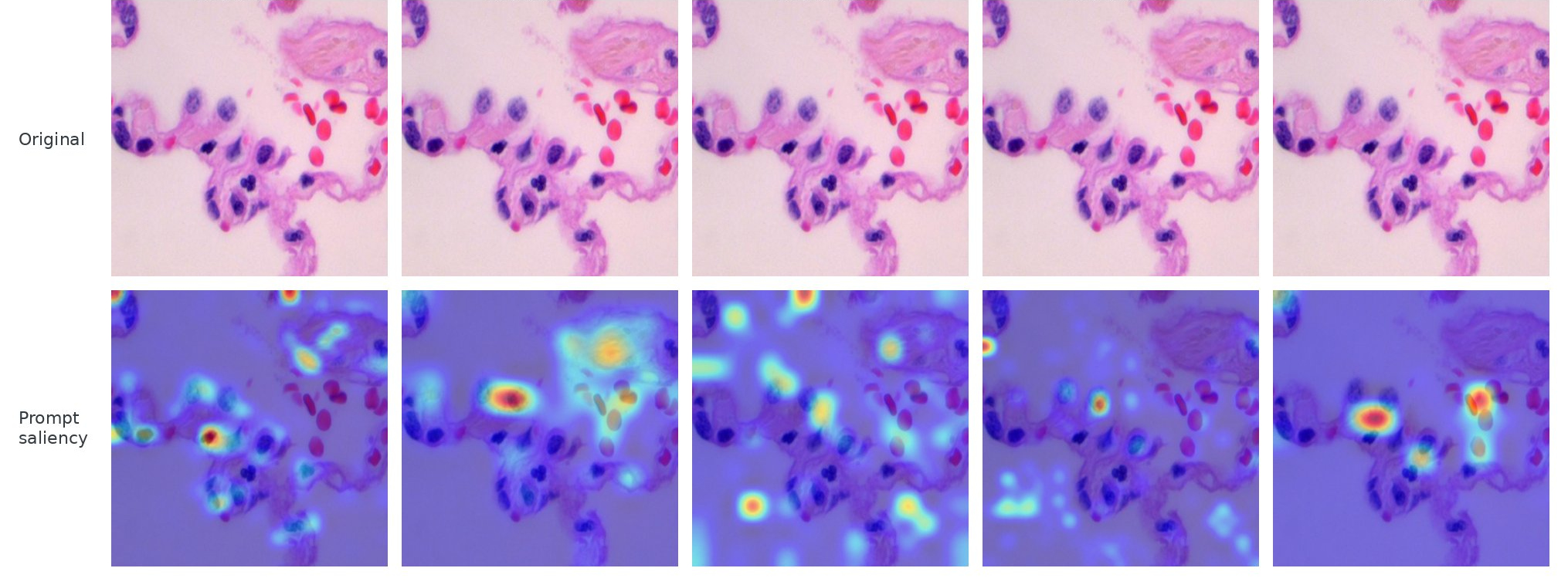}\\[4pt]
{\footnotesize (c)}\\[1pt]
\includegraphics[width=\linewidth]{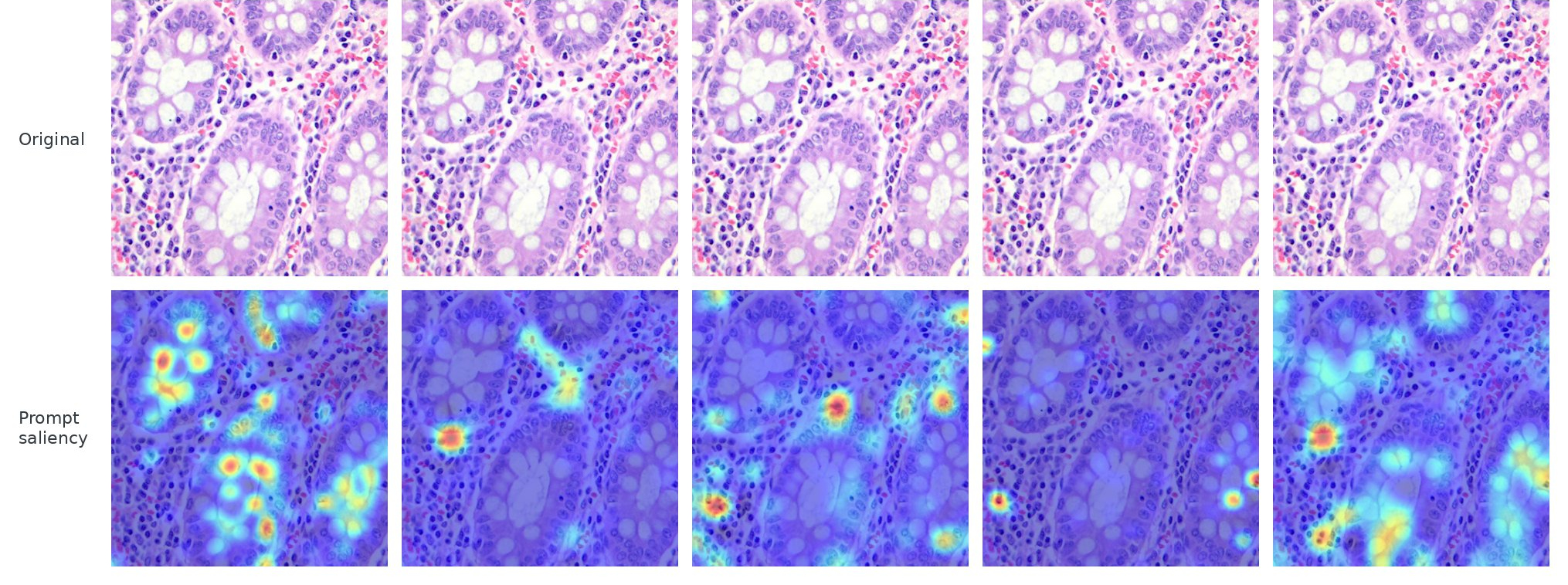}
\caption{\textbf{LC25000 zero-shot classification.} Examples: (a) colon tissue, (b) lung tissue, and (c) colon tissue. In panel (a), the proposed technique's responses are distributed among the visible tissue structures. BMC, BMCA, and UniMed-CLIP place prominent foci at different positions within the same field, while ConceptCLIP highlights fewer locations. The comparison concerns the spatial response to the colon-tissue description, which can depend on tissue context rather than a single lesion. Heatmaps use ground-truth-class descriptions. Saliency construction follows Section~\ref{app:comparative-saliency-method}.}
\label{fig:saliency-lc25000}
\end{minipage}
\end{figure}
\clearpage

\begin{figure}[p]
\centering\begin{minipage}{5.5in}
\centering
\saliencycolumns
\vspace{3pt}
{\footnotesize (a)}\\[1pt]
\includegraphics[width=\linewidth]{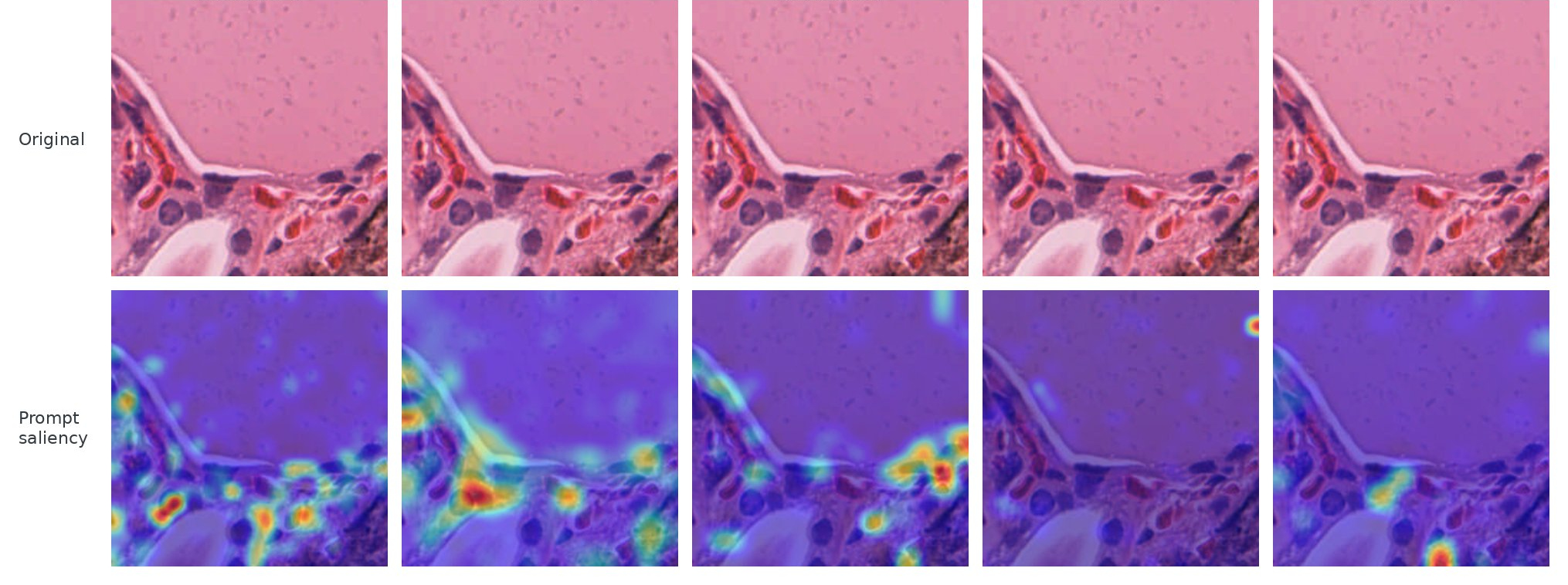}\\[4pt]
{\footnotesize (b)}\\[1pt]
\includegraphics[width=\linewidth]{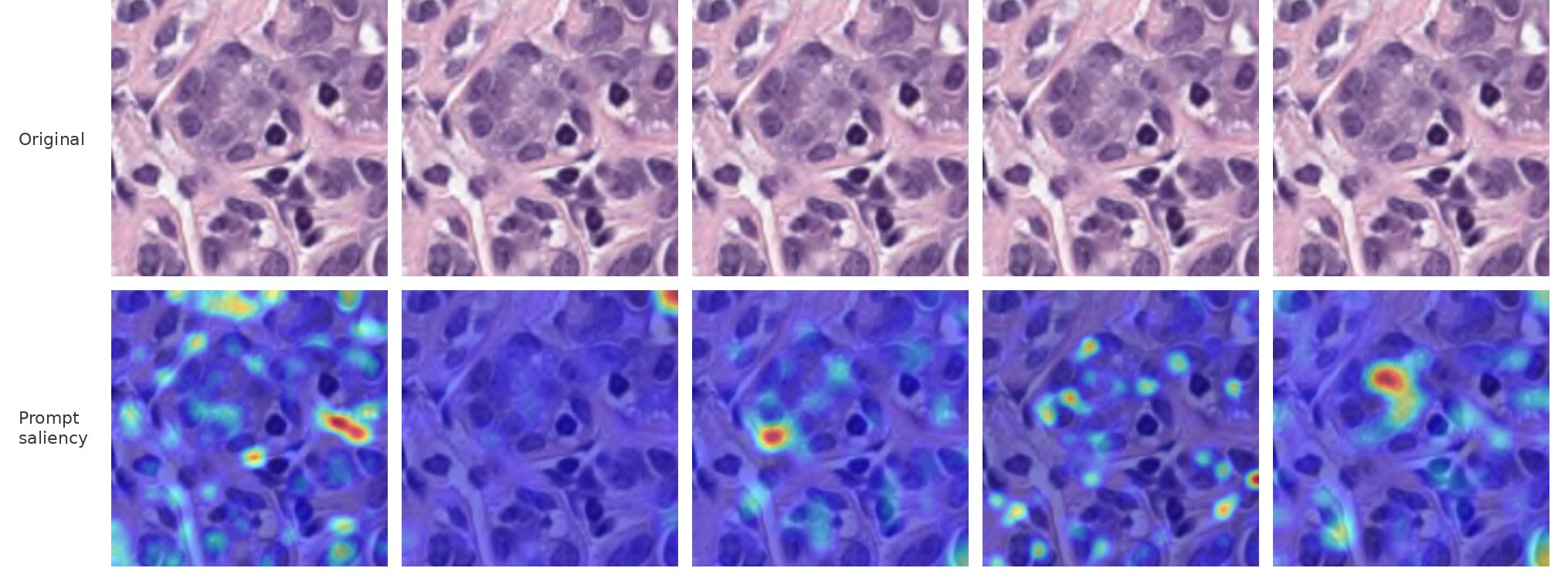}\\[4pt]
{\footnotesize (c)}\\[1pt]
\includegraphics[width=\linewidth]{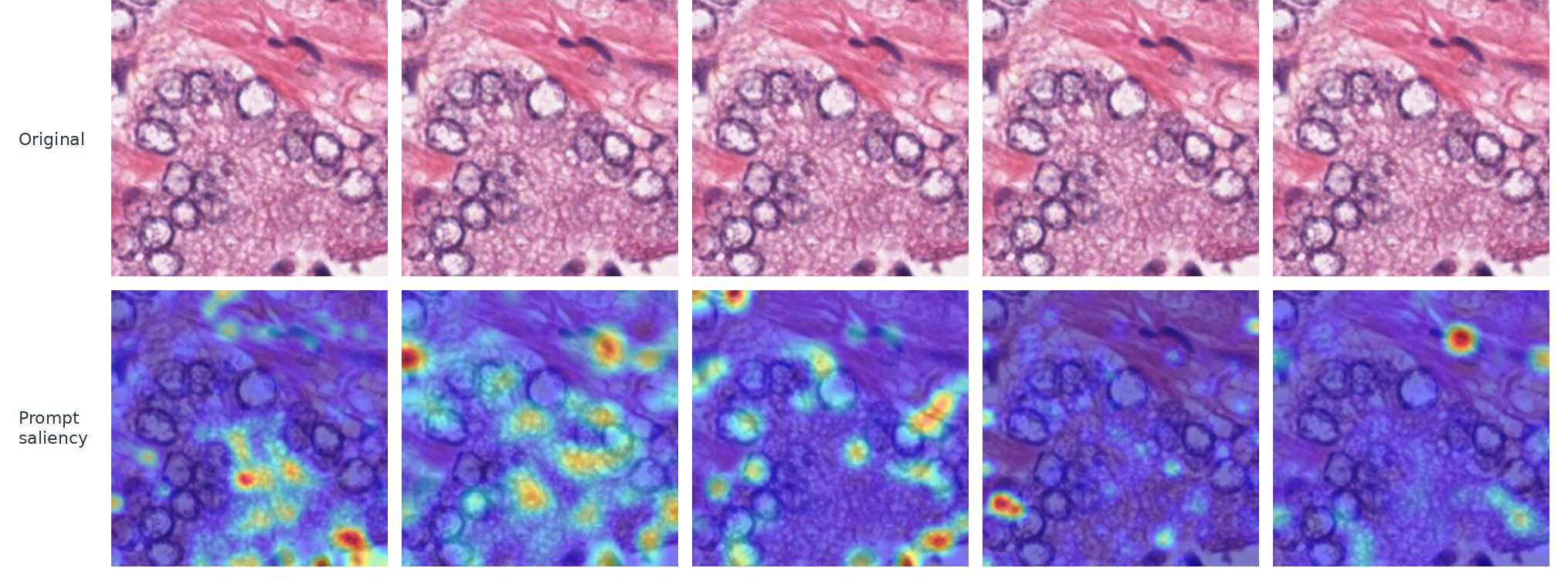}
\caption{\textbf{PanNuke zero-shot classification.} Examples: (a) thyroid, (b) breast, and (c) prostate. In panel (a), the proposed technique highlights several areas along the cellular band in the lower portion of the tissue image. BMC and BMCA also respond along this band, with different spatial extents, while ConceptCLIP and UniMed-CLIP emphasize other locations. These are responses to a tissue-class description, rather than nucleus-instance localization maps. Heatmaps use ground-truth-class descriptions. Saliency construction follows Section~\ref{app:comparative-saliency-method}.}
\label{fig:saliency-pannuke}
\end{minipage}
\end{figure}
\clearpage

\begin{figure}[p]
\centering\begin{minipage}{5.5in}
\centering
\saliencycolumns
\vspace{3pt}
{\footnotesize (a)}\\[1pt]
\includegraphics[width=\linewidth]{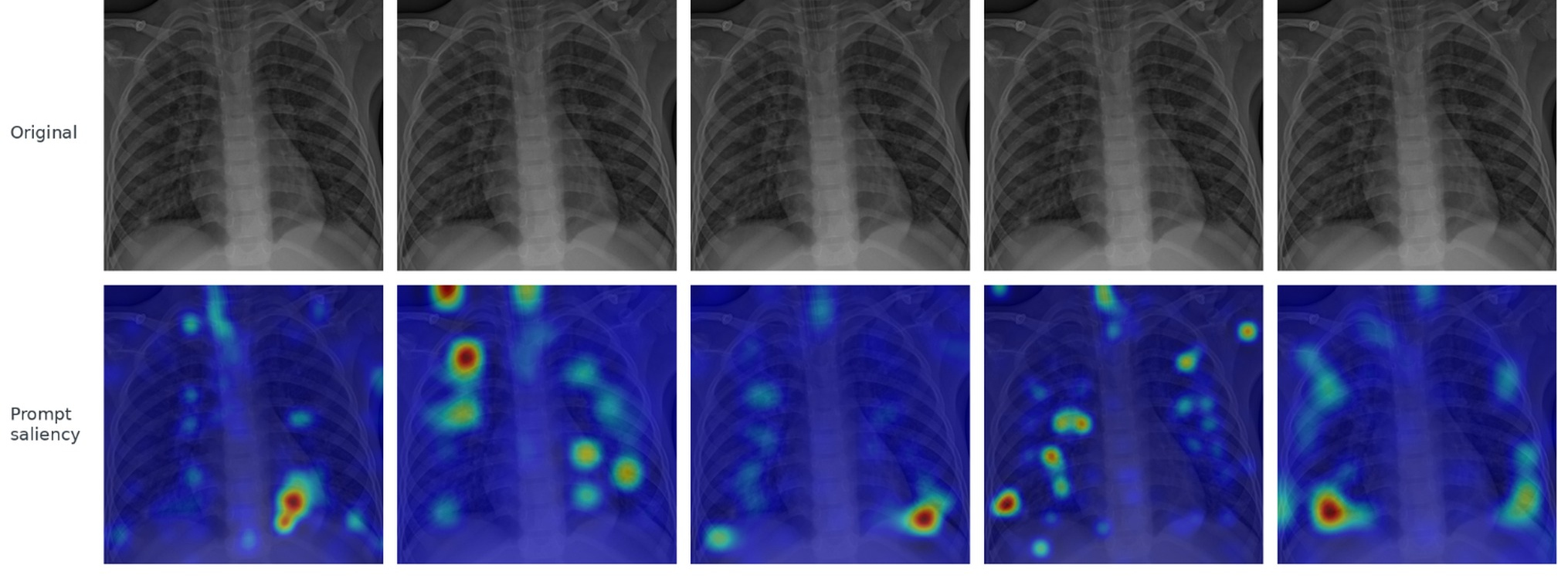}\\[4pt]
{\footnotesize (b)}\\[1pt]
\includegraphics[width=\linewidth]{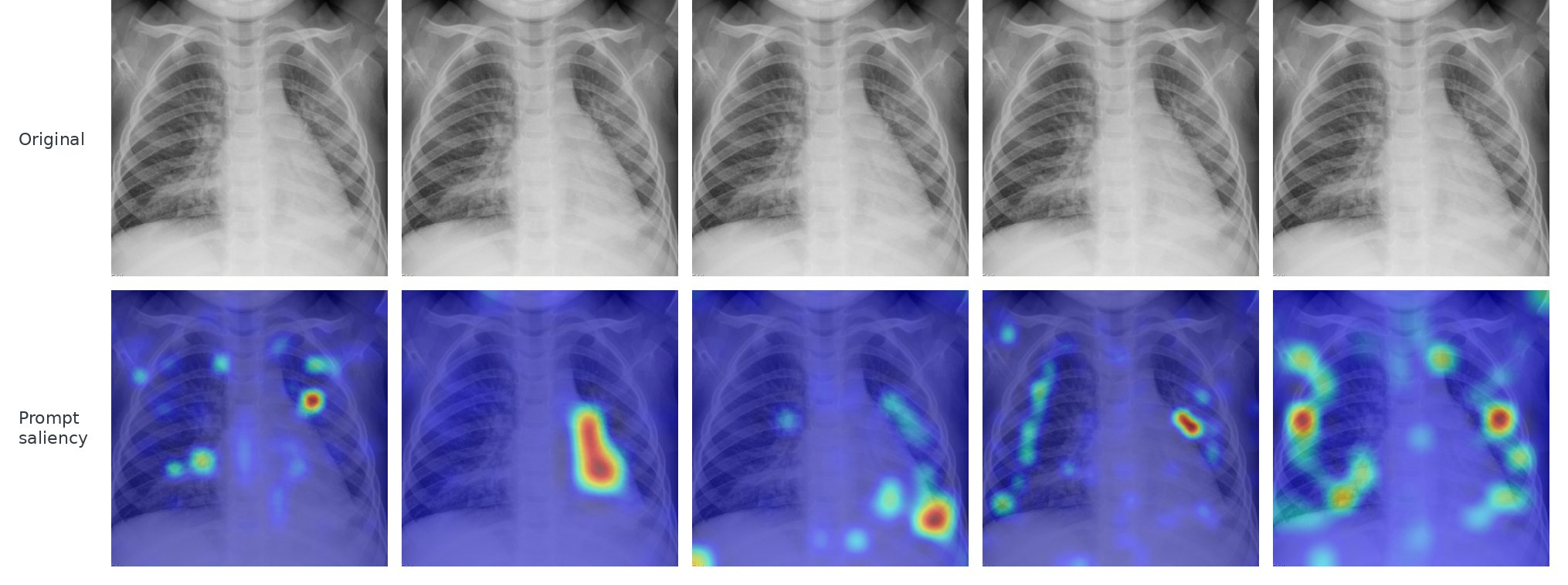}\\[4pt]
{\footnotesize (c)}\\[1pt]
\includegraphics[width=\linewidth]{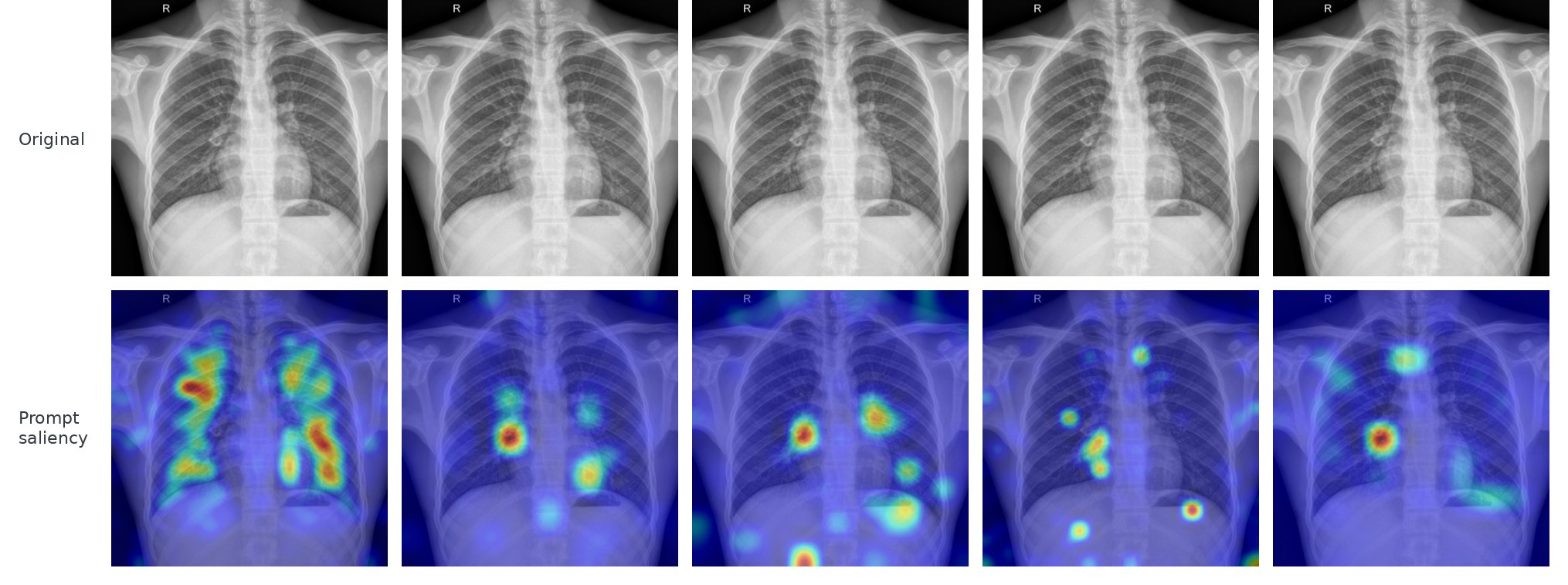}
\caption{\textbf{Pneumonia CXR zero-shot classification.} Examples: (a) pneumonia, (b) pneumonia, and (c) normal. In panel (a), the proposed technique includes a prominent focus in the lower portion of the thorax on the right side of the displayed image. BMC, BMCA, ConceptCLIP, and UniMed-CLIP distribute their foci differently across the lung fields and upper thorax. These differences concern the spatial response to the pneumonia prompt and do not measure localization accuracy. Heatmaps use ground-truth-class descriptions. Saliency construction follows Section~\ref{app:comparative-saliency-method}.}
\label{fig:saliency-pneumonia-cxr}
\end{minipage}
\end{figure}
\clearpage

\begin{figure}[p]
\centering\begin{minipage}{5.5in}
\centering
\saliencycolumns
\vspace{3pt}
{\footnotesize (a)}\\[1pt]
\includegraphics[width=\linewidth]{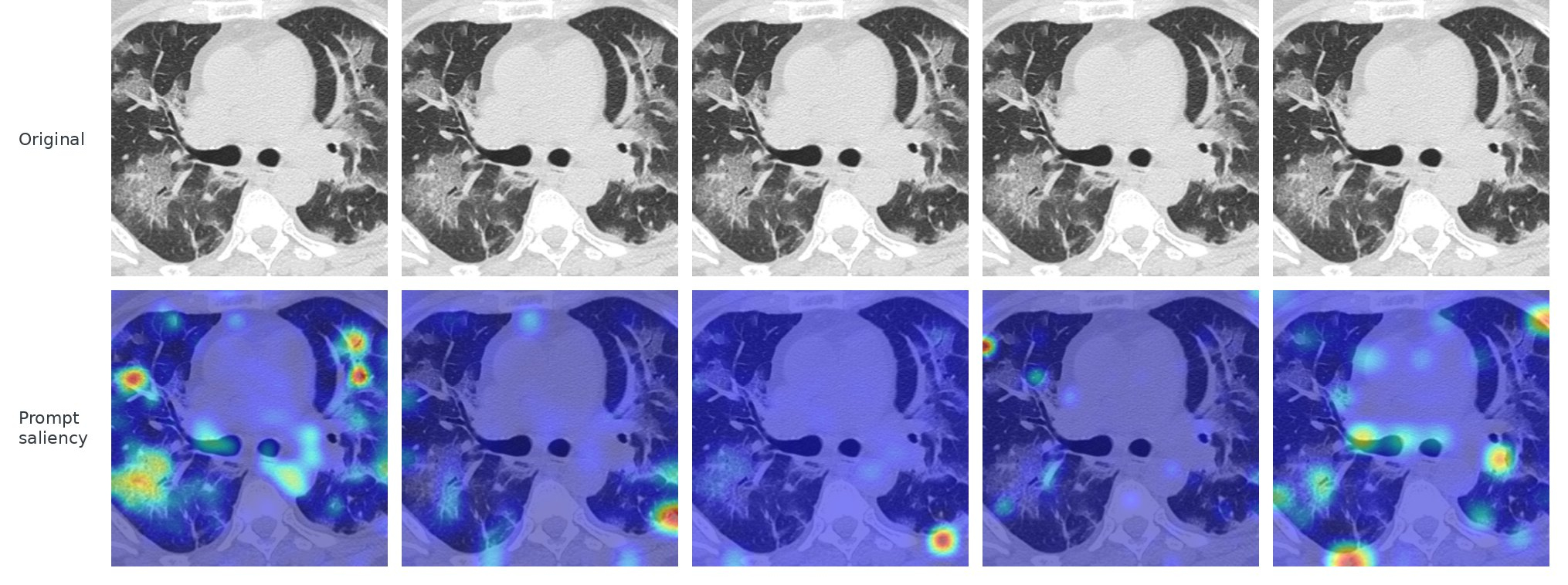}\\[4pt]
{\footnotesize (b)}\\[1pt]
\includegraphics[width=\linewidth]{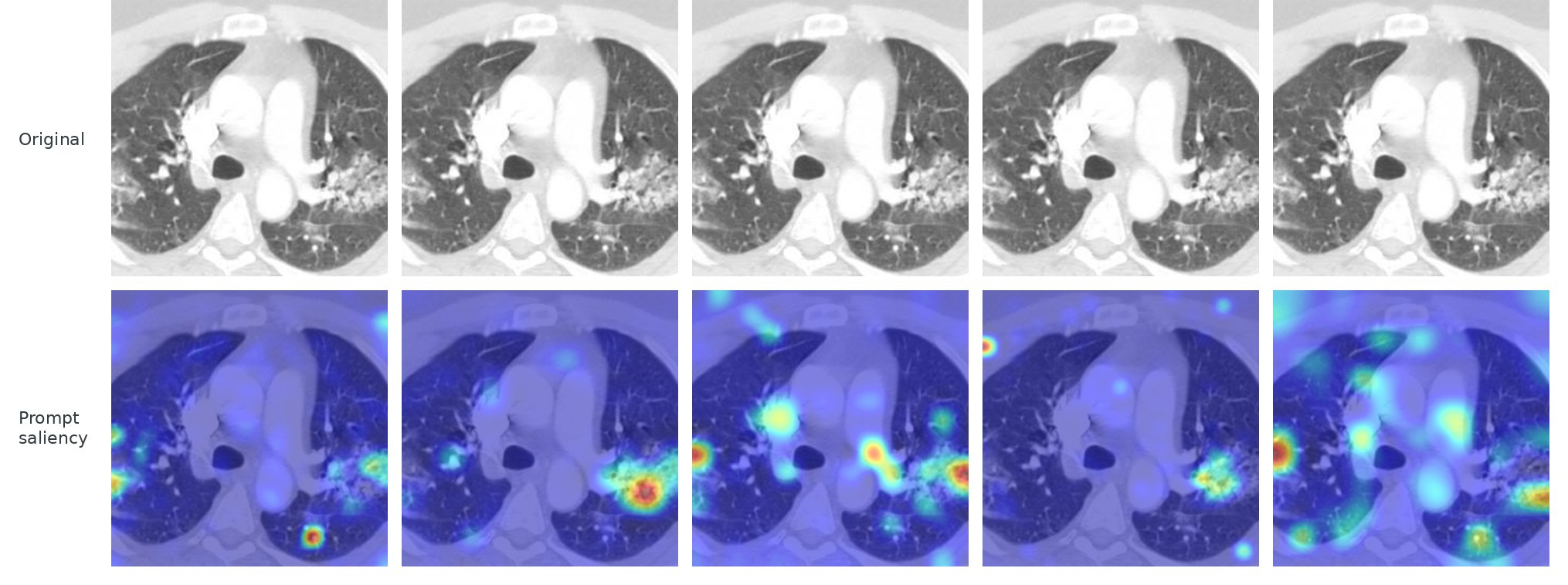}\\[4pt]
{\footnotesize (c)}\\[1pt]
\includegraphics[width=\linewidth]{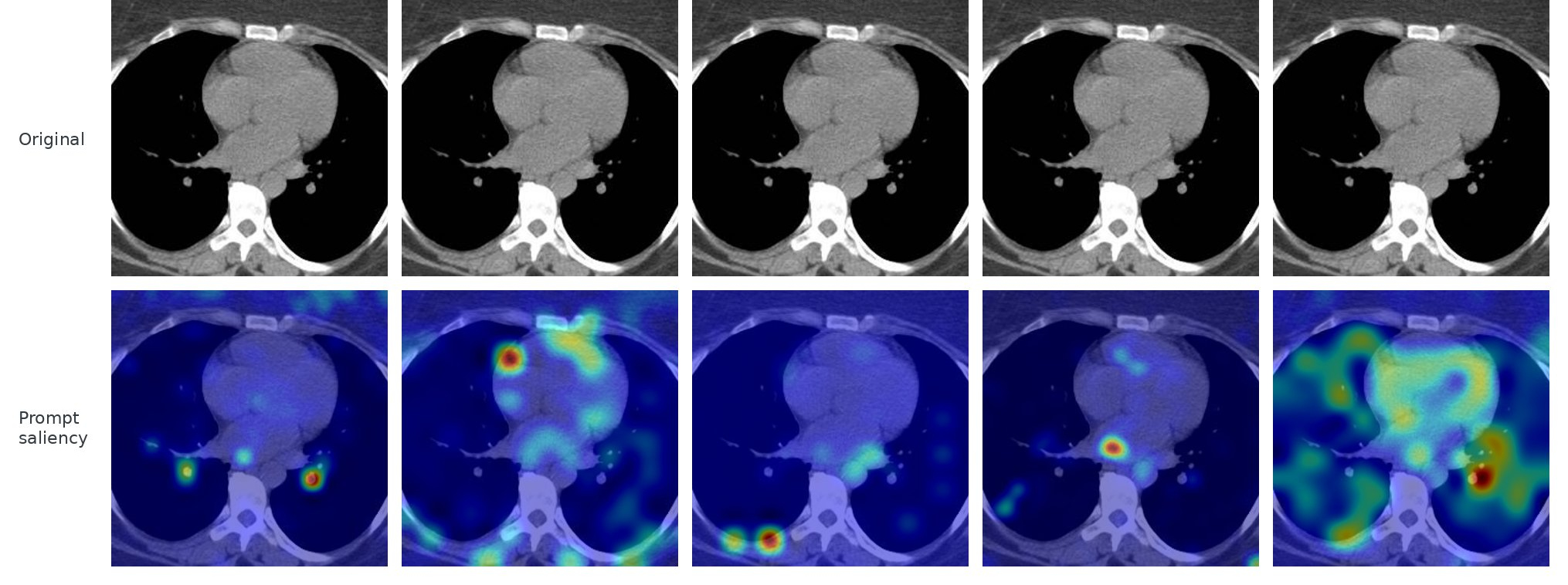}
\caption{\textbf{SARS-CoV-2 CT zero-shot classification.} Examples: (a) COVID-positive, (b) COVID-positive, and (c) non-COVID. In panel (a), the proposed technique produces responses at several locations within the lung fields. BMC and BMCA include peripheral foci, ConceptCLIP highlights a different set of sparse locations, and UniMed-CLIP also responds near central thoracic structures. The overlays show prompt-conditioned responses on a CT slice without marking a validated infection boundary. Heatmaps use ground-truth-class descriptions. Saliency construction follows Section~\ref{app:comparative-saliency-method}.}
\label{fig:saliency-sars-cov-2-ct}
\end{minipage}
\end{figure}
\clearpage

\begin{figure}[p]
\centering\begin{minipage}{5.5in}
\centering
\saliencycolumns
\vspace{3pt}
{\footnotesize (a)}\\[1pt]
\includegraphics[width=\linewidth]{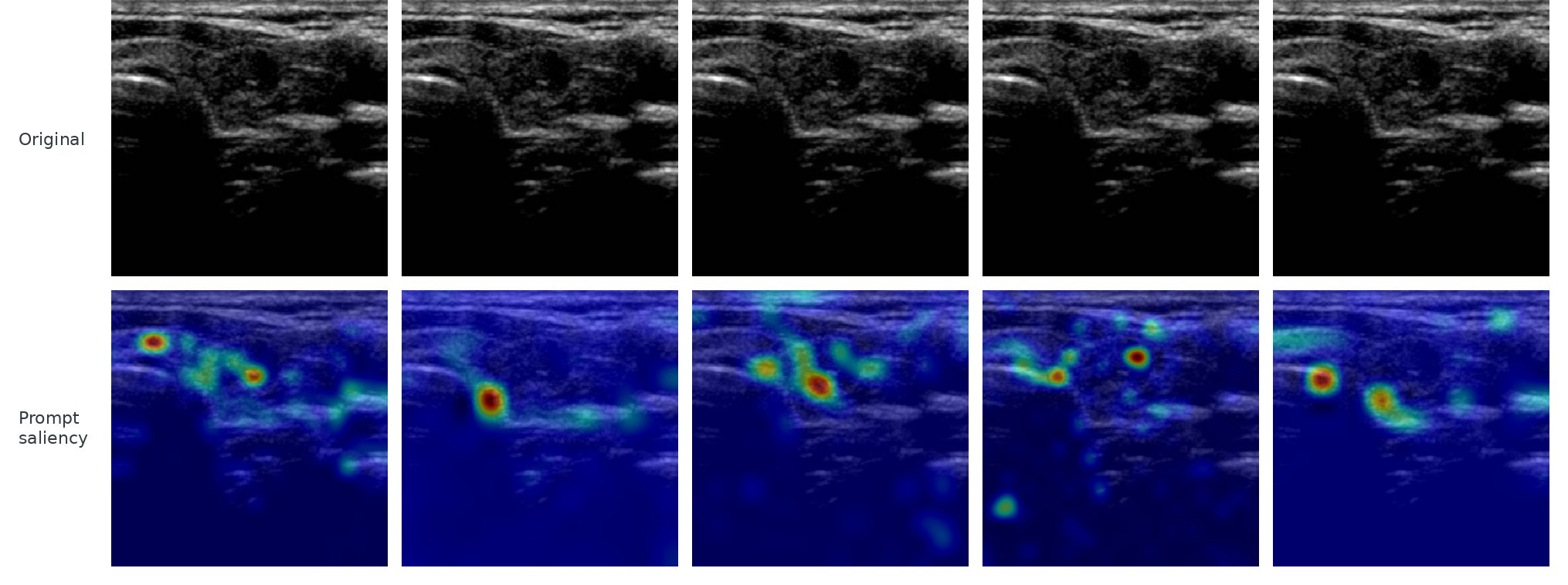}\\[4pt]
{\footnotesize (b)}\\[1pt]
\includegraphics[width=\linewidth]{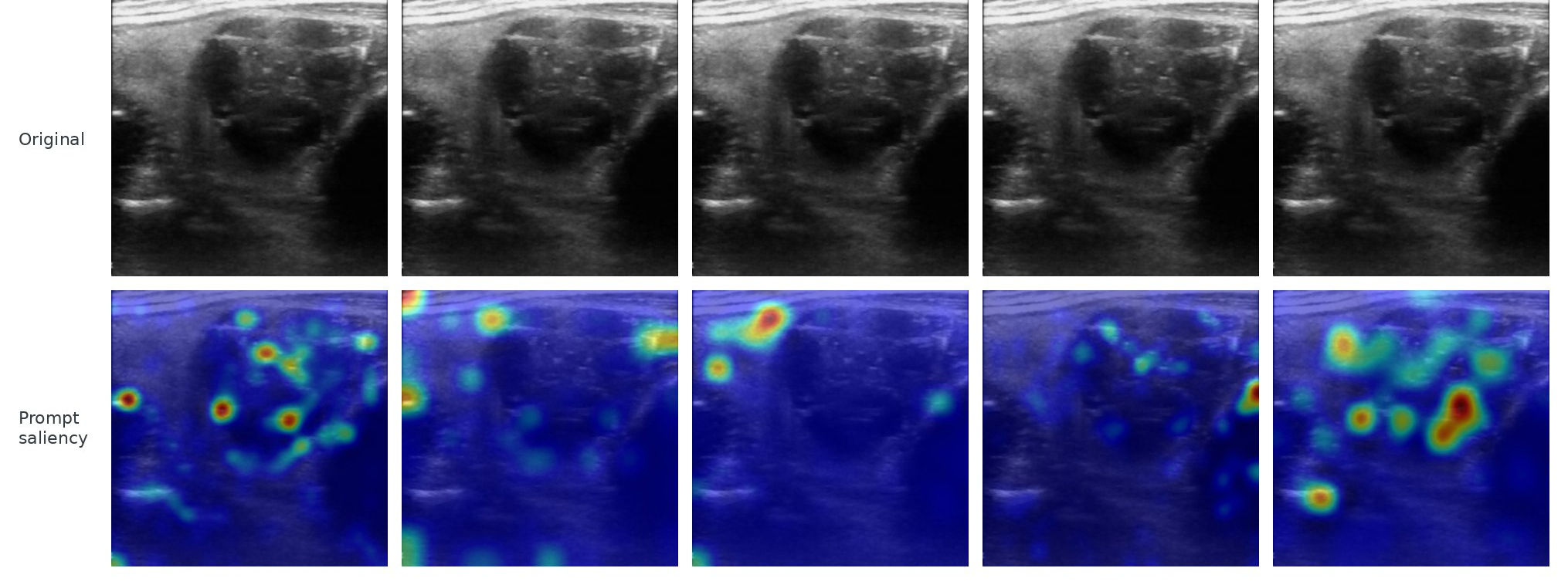}\\[4pt]
{\footnotesize (c)}\\[1pt]
\includegraphics[width=\linewidth]{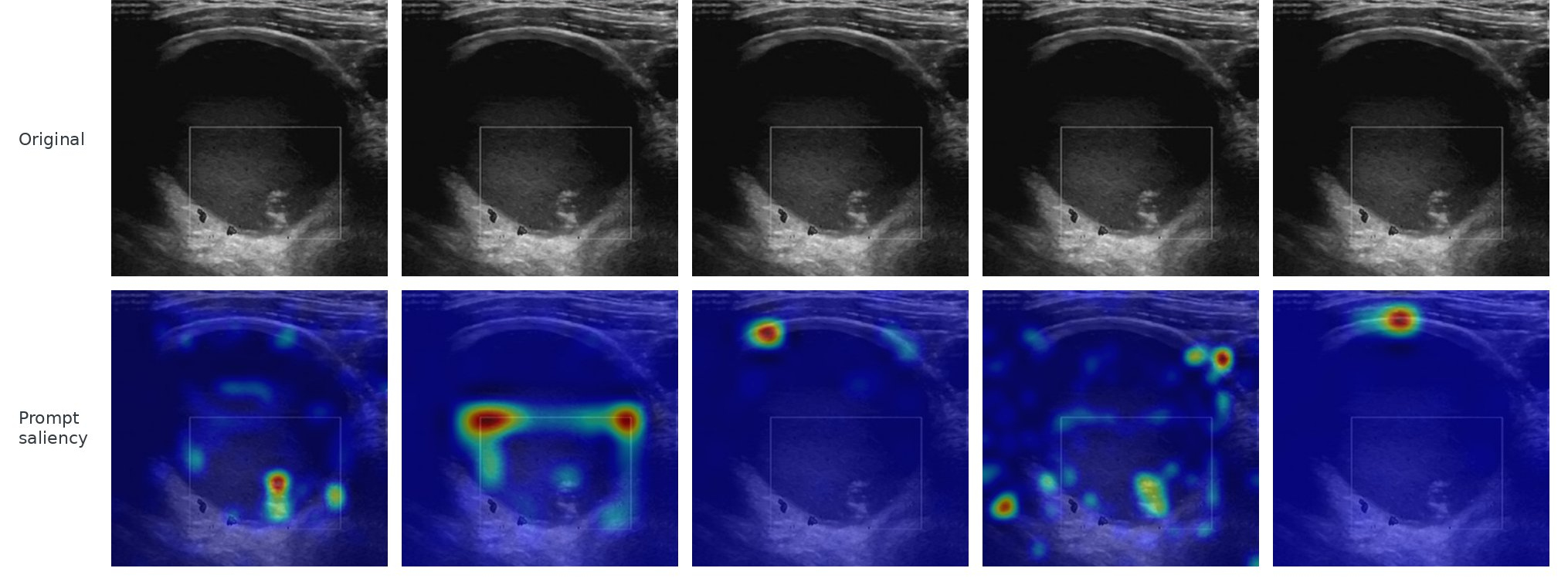}
\caption{\textbf{Thyroid US zero-shot classification.} Examples: (a) malignant thyroid nodule, (b) malignant thyroid nodule, and (c) benign thyroid nodule. In panel (a), the proposed technique, BMC, and BMCA emphasize different parts of the ultrasound tissue band. ConceptCLIP and UniMed-CLIP also produce localized foci, but at different positions. The comparison shows variation in responses to a thyroid-nodule description across the same ultrasound image; it does not establish which response delineates the nodule most accurately. Heatmaps use ground-truth-class descriptions. Saliency construction follows Section~\ref{app:comparative-saliency-method}.}
\label{fig:saliency-thyroid-us}
\end{minipage}
\end{figure}
\clearpage

\begin{figure}[p]
\centering\begin{minipage}{5.5in}
\centering
\saliencycolumns
\vspace{3pt}
{\footnotesize (a)}\\[1pt]
\includegraphics[width=\linewidth]{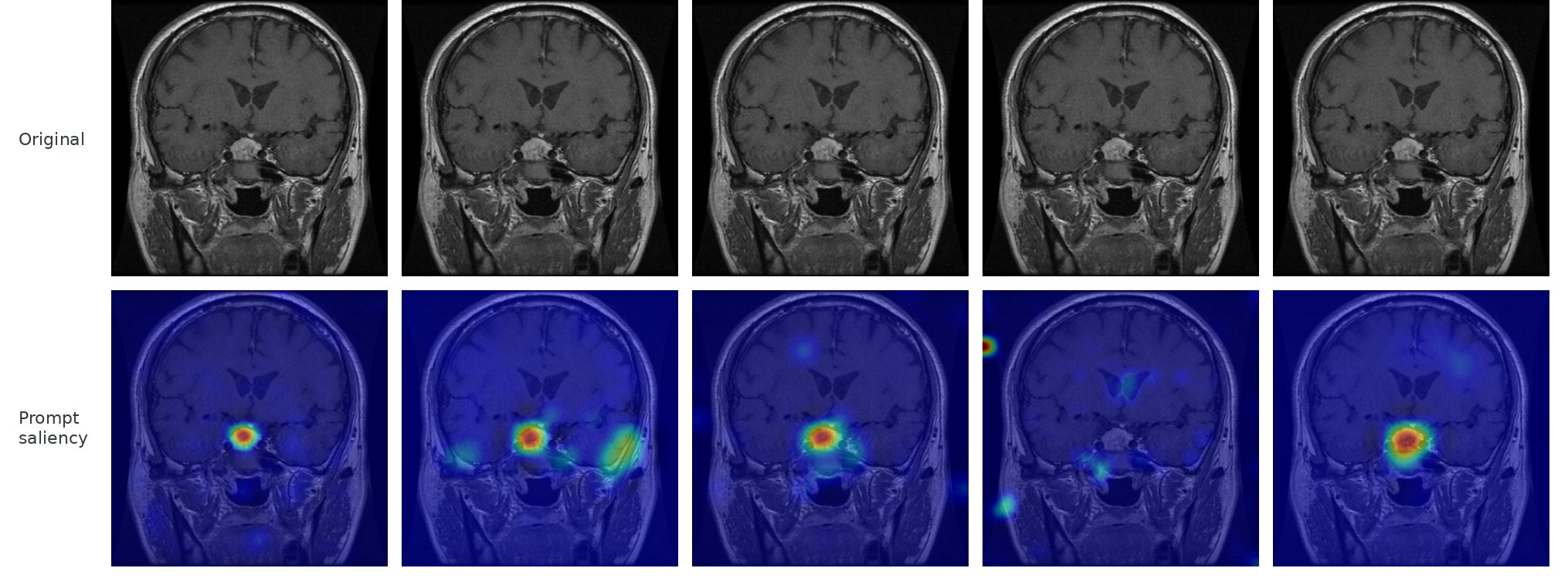}\\[4pt]
{\footnotesize (b)}\\[1pt]
\includegraphics[width=\linewidth]{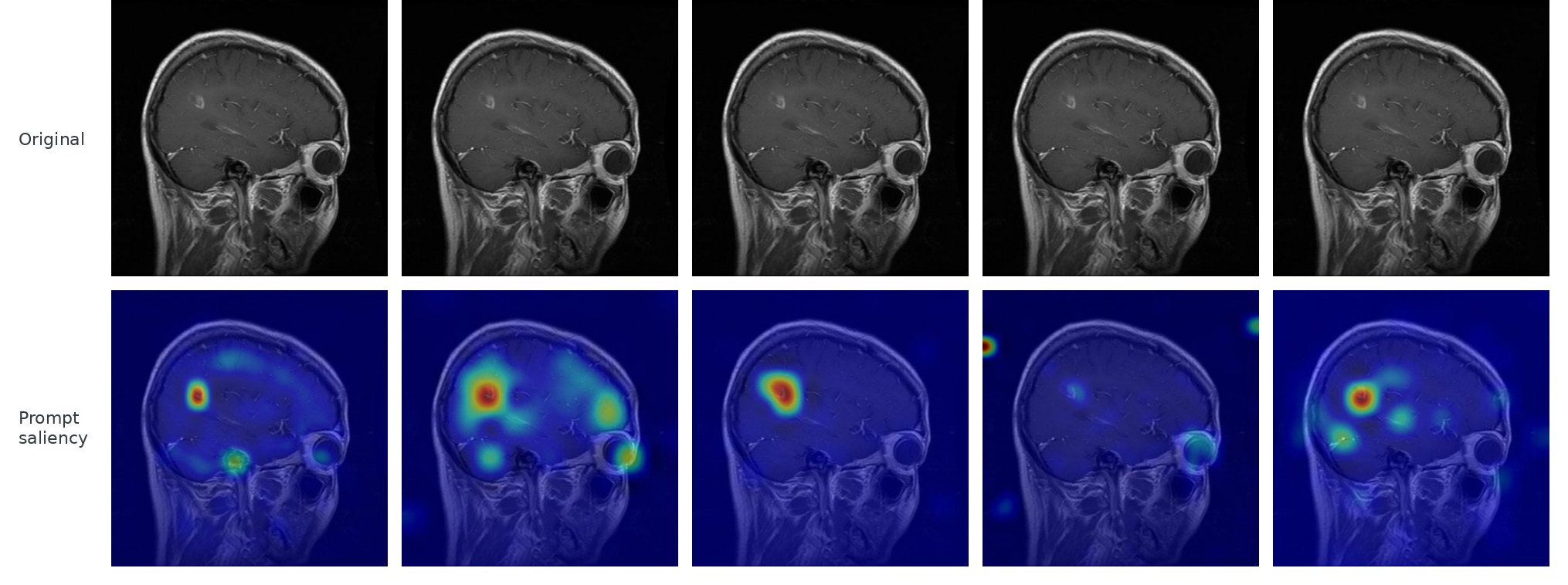}\\[4pt]
{\footnotesize (c)}\\[1pt]
\includegraphics[width=\linewidth]{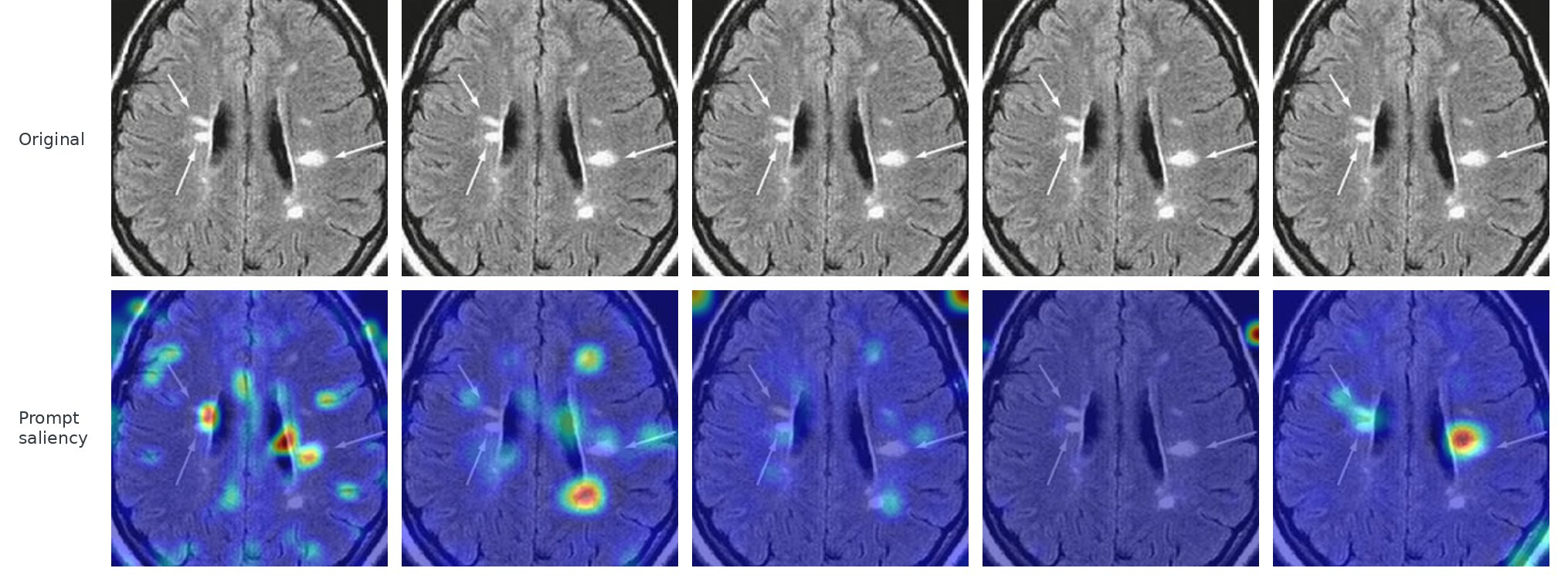}
\caption{\textbf{Brain Tumor MRI linear-probe classification.} Examples: (a) pituitary tumor, (b) glioma, and (c) no tumor. In panel (a), the proposed technique, BMC, BMCA, and UniMed-CLIP show a prominent response around the sellar region. BMC additionally highlights a lateral area, while ConceptCLIP shows a different spatial distribution. These maps explain similarity to the pituitary-related text prompt, independently of the image-only linear classifier. Heatmaps explain image--prompt similarity, not the linear-probe class score. Saliency construction follows Section~\ref{app:comparative-saliency-method}.}
\label{fig:saliency-brain-tumor-mri}
\end{minipage}
\end{figure}
\clearpage

\begin{figure}[p]
\centering\begin{minipage}{5.5in}
\centering
\saliencycolumns
\vspace{3pt}
{\footnotesize (a)}\\[1pt]
\includegraphics[width=\linewidth]{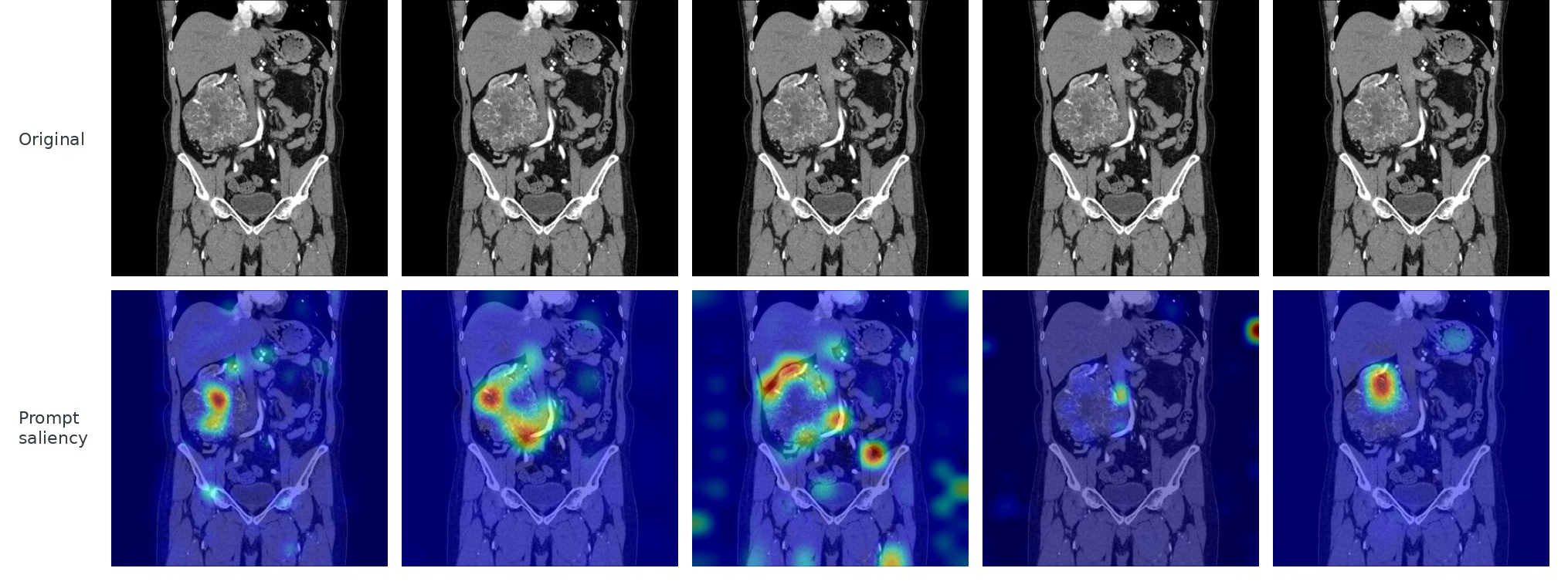}\\[4pt]
{\footnotesize (b)}\\[1pt]
\includegraphics[width=\linewidth]{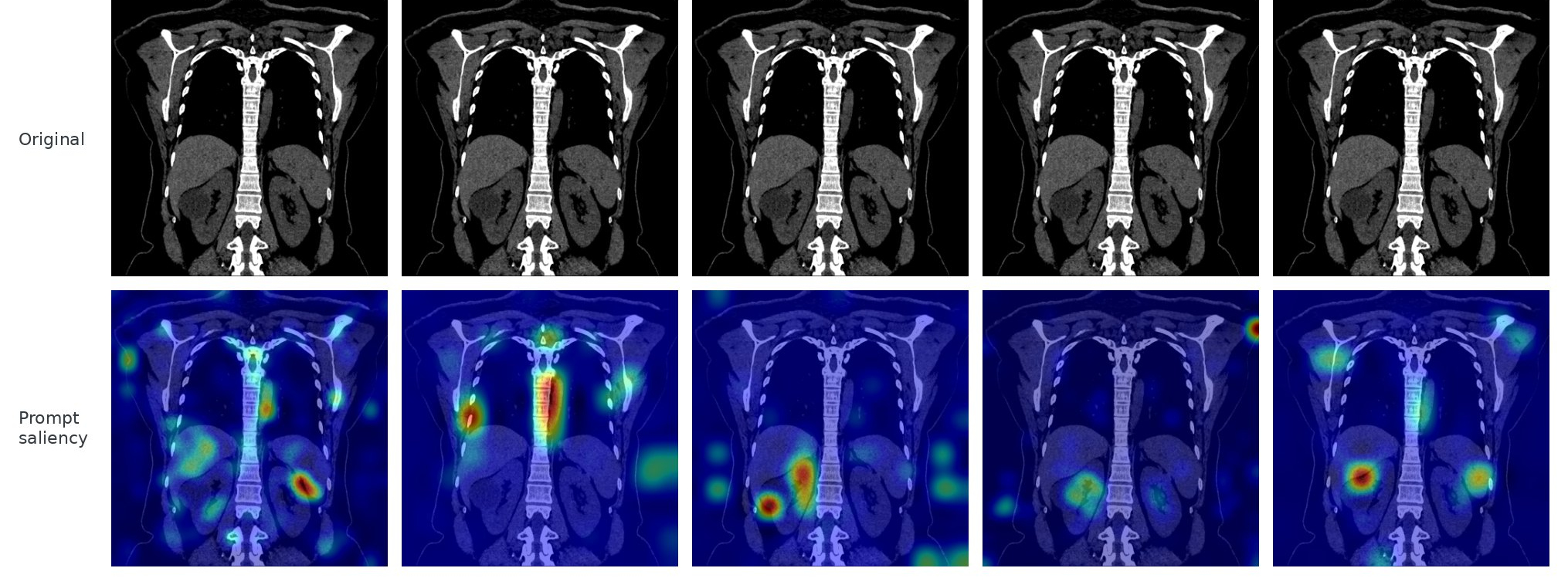}\\[4pt]
{\footnotesize (c)}\\[1pt]
\includegraphics[width=\linewidth]{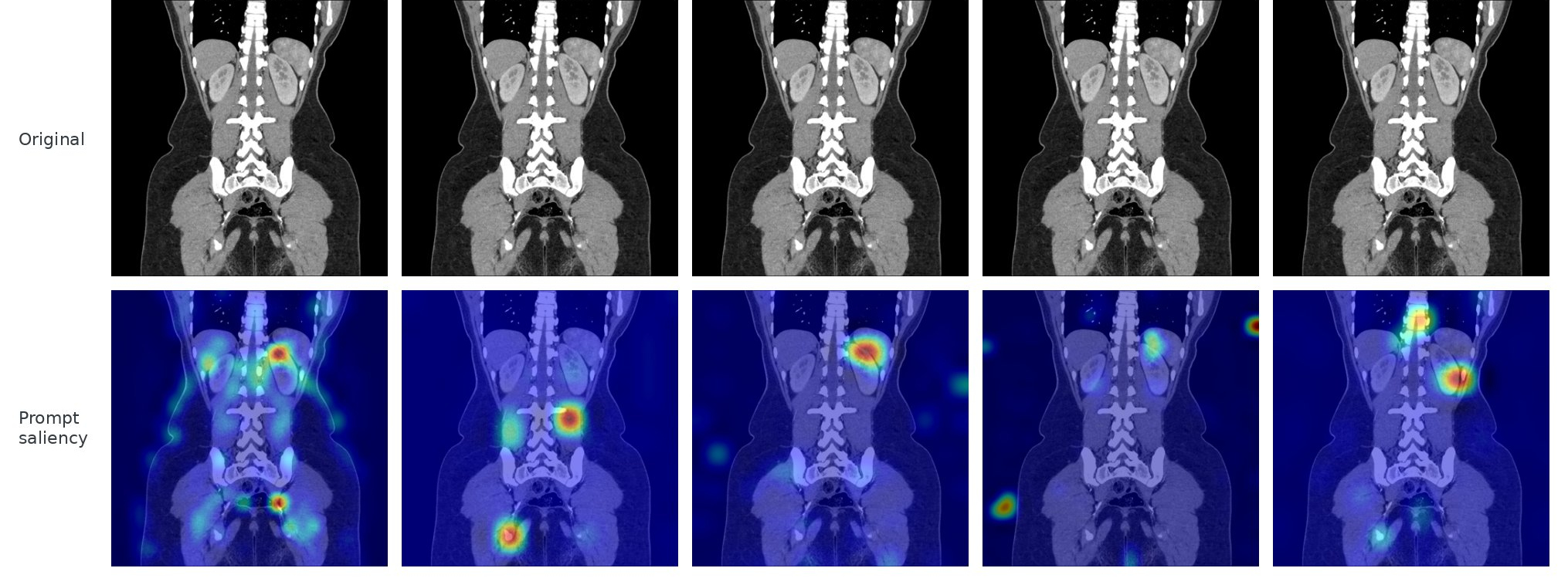}
\caption{\textbf{CT-Kidney linear-probe classification.} Examples: (a) tumor, (b) cyst, and (c) normal. In panel (a), the proposed technique and UniMed-CLIP show a localized response in the central abdominal region on the left side of the displayed image. BMC and BMCA show responses with a larger spatial extent, while ConceptCLIP emphasizes different locations. The maps describe similarity to the tumor-related text prompt rather than contributions to the linear-probe score. Heatmaps explain image--prompt similarity, not the linear-probe class score. Saliency construction follows Section~\ref{app:comparative-saliency-method}.}
\label{fig:saliency-ct-kidney}
\end{minipage}
\end{figure}
\clearpage

\begin{figure}[p]
\centering\begin{minipage}{5.5in}
\centering
\saliencycolumns
\vspace{3pt}
{\footnotesize (a)}\\[1pt]
\includegraphics[width=\linewidth]{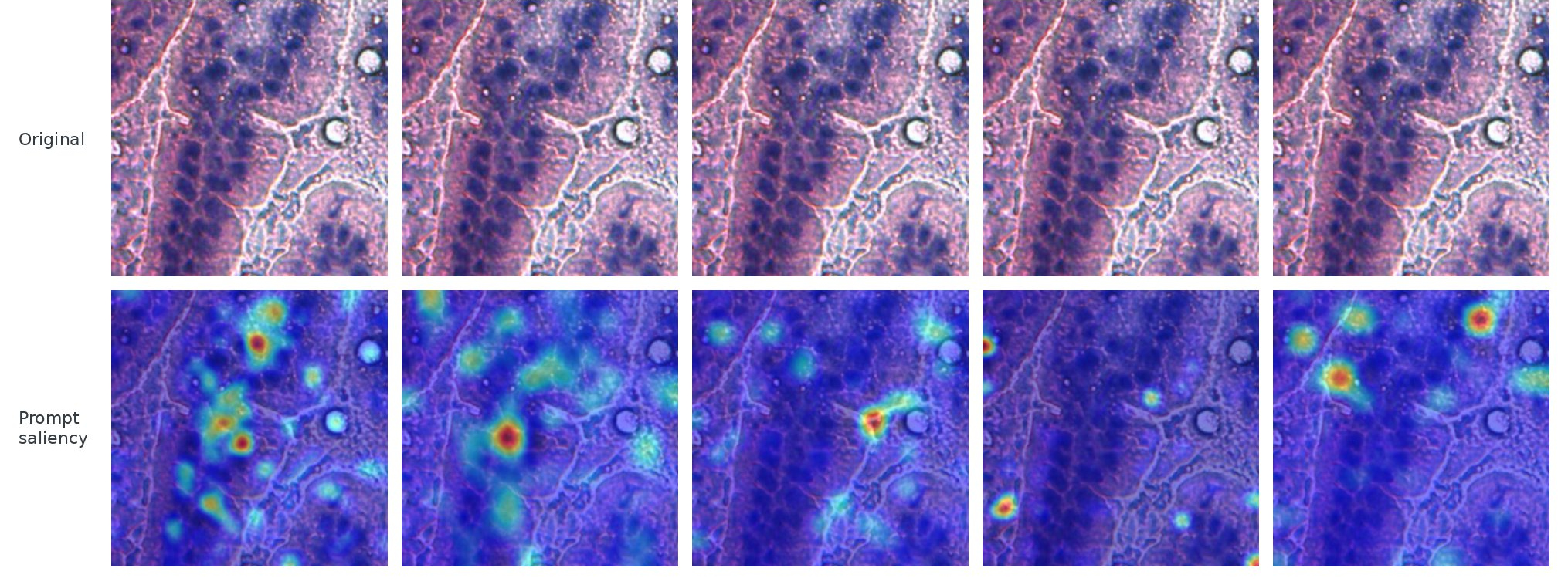}\\[4pt]
{\footnotesize (b)}\\[1pt]
\includegraphics[width=\linewidth]{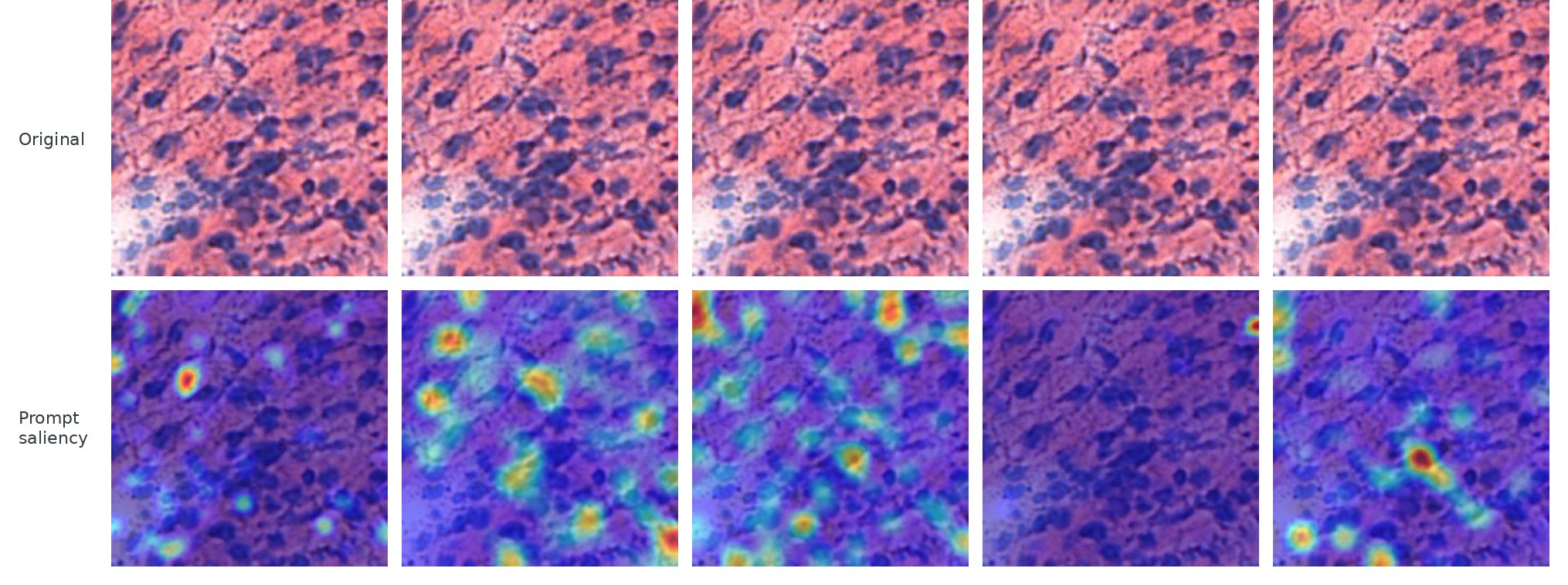}\\[4pt]
{\footnotesize (c)}\\[1pt]
\includegraphics[width=\linewidth]{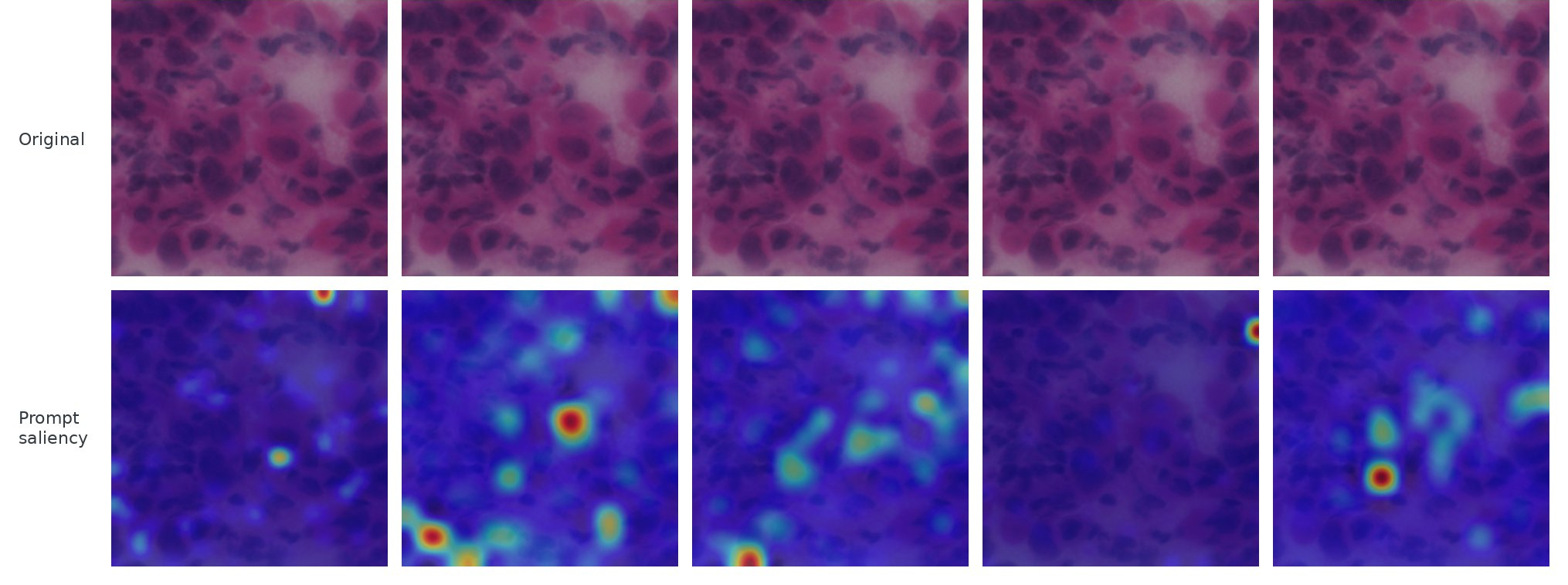}
\caption{\textbf{HEST-320 linear-probe classification.} Examples: (a) tissue class 1, (b) tissue class 4, and (c) tissue class 7. In panel (a), the proposed technique's responses include several locations within and around the darker cellular tissue. BMC, BMCA, and UniMed-CLIP emphasize different portions of the field, while ConceptCLIP shows fewer prominent foci. Because the target is a tissue-class description, these responses concern tissue context and should not be interpreted as cell-level segmentation. Heatmaps explain image--prompt similarity, not the linear-probe class score. Saliency construction follows Section~\ref{app:comparative-saliency-method}.}
\label{fig:saliency-hest-320}
\end{minipage}
\end{figure}
\clearpage

\begin{figure}[p]
\centering\begin{minipage}{5.5in}
\centering
\saliencycolumns
\vspace{3pt}
{\footnotesize (a)}\\[1pt]
\includegraphics[width=\linewidth]{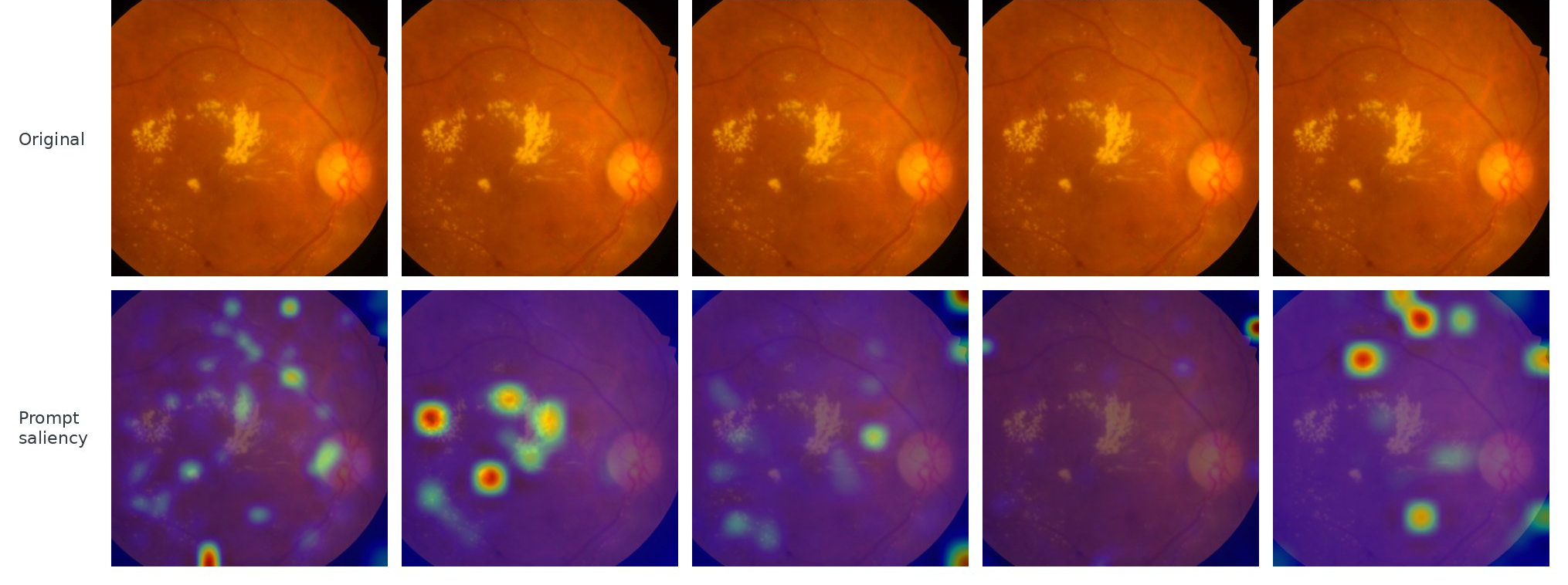}\\[4pt]
{\footnotesize (b)}\\[1pt]
\includegraphics[width=\linewidth]{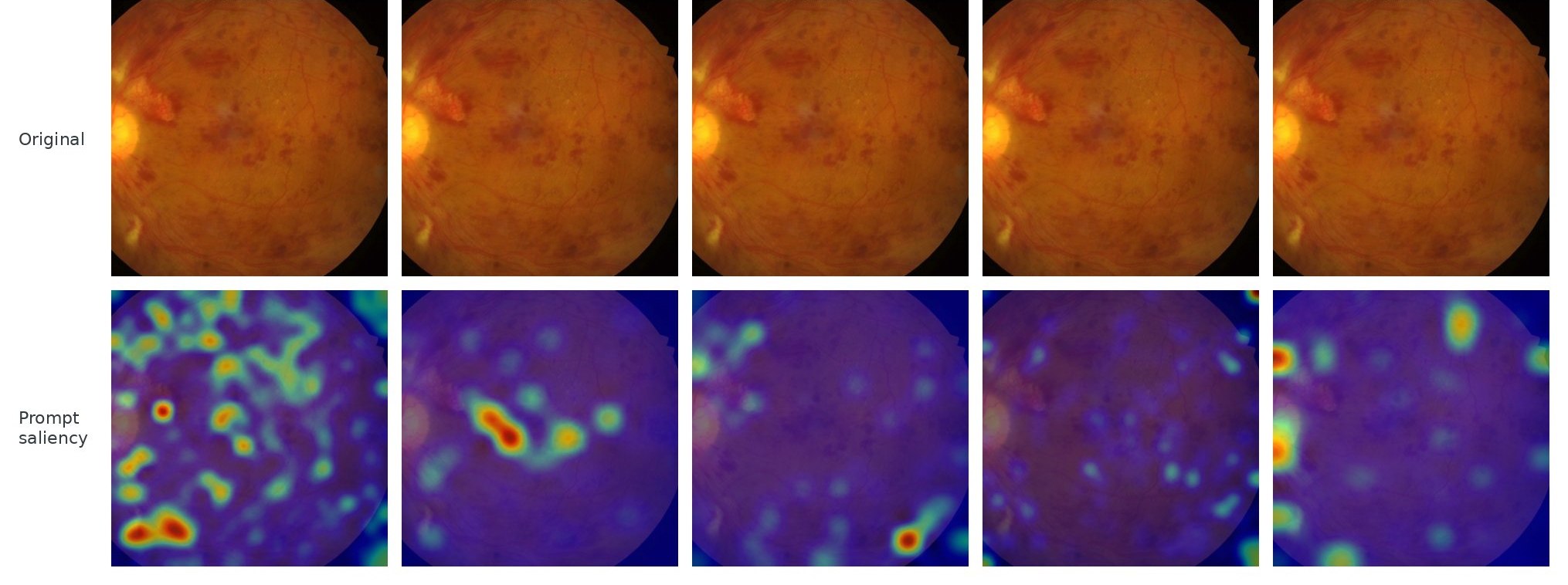}\\[4pt]
{\footnotesize (c)}\\[1pt]
\includegraphics[width=\linewidth]{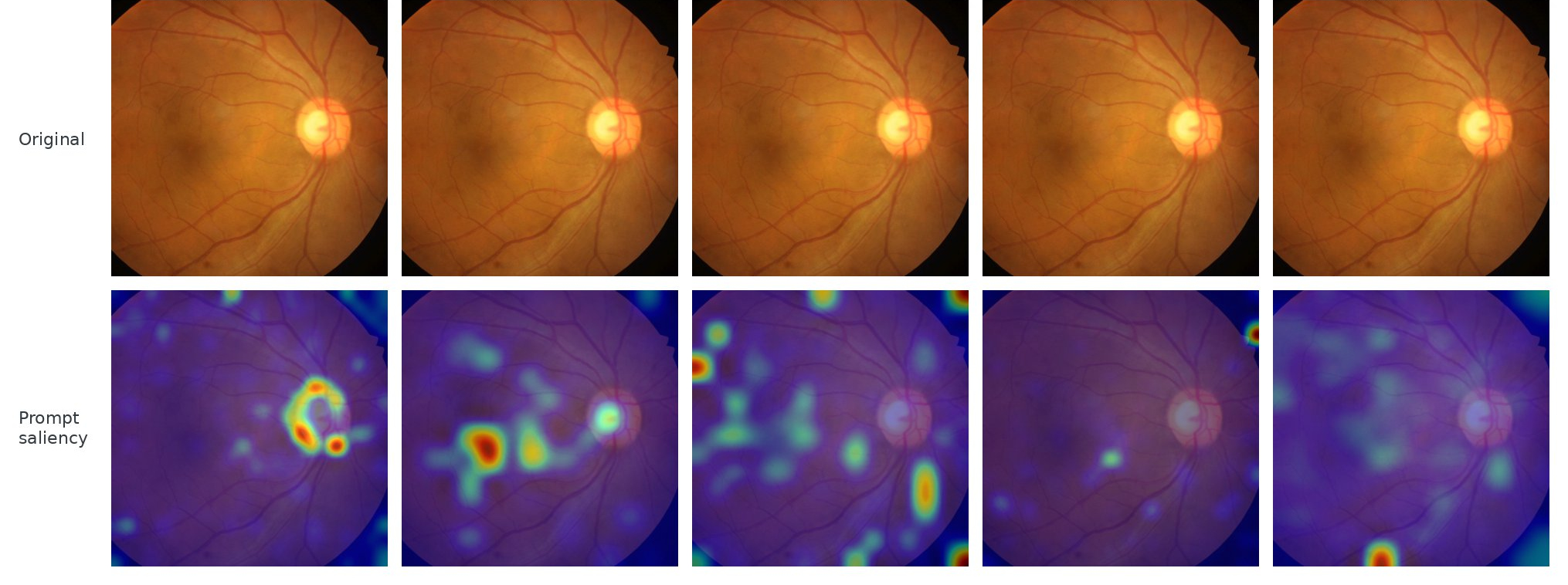}
\caption{\textbf{IDRiD linear-probe classification.} Examples: (a) moderate diabetic retinopathy, (b) proliferative diabetic retinopathy, and (c) no diabetic retinopathy. In panel (a), the proposed technique distributes responses across several retinal locations. BMC includes prominent foci around the central bright retinal markings, while BMCA, ConceptCLIP, and UniMed-CLIP produce different spatial patterns. The maps illustrate text-conditioned responses for a retinopathy example; the accompanying linear probe uses image features without a text prompt. Heatmaps explain image--prompt similarity, not the linear-probe class score. Saliency construction follows Section~\ref{app:comparative-saliency-method}.}
\label{fig:saliency-idrid}
\end{minipage}
\end{figure}
\clearpage

\begin{figure}[p]
\centering\begin{minipage}{5.5in}
\centering
\saliencycolumns
\vspace{3pt}
{\footnotesize (a)}\\[1pt]
\includegraphics[width=\linewidth]{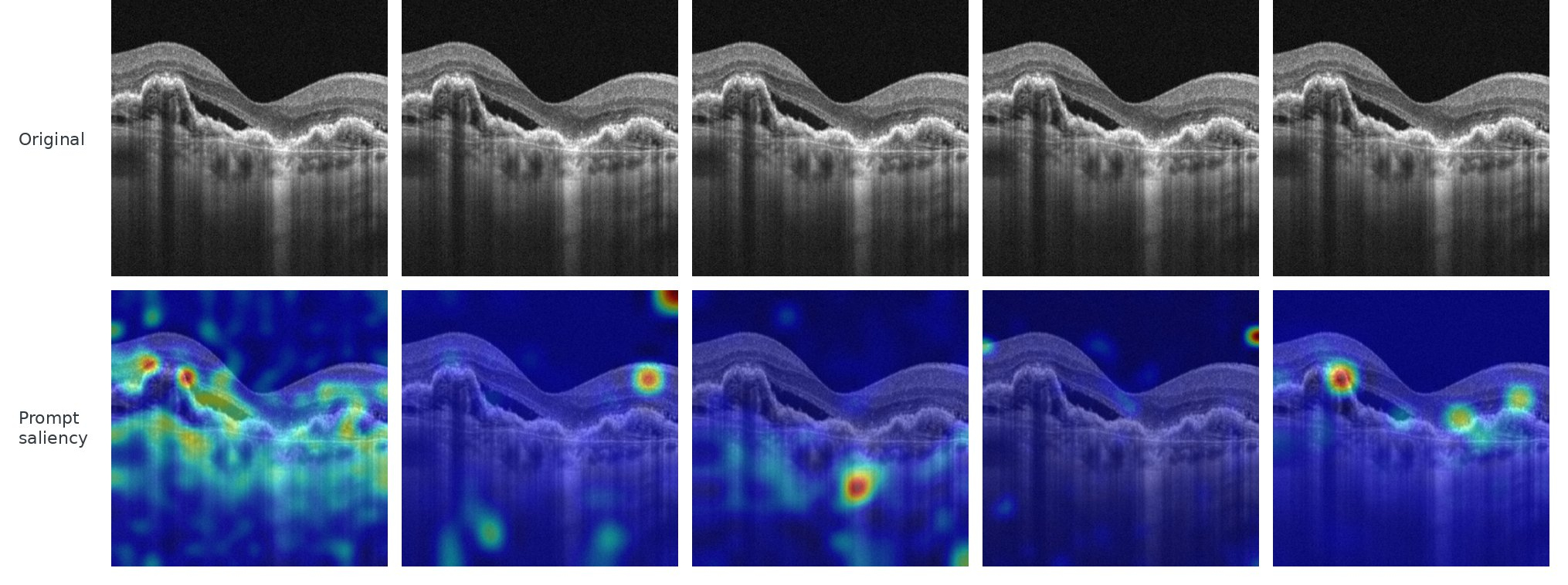}\\[4pt]
{\footnotesize (b)}\\[1pt]
\includegraphics[width=\linewidth]{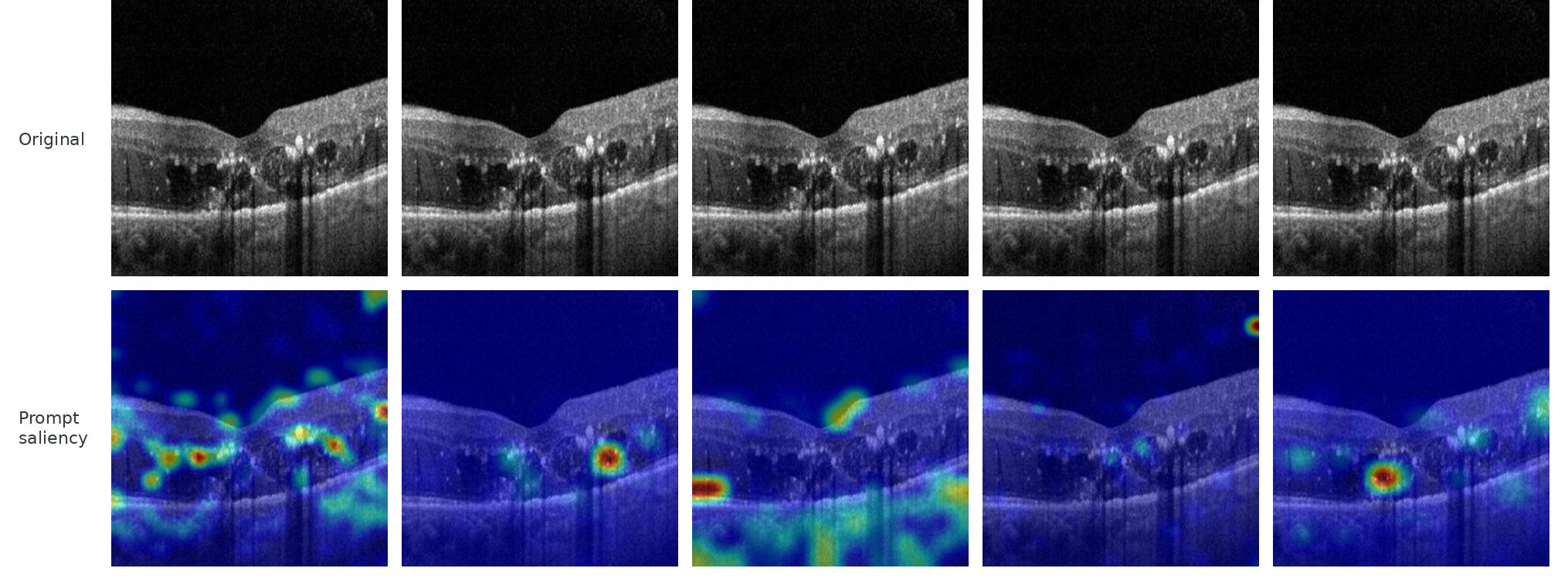}\\[4pt]
{\footnotesize (c)}\\[1pt]
\includegraphics[width=\linewidth]{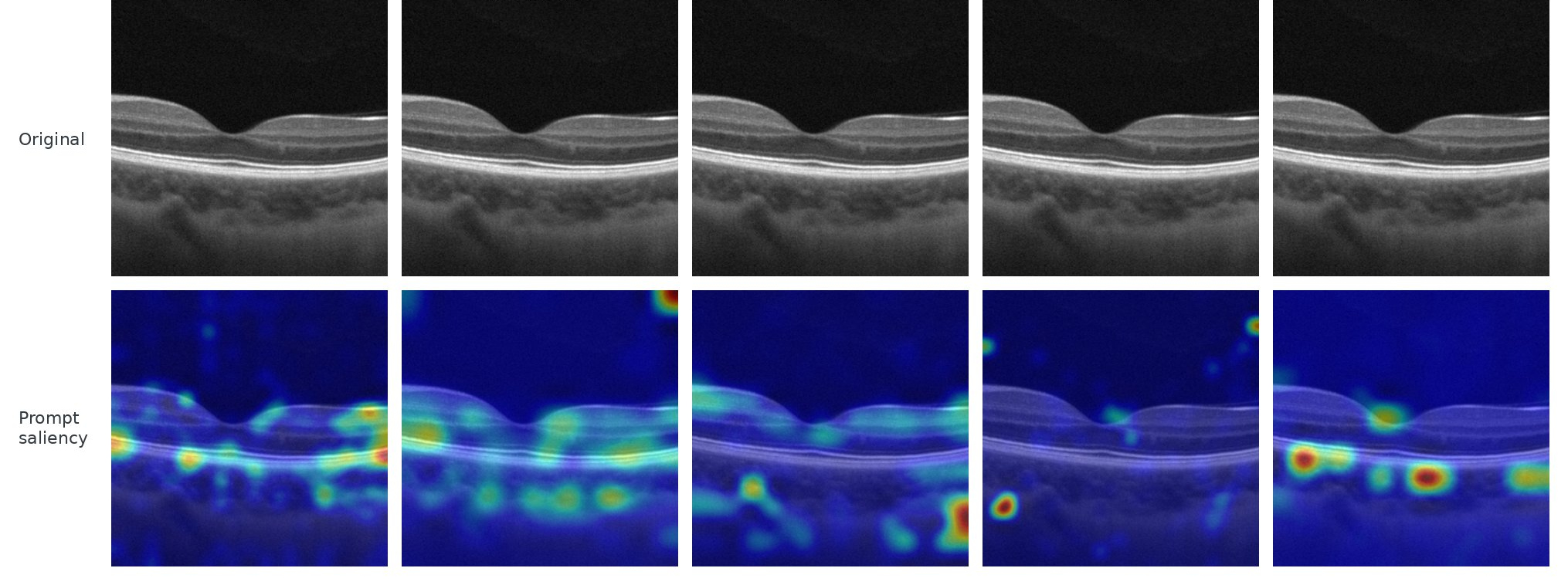}
\caption{\textbf{OCT-DL linear-probe classification.} Examples: (a) age-related macular degeneration, (b) diabetic macular edema, and (c) normal OCT. In panel (a), the proposed technique's response extends along portions of the visible retinal profile and into the region beneath it. BMC, BMCA, ConceptCLIP, and UniMed-CLIP emphasize different locations within the same cross-section. This comparison shows how the disease-related text prompt interacts with retinal structure, without supplying retinal-layer or lesion segmentation. Heatmaps explain image--prompt similarity, not the linear-probe class score. Saliency construction follows Section~\ref{app:comparative-saliency-method}.}
\label{fig:saliency-oct-dl}
\end{minipage}
\end{figure}
\clearpage

\begin{figure}[p]
\centering\begin{minipage}{5.5in}
\centering
\saliencycolumns
\vspace{3pt}
{\footnotesize (a)}\\[1pt]
\includegraphics[width=\linewidth]{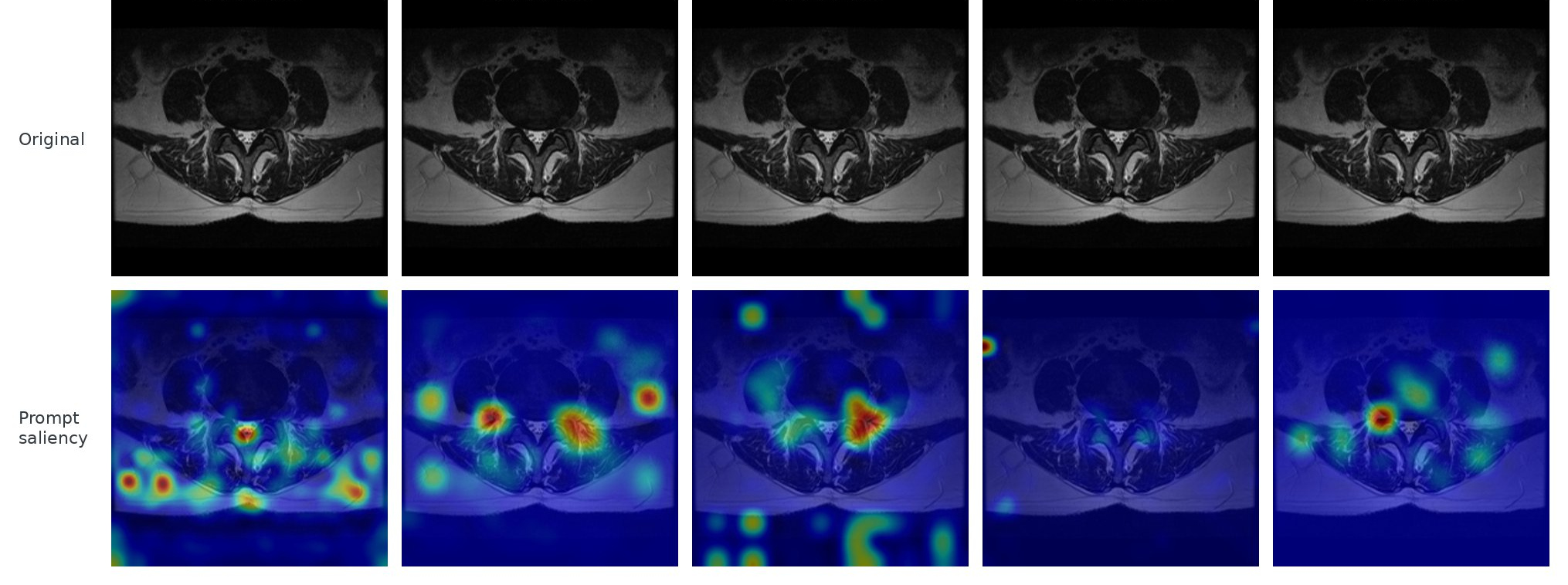}\\[4pt]
{\footnotesize (b)}\\[1pt]
\includegraphics[width=\linewidth]{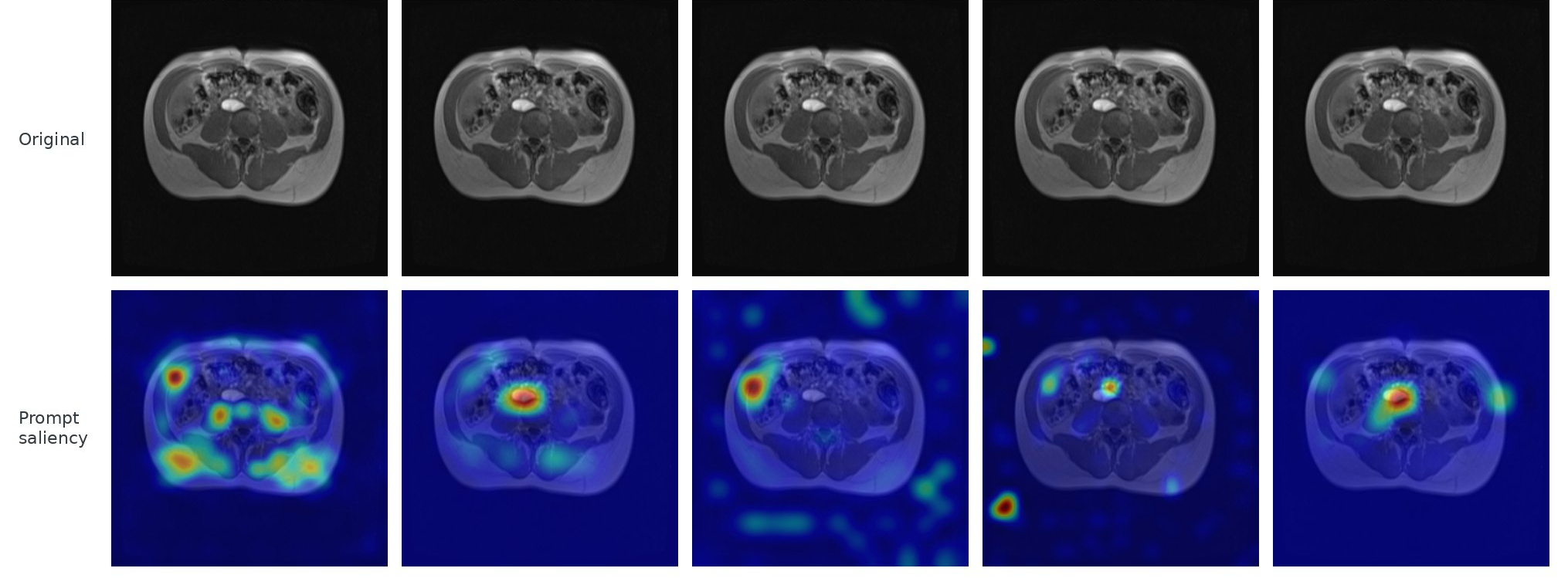}\\[4pt]
{\footnotesize (c)}\\[1pt]
\includegraphics[width=\linewidth]{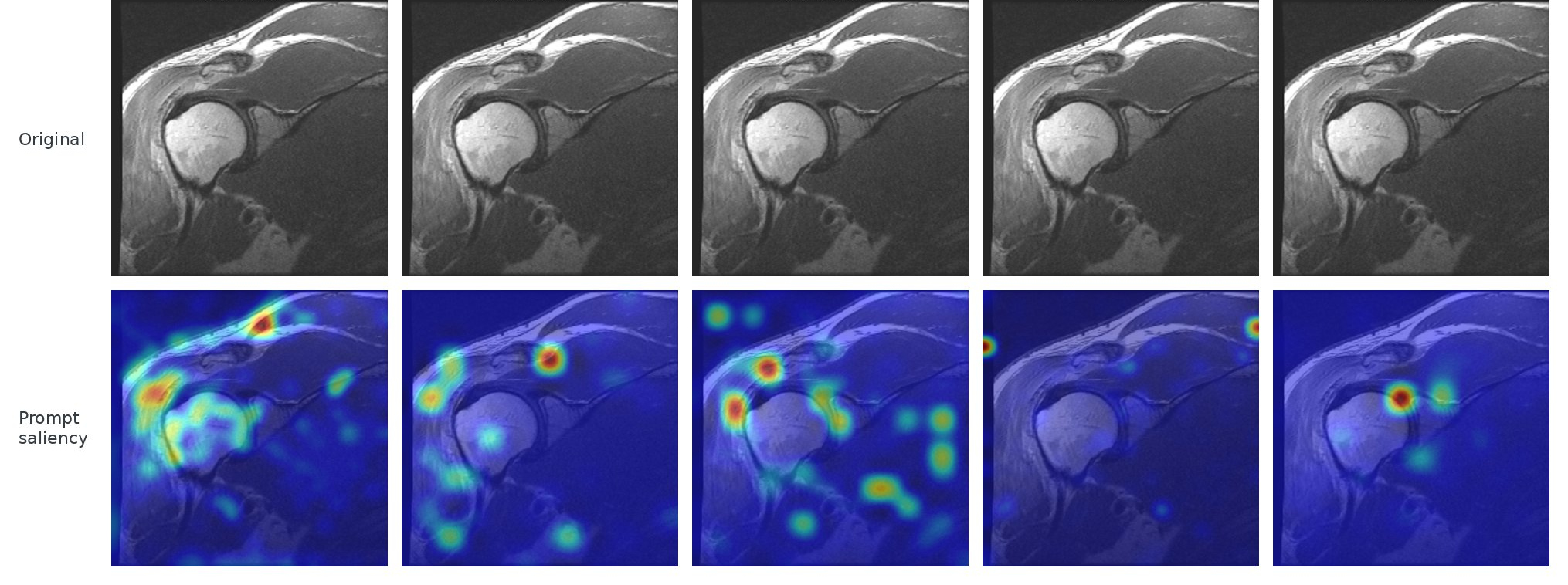}
\caption{\textbf{RadImageNet linear-probe classification.} Examples: (a) spine, (b) abdomen, and (c) shoulder. In panel (a), the proposed technique includes responses around the central spinal region and the lower image periphery. BMC, BMCA, and UniMed-CLIP emphasize different central and lateral locations, while ConceptCLIP shows a more limited pattern. The target describes an anatomical site, so these maps concern anatomy-related similarity rather than disease localization or the linear-probe decision. Heatmaps explain image--prompt similarity, not the linear-probe class score. Saliency construction follows Section~\ref{app:comparative-saliency-method}.}
\label{fig:saliency-radimagenet}
\end{minipage}
\end{figure}
\clearpage

\begin{figure}[p]
\centering\begin{minipage}{5.5in}
\centering
\saliencycolumns
\vspace{3pt}
{\footnotesize (a)}\\[1pt]
\includegraphics[width=\linewidth]{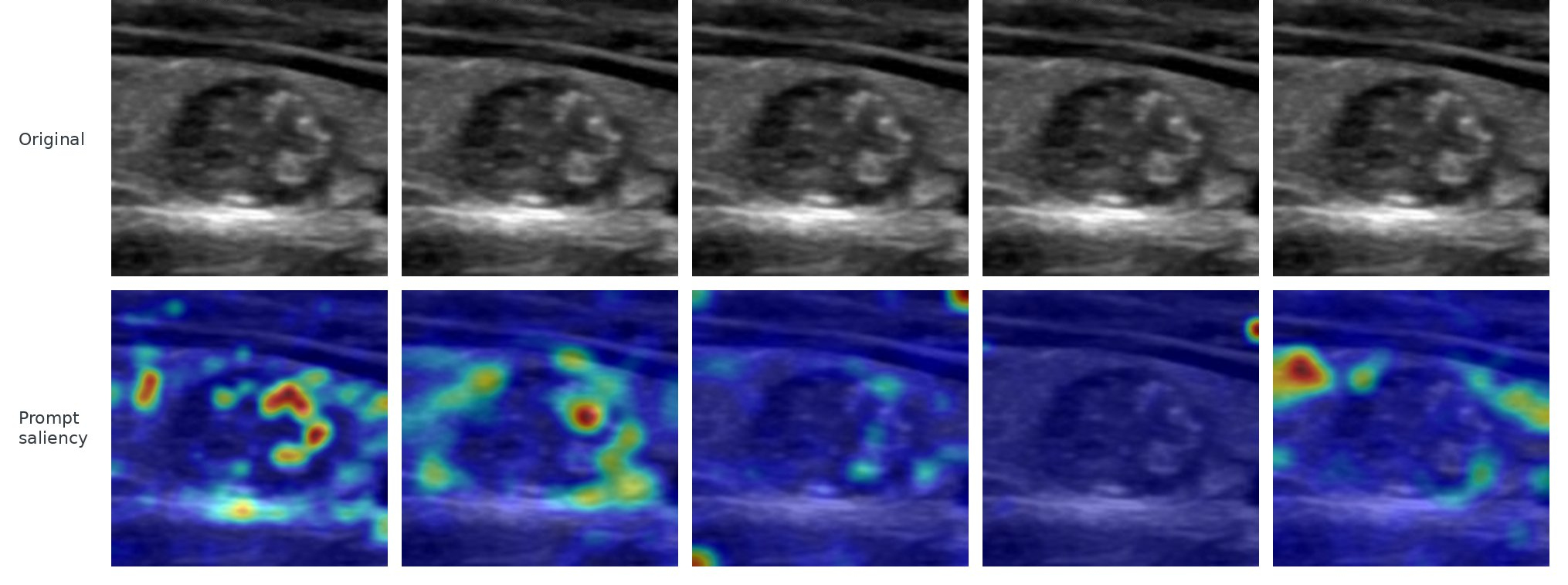}\\[4pt]
{\footnotesize (b)}\\[1pt]
\includegraphics[width=\linewidth]{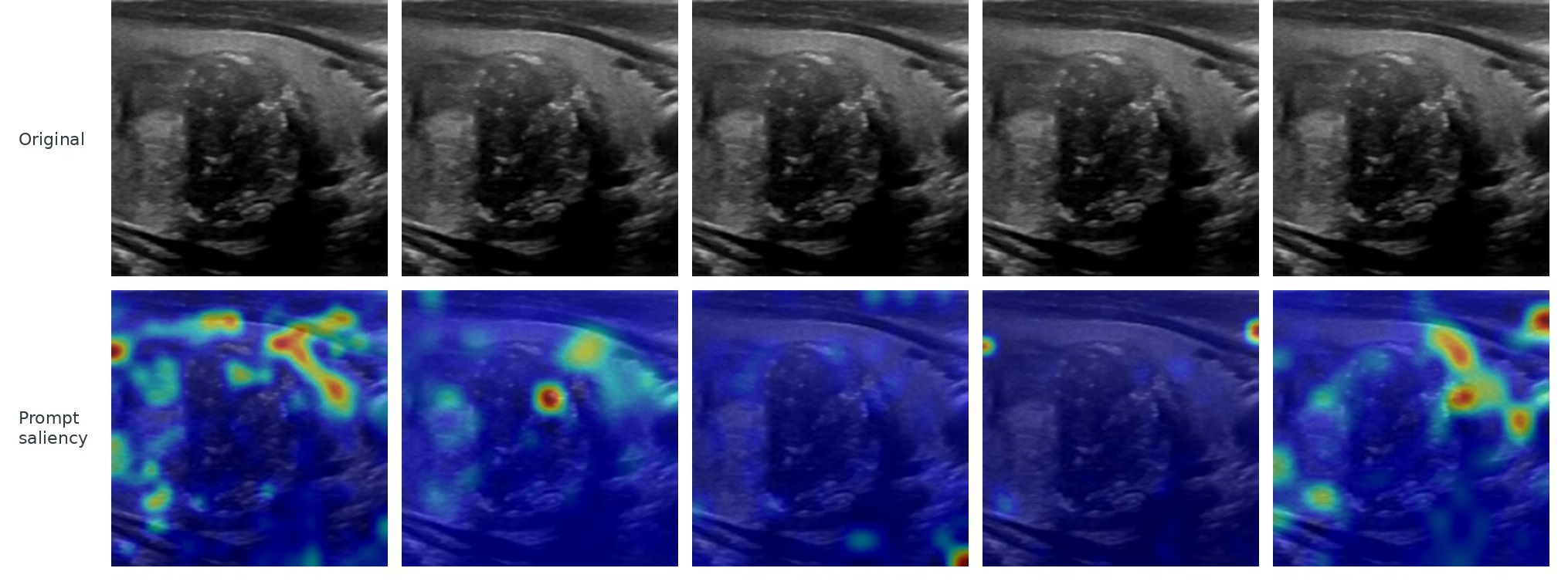}\\[4pt]
{\footnotesize (c)}\\[1pt]
\includegraphics[width=\linewidth]{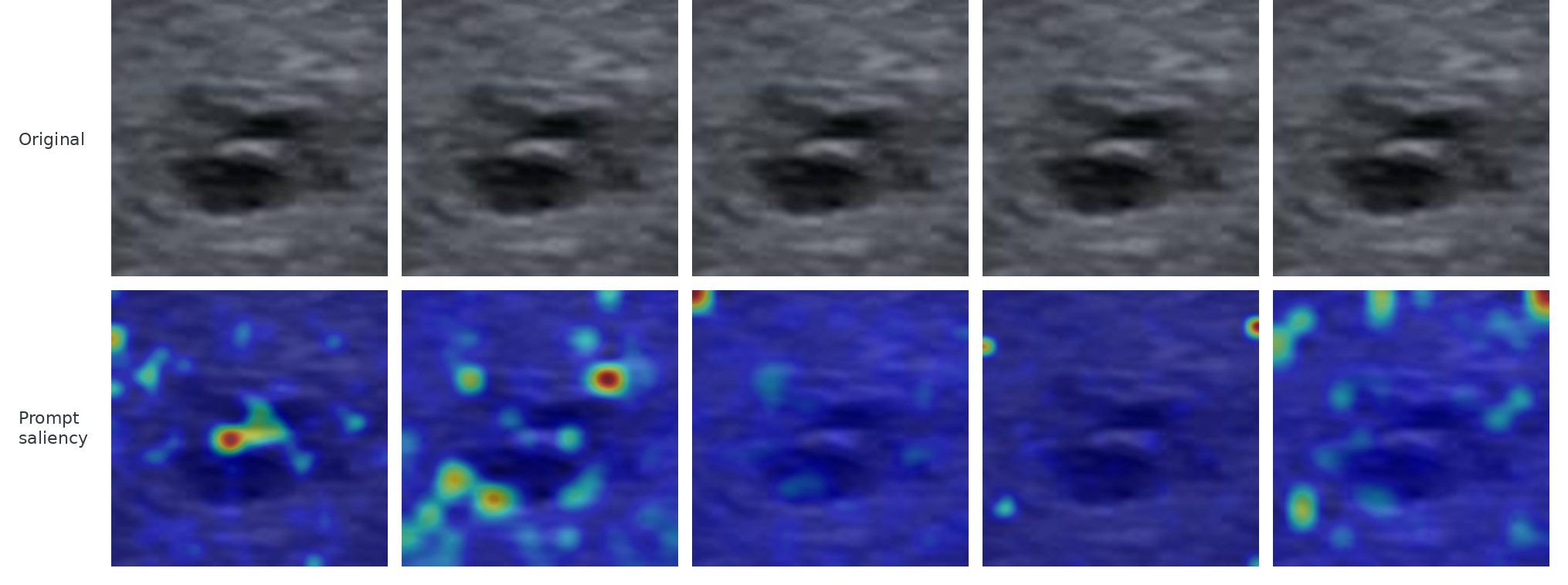}
\caption{\textbf{TN5000 linear-probe classification.} Examples: (a) malignant thyroid nodule, (b) malignant thyroid nodule, and (c) benign thyroid nodule. In panel (a), the proposed technique shows several responses within and around the central nodule appearance. BMC has a prominent central focus, while BMCA, ConceptCLIP, and UniMed-CLIP highlight different central or peripheral locations. The comparison explains similarity to the nodule-related prompt and remains separate from the image-only classifier's score. Heatmaps explain image--prompt similarity, not the linear-probe class score. Saliency construction follows Section~\ref{app:comparative-saliency-method}.}
\label{fig:saliency-tn5000}
\end{minipage}
\end{figure}
\clearpage

\begin{figure}[p]
\centering\begin{minipage}{5.5in}
\centering
\saliencycolumns
\vspace{3pt}
{\footnotesize (a)}\\[1pt]
\includegraphics[width=\linewidth]{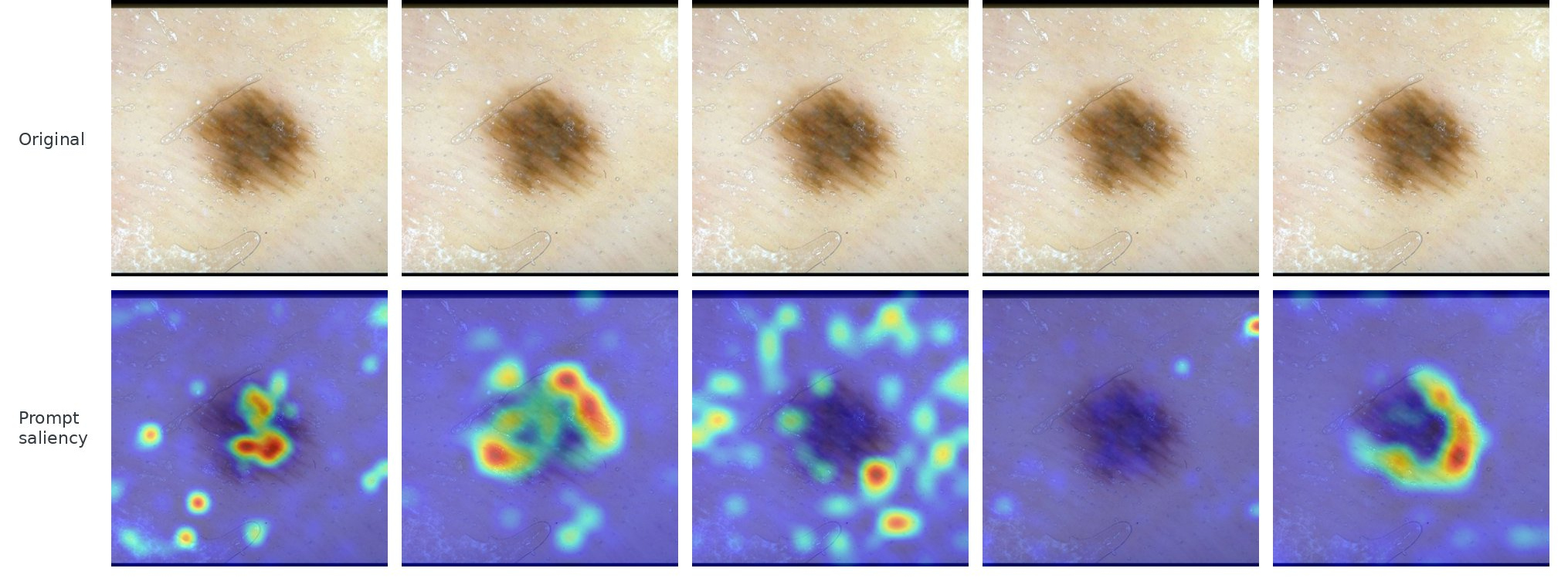}\\[4pt]
{\footnotesize (b)}\\[1pt]
\includegraphics[width=\linewidth]{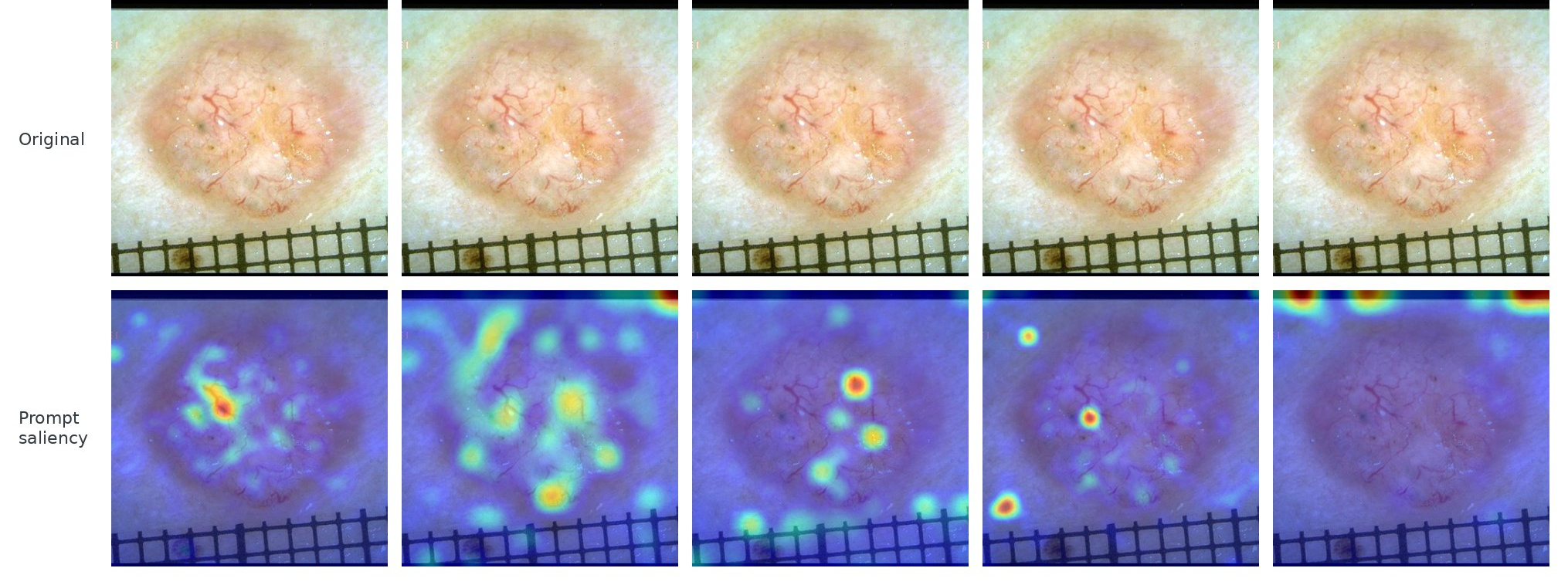}\\[4pt]
{\footnotesize (c)}\\[1pt]
\includegraphics[width=\linewidth]{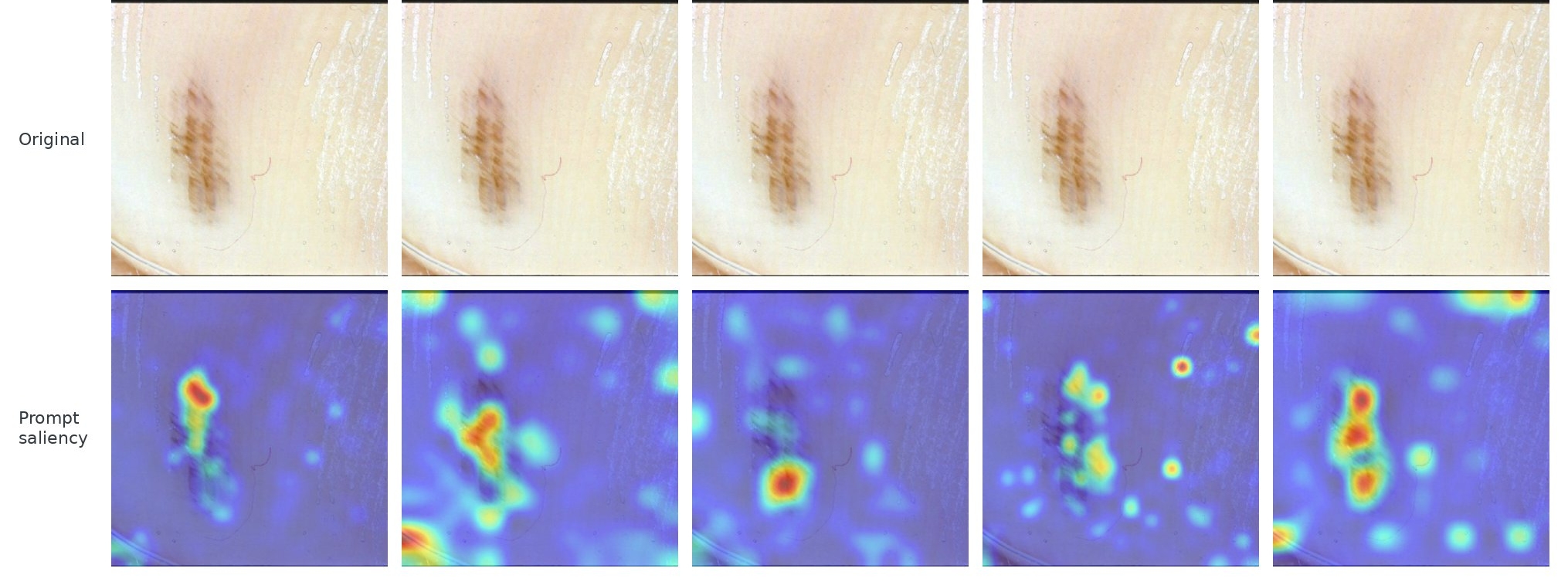}
\caption{\textbf{Derm7pt image-to-text retrieval.} Examples: (a) Clark nevus, (b) basal cell carcinoma, and (c) Clark nevus. In panel (a), the proposed technique, BMC, and UniMed-CLIP include responses within the pigmented lesion, with different spatial extents. BMCA produces several responses outside the lesion, while ConceptCLIP shows limited lesion-associated emphasis. These are responses to the paired text; their visual plausibility alone does not determine whether that text is ranked correctly in the retrieval gallery. Heatmaps use the paired caption or report. Saliency construction follows Section~\ref{app:comparative-saliency-method}.}
\label{fig:saliency-derm7pt}
\end{minipage}
\end{figure}
\clearpage

\begin{figure}[p]
\centering\begin{minipage}{5.5in}
\centering
\saliencycolumns
\vspace{3pt}
{\footnotesize (a)}\\[1pt]
\includegraphics[width=\linewidth]{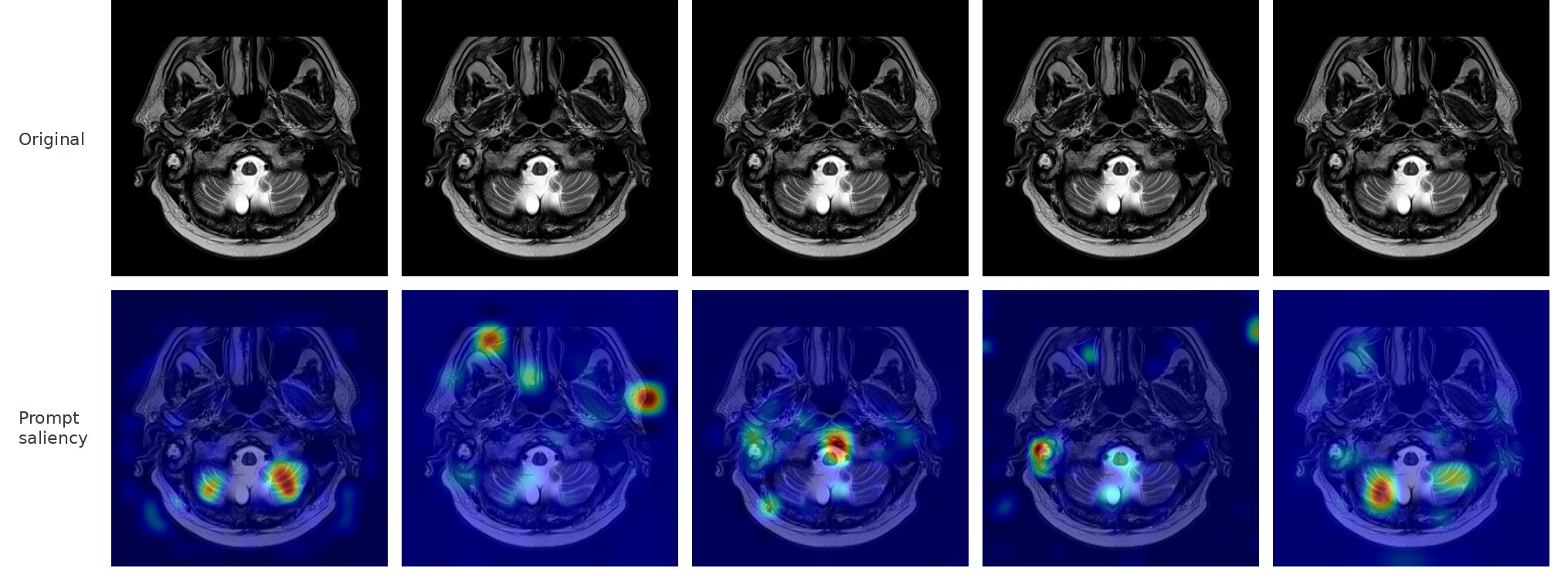}\\[4pt]
{\footnotesize (b)}\\[1pt]
\includegraphics[width=\linewidth]{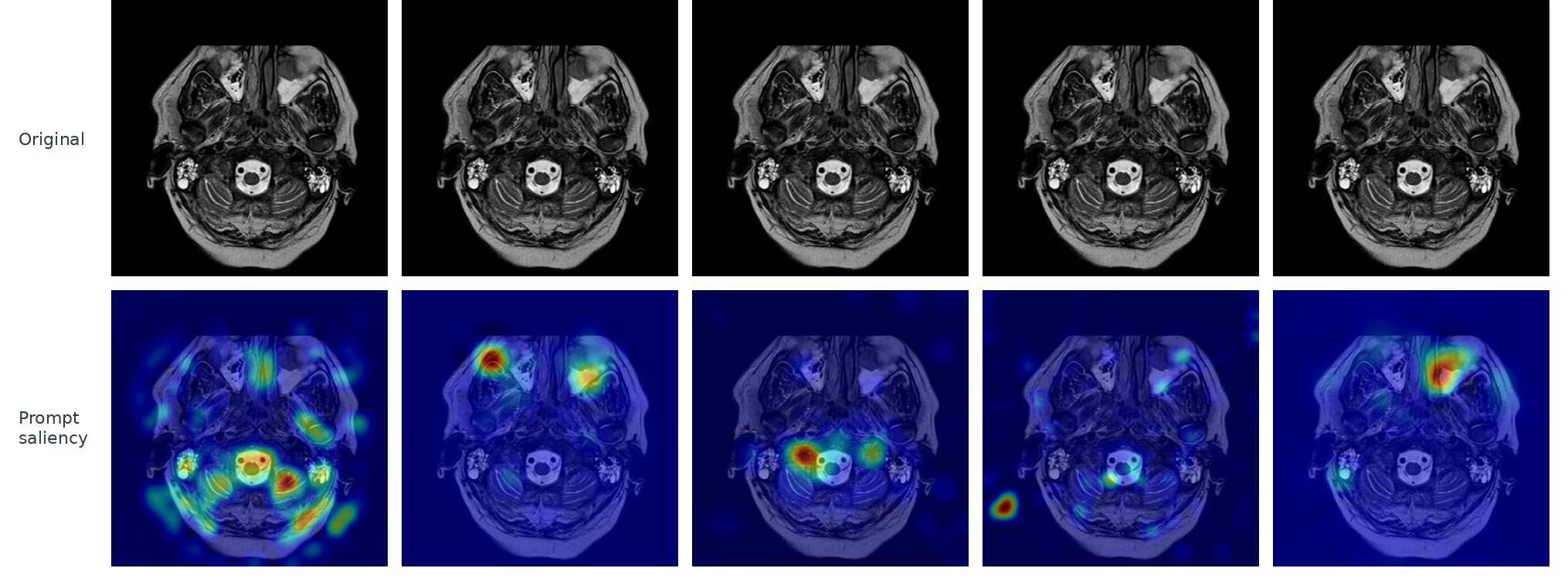}\\[4pt]
{\footnotesize (c)}\\[1pt]
\includegraphics[width=\linewidth]{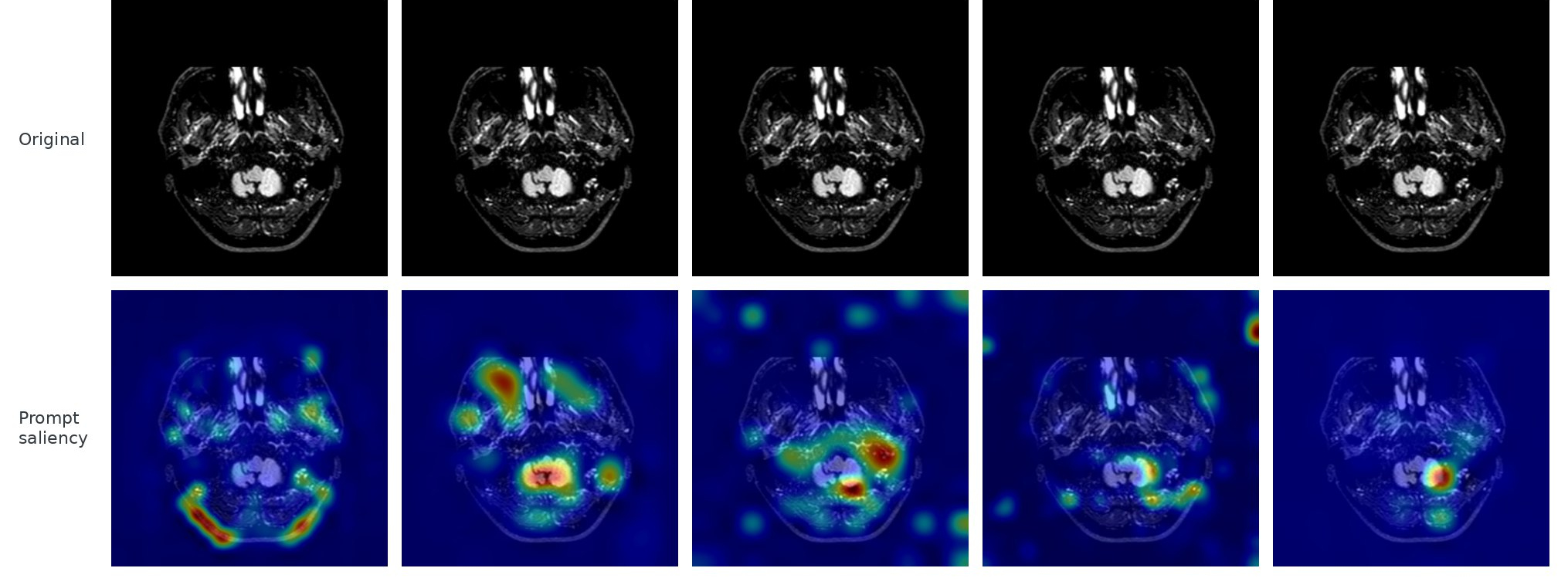}
\caption{\textbf{MR-RATE image-to-text retrieval.} Examples: (a) ischemic changes and atrophy, (b) bone/dural metastases, and (c) cranial MRI within normal limits. In panel (a), the proposed technique includes responses within the posterior portion of the displayed cranial image. BMC emphasizes more anterior and peripheral locations, while BMCA, ConceptCLIP, and UniMed-CLIP highlight different posterior or lateral areas. The target is the paired report, so the map should be read as a response to its full description rather than to one diagnostic term. Heatmaps use the paired caption or report. Saliency construction follows Section~\ref{app:comparative-saliency-method}.}
\label{fig:saliency-mr-rate}
\end{minipage}
\end{figure}
\clearpage

\begin{figure}[p]
\centering\begin{minipage}{5.5in}
\centering
\saliencycolumns
\vspace{3pt}
{\footnotesize (a)}\\[1pt]
\includegraphics[width=\linewidth]{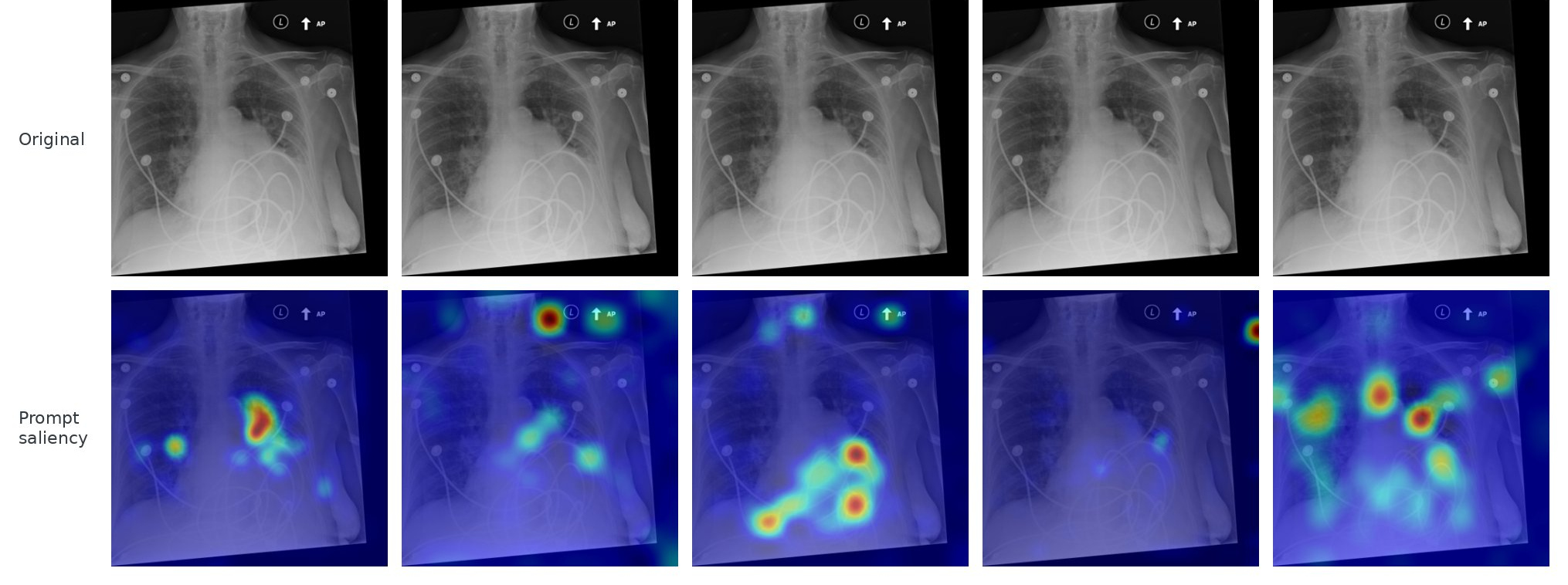}\\[4pt]
{\footnotesize (b)}\\[1pt]
\includegraphics[width=\linewidth]{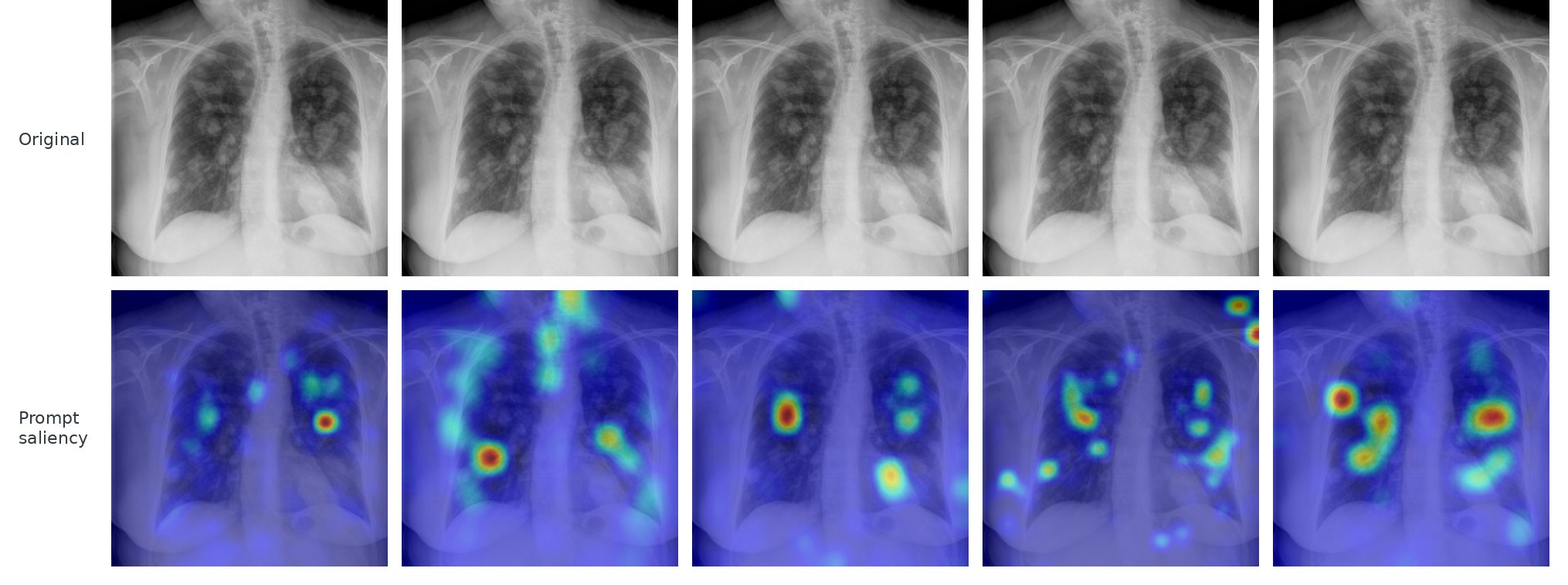}\\[4pt]
{\footnotesize (c)}\\[1pt]
\includegraphics[width=\linewidth]{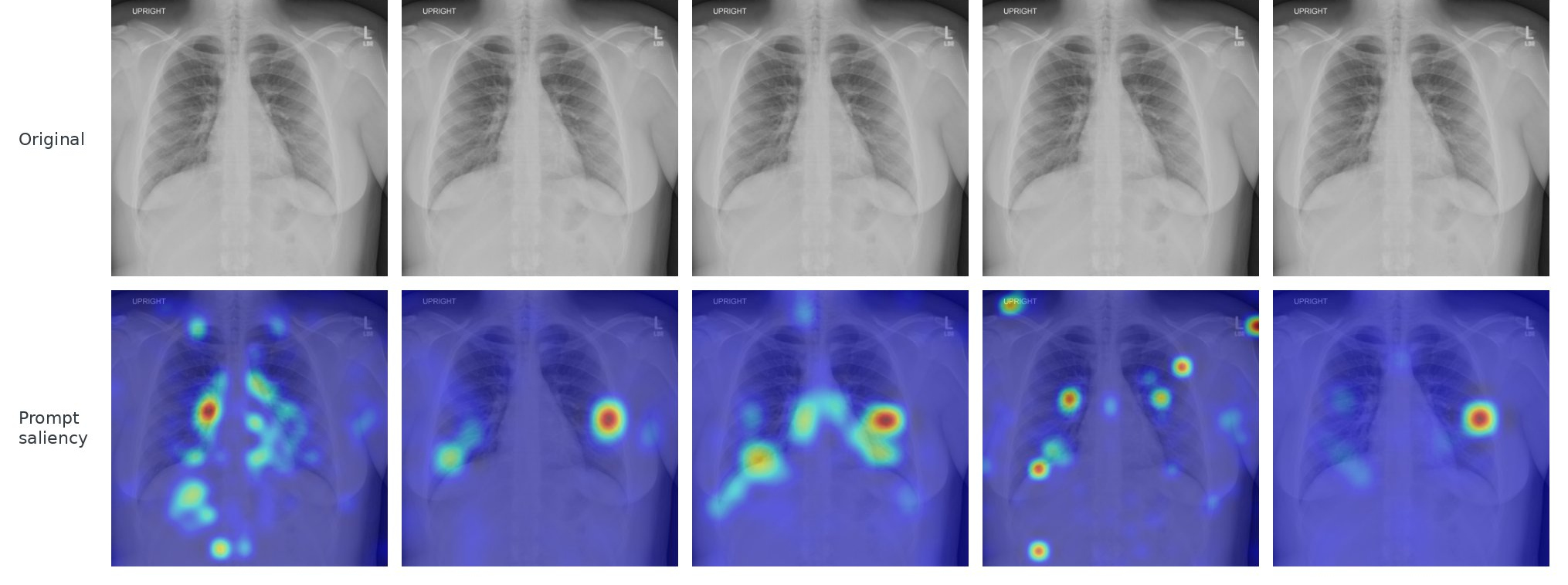}
\caption{\textbf{OpenI image-to-text retrieval.} Examples: (a) cardiomegaly and basilar opacity, (b) pulmonary nodules/mass, and (c) normal chest report. In panel (a), the proposed technique includes a prominent response near the central thorax. BMC emphasizes upper and central locations, BMCA includes responses lower in the thorax, ConceptCLIP shows a peripheral focus, and UniMed-CLIP distributes responses across several thoracic areas. The comparison illustrates different spatial responses to the same reference report without validating individual findings in that report. Heatmaps use the paired caption or report. Saliency construction follows Section~\ref{app:comparative-saliency-method}.}
\label{fig:saliency-openi}
\end{minipage}
\end{figure}
\clearpage

\begin{figure}[p]
\centering\begin{minipage}{5.5in}
\centering
\saliencycolumns
\vspace{3pt}
{\footnotesize (a)}\\[1pt]
\includegraphics[width=\linewidth]{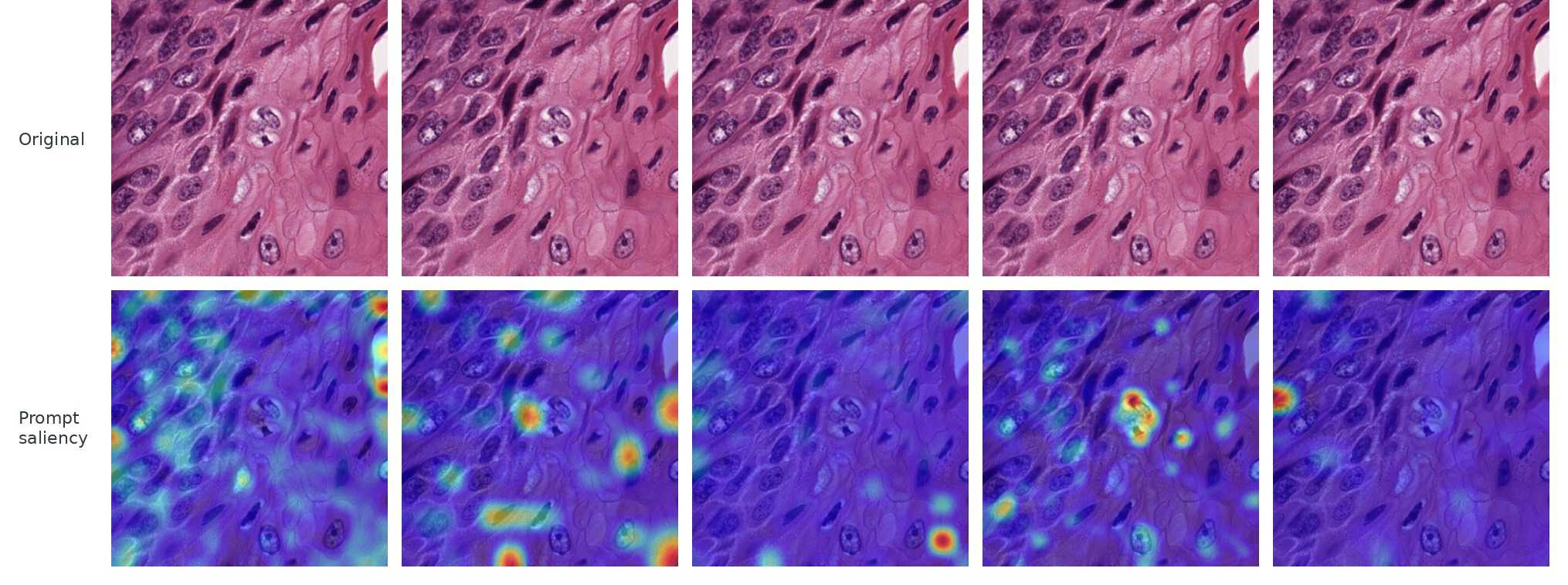}\\[4pt]
{\footnotesize (b)}\\[1pt]
\includegraphics[width=\linewidth]{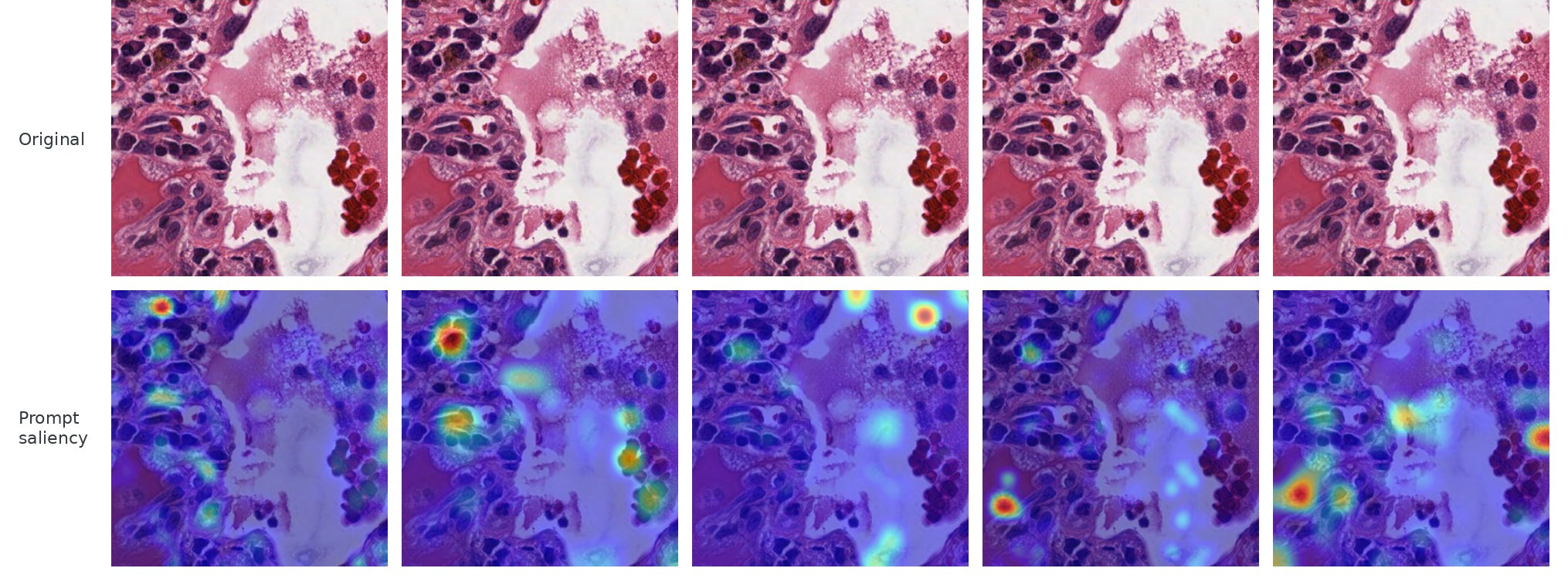}\\[4pt]
{\footnotesize (c)}\\[1pt]
\includegraphics[width=\linewidth]{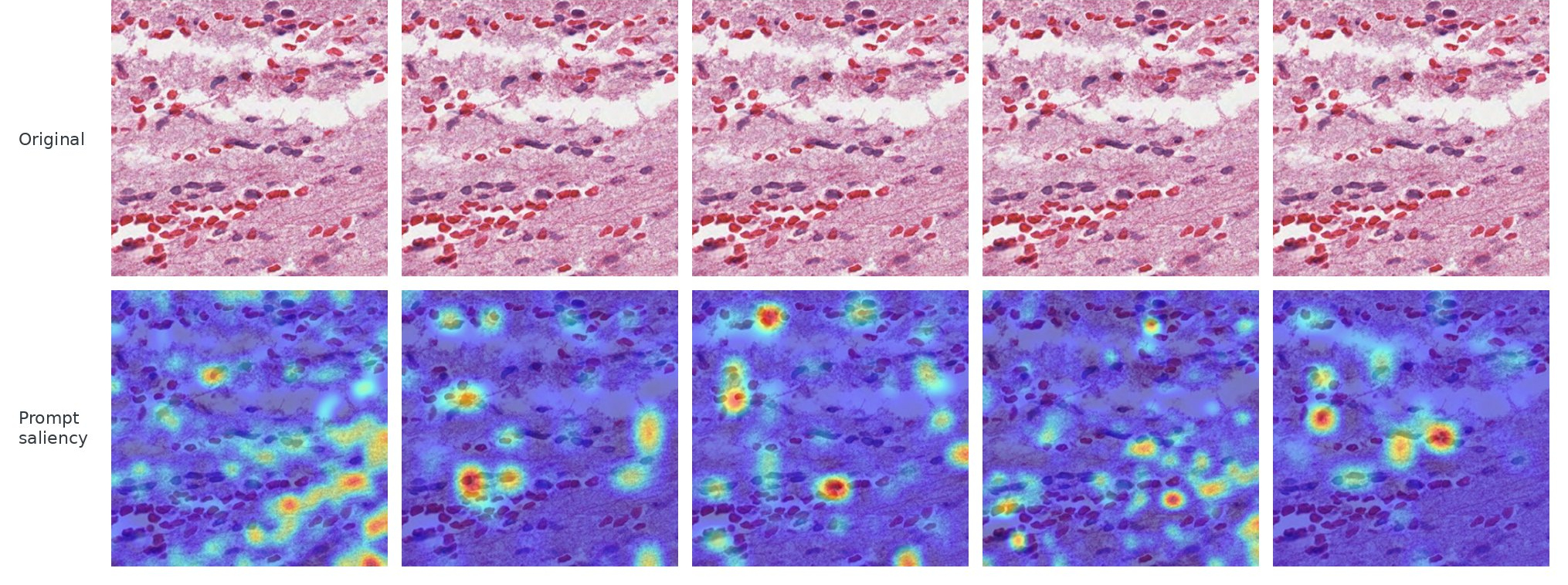}
\caption{\textbf{PathGen-600 image-to-text retrieval.} Examples: (a) papilloma / HPV-associated changes, (b) granulomatous inflammation, and (c) normal brain tissue. In panel (a), the proposed technique distributes responses among several cellular regions in the histology field. BMC and ConceptCLIP emphasize different groups of locations, while BMCA and UniMed-CLIP include prominent peripheral foci. These maps explain similarity to the paired pathology text and should not be interpreted as instance-level annotation of the reported cellular changes. Heatmaps use the paired caption or report. Saliency construction follows Section~\ref{app:comparative-saliency-method}.}
\label{fig:saliency-pathgen-600}
\end{minipage}
\end{figure}
\clearpage

\end{document}